\documentclass{article}

\usepackage{PRIMEarxiv}
\usepackage[utf8]{inputenc} 
\usepackage[T1]{fontenc}    
\usepackage{url}            
\usepackage{amsfonts}       
\usepackage{nicefrac}       
\usepackage{microtype}      
\usepackage{lipsum}
\usepackage{fancyhdr}       
\usepackage{graphicx}       
\graphicspath{{media/}}     

\usepackage{cite}           
\usepackage{amsmath,amssymb}
\usepackage{algorithmic}    
\usepackage{textcomp}
\usepackage{xcolor}      
\colorlet{red}{black} 

\def\BibTeX{{\rm B\kern-.05em{\sc i\kern-.025em b}\kern-.08em
    T\kern-.1667em\lower.7ex\hbox{E}\kern-.125emX}}

\usepackage{mathtools}
\usepackage{bm}             
\usepackage{amsthm}        
\usepackage{thmtools}

\usepackage{booktabs}       
\usepackage{array}
\usepackage{adjustbox}
\usepackage{multirow}
\usepackage{float}
\usepackage{grffile}
\usepackage{subcaption}

\usepackage{hyperref} 
\hypersetup{
    colorlinks=true,
    linkcolor=blue,
    citecolor=blue,
    urlcolor=blue
}

\usepackage[capitalize,noabbrev]{cleveref}

\usepackage[textsize=tiny]{todonotes}
\usepackage{tcolorbox}
\usepackage{enumitem}

\newif\ifcomments
\commentsfalse 
\newcommand{\shenyang}[1]{\ifcomments \textcolor{blue}{[Shenyang:\ #1]} \fi}
\newcommand{\yaoqing}[1]{\ifcomments \textcolor{orange}{[Yaoqing:\ #1]} \fi}

\newcommand{\tianyu}[1]{\ifcomments \textcolor{red}{[Tianyu:\ #1]} \fi}

\theoremstyle{plain}
\newtheorem{theorem}{Theorem}[section]

\newtheorem{lemma}[theorem]{Lemma}
\newtheorem{corollary}[theorem]{Corollary}

\theoremstyle{definition}

\newtheorem{assumption}[theorem]{Assumption}

\theoremstyle{remark}

\theoremstyle{definition}
\newtheorem{init}[theorem]{Initialization}

\newtheorem*{theorem-restated}{Theorem}
\newtheorem*{assumption-restated}{Assumption}
\newtheorem*{corollary-restated}{Corollary}
\newtheorem*{lemma-restated}{Lemma}

\title{Depth, Not Data:\\ An Analysis of Hessian Spectral Bifurcation}

\author{
  Shenyang Deng\thanks{Equal contribution} \\ 
  Department of Computer Science \\
  Dartmouth College \\
  Hanover, NH, USA \\
  \texttt{shenyang.deng.gr@dartmouth.edu} \\
  \And
  Boyao Liao\textsuperscript{*} \\ 
  Department of Mathematics \\
  University of Birmingham \\
  Birmingham, UK \\
  \texttt{bxl307@student.bham.ac.uk} \\
  \AND
  Zhuoli Ouyang\textsuperscript{*} \\ %
  Department of Computer Science \\
  Dartmouth College \\
  Hanover, NH, USA \\
  \texttt{Zhuoli.Ouyang@dartmouth.edu} \\
  \And
  Tianyu Pang \\
  Department of Computer Science \\
  Dartmouth College \\
  Hanover, NH, USA \\
  \texttt{tianyu.pang.gr@dartmouth.edu} \\
  \And
  Yaoqing Yang \\
  Department of Computer Science \\
  Dartmouth College \\
  Hanover, NH, USA \\
  \texttt{yaoqing.yang@dartmouth.edu} \\
}

\begin{document}
\maketitle

\begin{abstract}
The eigenvalue distribution of the Hessian matrix plays a crucial role in understanding the optimization landscape of deep neural networks. Prior work has attributed the well-documented ``bulk-and-spike'' spectral structure, where a few dominant eigenvalues are separated from a bulk of smaller ones, to the imbalance in the data covariance matrix. In this work, we challenge this view by demonstrating that such spectral \emph{Bifurcation} can arise purely from the network architecture, independent of data imbalance. 

Specifically, we analyze a deep linear network setup and prove that, even when the data covariance is perfectly balanced, the Hessian still exhibits a \emph{Bifurcation} eigenvalue structure: a dominant cluster and a bulk cluster. Crucially, we establish that the ratio between dominant and bulk eigenvalues scales linearly with the network depth. This reveals that the spectral gap is strongly affected by the network architecture rather than solely by data distribution. Our results suggest that both model architecture and data characteristics should be considered when designing optimization algorithms for deep networks. {\color{red} Our code is available at this \href{https://github.com/Dominator-Index/ISIT2026-Depth}{link}.}
\end{abstract}

\section{Introduction}
As the application scope of Deep Neural Networks (DNNs) continues to expand, there is a pressing need to improve optimization algorithms designed for these models. Achieving such algorithmic improvements requires a deeper theoretical understanding of the loss landscape of DNNs because the geometry of the loss landscape often strongly affects the effectiveness of optimization \cite{jin2017escape,yao2021adahessian,liu2024sophia,zhang2025adam} and the quality of the converged model (e.g., generalization performance) \cite{keskar2017large,li2018visualizing,garipov2018loss}.

In this context, the Hessian matrix has become a critical tool to characterize the curvature of the loss landscape. Extensive empirical research has demonstrated that the loss landscapes of various neural networks exhibit an ill-conditioned eigenvalue distribution structure \cite{sagun2017empirical,gur2018gradient,ghorbani2019investigation,yao2020pyhessian,cohen2021gradient,song2025does}. Specifically:
\shenyang{Although this part is highly overlapping with related work, I think I need this part to do the Intro narrative. And it seems that I have no space for related work about this and deeplinear. Do we need a related work in main text?(Notice that ISIT does not have a Appendix space.(But it offer an arxiv option) What do you think?}
\cite{sagun2017empirical} and \cite{gur2018gradient} observed that the gradient descent dynamics typically happen in a ``tiny subspace'' spanned by a few dominant Hessian eigenvectors, while the majority of directions are flat. \cite{ghorbani2019investigation} and \cite{yao2020pyhessian} provided a detailed spectral analysis, revealing that the Hessian spectrum is composed of massive ``bulk'' concentrated near zero and a few ``outlier'' eigenvalues (spikes). Furthermore, \cite{martin2021implicit} and \cite{hodgkinson2025models} argued that these spectral distributions are often heavy-tailed, suggesting a form of implicit self-regularization.


Complementing these empirical phenomena, theoretical work has offered explanations in specific settings. For example, \cite{pennington2017geometry} and \cite{louart2018random} use Random Matrix Theory (RMT) to build a direct link between the Hessian spectrum and the covariance matrix of the input data. A prevailing conclusion in these works \cite{pennington2018spectrum, papyan2020traces} is that the eigenvalue distribution of the Hessian $H_{L}=\nabla^2L(W)$ is heavily entangled with the distribution of the data covariance matrix. Consequently, the common perspective is that the ``imbalance'' observed in the Hessian spectrum—where a few large eigenvalues dominate—is primarily inherited from the intrinsic imbalance of the data covariance matrix.

This brings us to the core of our work. Following the context established by the aforementioned bulk-plus-spike literature, we classify the eigenvalues into a ``bulk space'' and a ``dominant spike space.'' Let $\lambda_1, \dots, \lambda_p$ be the non-zero eigenvalues of the Hessian $H_{L}$. We typically observe an index $k<p$ that separates the eigenvalues into two distinct clusters with a significant gap:
\begin{equation}
    \underbrace{\lambda_{1}, \dots, \lambda_k}_{\text{Dominant Spikes}} \in [ma, m(a+\delta)], \quad \underbrace{\lambda_{k+1}, \dots, \lambda_p}_{\text{Bulk}} \in [a, a+\delta],
\end{equation}
where $\delta\in (0,a)$ represents a fluctuation amplitude of eigenvalues in the corresponding eigensubspace, and $m \gg 1$ represents the magnitude of the spectral gap. In this work, we do not consider zero eigenvalues, as their corresponding eigenspaces contain no meaningful information.
\yaoqing{This is awkward. Why does the bulk start from $a$ instead of 0? Why does the right end point of the interval $\delta\in (0,a)$ coincide with $a$? I believe this should be explained in the main results. It would be better to be somewhat handwavy here. It seems quite arbitrary to say that "zero eigenvalues contain no meaningful information."}
While prior literature typically attributes this sharp \textbf{spectral bifurcation}---where the top $k$ eigenvalues separate from the bulk by a large magnitude $m$---to the imbalance of the data distribution, our work challenges this prevailing view by posing the following question:

\begin{center}
\begin{tcolorbox}[colback=blue!5!white,colframe=blue!75!black,title=Key Question]
\textbf{Q1:} Is the spectral bifurcation of $H_{L}$ (with $k$ spikes and gap $m$) exclusively attributable to the spectral imbalance of the input data covariance?
\end{tcolorbox}
\end{center}

\textbf{Our work demonstrates that an imbalance in the data covariance (or cross-covariance between inputs and outputs) spectrum is not a necessary condition for such spectral bifurcation.} Although we acknowledge that the Hessian in DNN optimization is indeed influenced by imbalanced data, we show that even when the data is perfectly balanced (i.e., the data is whitened such that the covariance and cross covariance are identity-like or uniform), the Hessian of the neural network optimization problem still generates a bifurcation spectral structure.

We show in a theoretical setting that there exists a $k$ such that the top $k$ eigenvalues are orders of magnitude ($m$ times) larger than the remaining $p-k$ non-zero eigenvalues. \yaoqing{If you write things like this, it sounds like you are going to explain what $m$ is. You can perhaps briefly say that $m$ is a quantity that scales linearly with the model depth.}
Crucially, the severity of this ill-condition (represented by $m$) is intrinsic to the model design and correlates with the network depth/architecture, rather than being solely dependent on data distribution.


To give a clearer intuition, we first visualize this conclusion via a numerical simulation {\color{red}within a feasible numerical range (since an overly ill-conditioned Hessian's eigenvalue, i.e., with excessively large $L$, is hard to compute numerically)}, as shown in Figure \ref{fig:hessian_sim}. We train a deep linear neural network to learn a simple linear operator mapping ($\mathbf{y} = \Phi \mathbf{x}$). We utilize a sufficient quantity of samples to ensure the data is whitened such that the cross-correlation matrix satisfies $\Sigma_{yx} = \mathbb{E}[\mathbf{y}\mathbf{x}^\top] = \Phi \simeq I_r$ (similar to the setting in \cite{arora2018a}, where $I_r$ is an $r \times r$ identity matrix, $\simeq$ is the 0 padding operator; the definition can be found in Section \ref{section_pre}). As illustrated, even when the eigenvalues of the data distribution are uniform, the Hessian of the network exhibits a clear stratification into two non-zero eigenvalue clusters and a cluster near zero.

Following the theoretical settings of \cite{arora2018a} and \cite{singh2021analytic} for deep linear networks, we provide a theoretical answer to this problem. We analyze why this imbalanced structure forms and identify the factors influencing it. Under the deep linear network setup following \cite{arora2018a} and \cite{singh2021analytic}, our theoretical contributions are summarized as follows:

\begin{itemize}
    \item \textbf{C1 (Spectral Bifurcation Independence):} We prove that even when both data covariance and cross-covariance are perfectly balanced (i.e., $\Sigma_{xx} \simeq I_{d_\ast}$ and $\Sigma_{yx} = UV^\top$ where $U, V$ are column-orthogonal matrices), the Hessian still exhibits a \emph{Bifurcation} structure among its non-zero eigenvalues: a dominant cluster and a bulk cluster. The remaining eigenvalues are zero and correspond to directions outside the data support. This shows that data imbalance is \textit{not} a necessary condition for spectral bifurcation.
    
    \item \textbf{C2 (Depth-Dependent Gap):} We establish that the spectral gap between the dominant and bulk clusters scales linearly with the network depth $L$. Specifically, the dominant eigenvalues are approximately $L$ times larger than the bulk eigenvalues, i.e., $m = \Theta(L)$ as illustrated in Figure \ref{fig:hessian_sim}. This reveals that the ill-conditioning severity is intrinsically determined by the model architecture rather than solely by data distribution.
\end{itemize}


    

\vspace{-5pt}

\begin{figure}[h]
    \centering
    \includegraphics[width=0.8\textwidth]{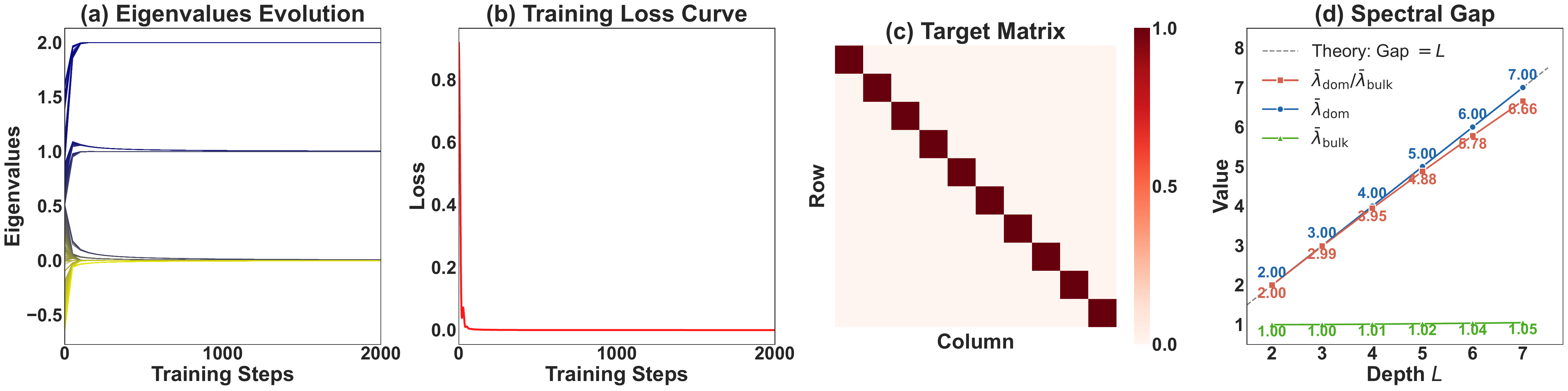}
    \caption{(a) Hessian eigenvalue evolution for $L=2$ on whitened data. (b) Training loss curve ($L=2$). (c) Target matrix (whitened data). (d) Spectral gap $\bar{\lambda}_{\rm dom}/\bar{\lambda}_{\rm bulk}$ across depths $L=2$--$7$, where $\bar{\lambda}_{\rm dom}$ and $\bar{\lambda}_{\rm bulk}$ denote the mean eigenvalues of the dominant and bulk spaces, respectively; the empirical gap closely matches the theoretical prediction $\text{Gap}=L$.}
    \label{fig:hessian_sim}
    \vspace{-10pt}
\end{figure}
\section{Preliminary}
\label{section_pre}
\noindent\textbf{Notation.}
For matrix products, we define $W^{k:l} = W^k \cdots W^l$ when $k > l$ and $W^{k:l} = (W^k)^\top \cdots (W^l)^\top$ when $k < l$. Additional operations include: $\otimes$ for Kronecker product, $\mathrm{vec}_r$ for row-wise matrix vectorization,\yaoqing{Is the vector a column?} $\mathrm{rk}(W)$ for matrix rank, and $I_k$ for the $k \times k$ identity matrix. Let matrix $A \in \mathbb{R}^{m \times n}$ and $B \in \mathbb{R}^{p \times q}$. To work with matrix-matrix derivatives in the row-wise vectorization form, we adopt the convention:
\begin{align*}
    \frac{\partial A}{\partial B} := \frac{\partial \,\mathrm{vec}_{r}(A)}{\partial \,\mathrm{vec}_{r}(B)^{\top}}.
\end{align*}
A useful identity for our analysis is:
\begin{align*}
    \frac{\partial (AWB)}{\partial W} = A \otimes B^{\top}.
\end{align*}
If $m \geq p$ and $n \geq q$, we define the padding operator $\simeq$ as:
$$
A \simeq B \quad \Leftrightarrow \quad A = \begin{bmatrix} B & \mathbf{0}_{p \times (n-q)} \\ \mathbf{0}_{(m-p) \times q} & \mathbf{0}_{(m-p) \times (n-q)} \end{bmatrix},
$$
where $\mathbf{0}_{i \times j}$ denotes an $i \times j$ zero matrix.\\
\addtolength{\topmargin}{0.1in}
\textbf{Deep Linear Neural Network.} Following the setup of \cite{arora2018a,singh2021analytic,bordelon2025deep}, we consider a depth-$L$ deep linear neural network defined by the composition $F(\mathbf{x}) = W^L \cdots W^1 \mathbf{x} = W^{L:1} \mathbf{x}$, where the weight matrix $W^l \in \mathbb{R}^{d_l \times d_{l-1}}$ parameterizes the $l$-th layer for $l = 1, \ldots, L$. The network has input dimension $d = d_0$, output dimension $K = d_L$, hidden layer widths $d_1, \ldots, d_{L-1}$, and total number of hidden neurons $D = \sum_{l=1}^{L-1} d_l$. We denote the total parameter set as $W := \{W^1, \ldots, W^L\}$ and use the explicit parameterization $F_W(\mathbf{x})=W^{L:1} \mathbf{x}$ when necessary.\\
\textbf{Loss Function.}
We consider a training dataset of i.i.d. samples $\{(\mathbf{x}_i, \mathbf{y}_i)\}_{i=1}^N$ from distribution $p$. We consider the squared loss $L_W(\mathbf{x}, \mathbf{y}) = \frac{1}{2} \|\mathbf{y} - \hat{\mathbf{y}}\|_2^2$ where $\hat{\mathbf{y}} = F_W(\mathbf{x})$.
The population loss is defined as $L(W) = \mathbb{E}_{(\mathbf{x},\mathbf{y}) \sim p}[L_W(\mathbf{x},\mathbf{y})]$. For brevity, specifically in subsequent derivations, we denote the expectation over the data distribution simply as $\mathbb{E}$. The Hessian of the population loss is denoted $H = \nabla^2 L(W)$. We define the residual at sample $(\mathbf{x}, \mathbf{y})$ as $\boldsymbol{\delta}_{\mathbf{x},\mathbf{y}} := \hat{\mathbf{y}} - \mathbf{y}$.

 Applying the above to the network output, we obtain:
\begin{align*}
    \frac{\partial F}{\partial W^{l}}
    = W^{L:l+1} \otimes \big(\mathbf{x}^{\top}W^{1:l-1}\big)
    \in \mathbb{R}^{d_L \times d_l d_{l-1}}.
\end{align*}

Using the chain rule (backpropagation), the gradient of the loss for a sample $(\mathbf{x}, \mathbf{y})$ with respect to $W^{l}$ is:
\begin{align*}
     \nabla_{W^l} L(W_t)
    &= \big[W^{l+1:L}\boldsymbol{\delta}_{\mathbf{x},\mathbf{y}}\big] \big[W^{l-1:1}\mathbf{x}\big]^{\top} \\
    &= W^{l+1:L}\big(W^{L:1}\mathbf{x}\mathbf{x}^{\top}-\mathbf{y}\mathbf{x}^{\top}\big)W^{1:l-1}.
\end{align*}
For subsequent analysis, we introduce the following shorthand notations for the expected quantities (where $\mathbb{E}$ denotes $\mathbb{E}_{(\mathbf{x}, \mathbf{y}) \sim \mu}$):
\begin{align*}
    &\Omega := \mathbb{E}[\boldsymbol{\delta}_{\mathbf{x},\mathbf{y}}\mathbf{x}^{\top}]
    = \mathbb{E}[W^{L:1}\mathbf{x}\mathbf{x}^{\top}] - \mathbb{E}[\mathbf{y}\mathbf{x}^{\top}], \\
    &\Sigma_{xx} := \mathbb{E}[\mathbf{x}\mathbf{x}^{\top}],\qquad
    \Sigma_{yx} := \mathbb{E}[\mathbf{y}\mathbf{x}^{\top}].
\end{align*}
\\
\textbf{Gauss-Newton Decomposition.}
Following the analytical technique in \cite{singh2021analytic,zhao2024theoretical}, we apply the Gauss-Newton Decomposition to analyze the Hessian of a deep linear neural network. This technique involves examining the Hessian through the lens of the chain rule, decomposing the total loss Hessian $H_L = H_o + H_f$, where:
\[
    H_o = \mathbb{E} \left[ \nabla_{W} F(\mathbf{x})^\top \left[ \frac{\partial^2 L_W}{\partial F ^2} \right] \nabla_{W} F(\mathbf{x}) \right],
\]
and
\[
    H_f = \mathbb{E} \left[ \sum_{c=1}^{d_L} \left[ \frac{\partial L_W}{\partial F_c} \right] \nabla^2_{W} F_c(\mathbf{x}) \right].
\]
In the case of Mean Squared Error (MSE) loss, this Gauss-Newton breakdown aligns precisely with the earlier analysis as shown in \cite{singh2021analytic}. Following convention, we call the initial component $H_o$ the \textit{outer-product Hessian} and the latter component $H_f$ the \textit{functional Hessian}.

\smallskip

\noindent\textbf{Balanced Initialization.}
We follow the balanced initialization procedure in \cite{arora2018a}: 

\begin{init}
\label{balance_init}
    Let $d_0, d_1, \ldots, d_L \in \mathbb{N}$ satisfy $\min\{d_1, \ldots, d_{L-1}\} \ge \min\{d_0, d_L\}$, and let $\mathcal{D}$ be a distribution over $d_L \times d_0$ matrices. A \textit{balanced initialization} of the weight matrices $W^l \in \mathbb{R}^{d_l \times d_{l-1}}$, for $l=1, \ldots, L$, is performed as follows:\\
\begin{enumerate}
    \item Sample a matrix $A \in \mathbb{R}^{d_L \times d_0}$ according to $\mathcal{D}$.
    \item Compute the singular value decomposition $A = U \Sigma_0 V^{\top}$, where $U \in \mathbb{R}^{d_L \times \min\{d_0, d_L\}}$ and $V \in \mathbb{R}^{d_0 \times \min\{d_0, d_L\}}$ have orthonormal columns, and $\Sigma_0 \in \mathbb{R}^{\min\{d_0, d_L\} \times \min\{d_0, d_L\}}$ is diagonal, containing the singular values of $A$.
    \item Set the weight matrices as $W^L \simeq U \Sigma_0^{1/L}, \, W^{L-1} \simeq \Sigma_0^{1/L}, \ldots, W^2 \simeq \Sigma_0^{1/L}, \, W^1 \simeq \Sigma_0^{1/L} V^{\top}$.
\end{enumerate}
\end{init}
\noindent \textbf{Gradient Descent Dynamics.}
We optimize the network parameters $W = \{W^1, \ldots, W^L\}$ using Gradient Descent (GD) with a constant learning rate $\eta > 0$. Following \cite{arora2018a}, we assume the data is whitened and consider the gradient update in the expectation sense. Substituting the gradient derivation from the previous section, the population gradient update takes the explicit form:
\begin{align}
    W_{t+1}^l &= W_t^l - \eta \cdot W_t^{l+1:L} \big( W_t^{L:1}\Sigma_{xx} - \Sigma_{yx} \big) W_t^{1:l-1} 
\end{align}
This update rule highlights that the evolution of each layer depends on the global error signal propagated through the adjacent layers.
\section{Hessian Bifurcation Theory}
\label{hessian_bifurcation}
\subsection{Theorem Assumptions}
 We have the following assumptions for our results.  However, note that not all assumptions are actually used in the {\color{red}every} theorem. Some are included solely to help the reader develop intuition for our results by considering a more concrete special case.

{\color{red}
\begin{assumption}[Sufficient Width \& Whitened {\color{red}{Balanced}} Input] \label{assumption 12}
    We assume the dimension of weight matrices satisfy: $d_\ast := \min\{d_L, d_0\} \leq \min\{d_{L-1}, \ldots, d_1\}$. The input data is whitened and balanced:
    \begin{equation}
        \Sigma_{xx} \simeq I_{d_\ast}.
    \end{equation}
\end{assumption}

\begin{assumption}[Alignment with Initialization] \label{assumption 4}
    The target matrix aligns with the principal directions of $A$ in \textbf{Initialization} \ref{balance_init}. With $U, V$ from the initialization and data rank $r$:
    \begin{equation}
        \Sigma_{yx} = U \mathcal{I}_r V^\top, \qquad \mathcal{I}_r \in \mathbb{R}^{d_\ast \times d_\ast} \simeq I_r.
    \end{equation}
\end{assumption}

\begin{assumption}[Uniform Spectral Initialization (USI)] \label{assumption 5}
    A special case considered later: the initial weight product has uniform singular values. With $\Sigma_0$ from \textbf{Initialization} \ref{balance_init},
    \begin{equation}
        \Sigma_0^{1/L} \simeq \mu I_r, \qquad \mu > 0,
    \end{equation}
    so the network acts as a scaled isometry at initialization.
\end{assumption}}
{\color{red} \noindent\textbf{Remark} The core purpose of the above assumptions is to construct an example in which the Hessian still exhibits bifurcation even when the data covariance matrix (or the target matrix) is balanced. This is not because we overlook the non-balanced case. Bifurcation arising from non-balanced data has already been extensively studied in prior work. Our goal is rather to show that non-balanced data is a sufficient but not necessary condition for ill-conditioning, and that model-intrinsic factors can also induce ill-conditioning.}
\subsection{Main Result}
To establish our main result on the spectral structure of the Hessian, we first need to understand how the weight matrices evolve during training. A key observation is that all weight matrices in a deep linear network share a common spectral structure throughout the optimization process. This property, which we formalize in the following lemma \ref{lemma shared spectral structure}, serves as the foundation for our subsequent analysis.
\begin{lemma}[Shared Spectral Structure]
\label{lemma shared spectral structure}
    Under Assumptions \ref{assumption 12} and \ref{assumption 4}, all weight matrices $W^k_t$ for $k = 1, \ldots, L$ share a same spectral structure $\Sigma_t^{1/L} = \text{diag}(\lambda_{1,t}, \lambda_{2,t}, \ldots, \lambda_{d_*,t})$ at any time $t \geq 0$.\yaoqing{where $\Sigma_t$ means the...} Specifically:
    \begin{enumerate}
        \item For $1 < k < L$: $W^k_t \simeq \Sigma_t^{1/L}$.
        \item For $k = 1$: $W^1_t \simeq \Sigma_t^{1/L} V^\top$.
        \item For $k = L$: $W^L_t \simeq U \Sigma_t^{1/L}$.
    \end{enumerate}
    where $U \in \mathbb{R}^{d_L \times d_*}$ and $V \in \mathbb{R}^{d_0 \times d_*}$ are the left and right singular vector matrices from the balanced initialization, which remain constant throughout training. Furthermore, the eigenvalues $\lambda_{i,t}$ for $i = 1, \ldots, r$ evolve according to:
    \begin{align}
        \lambda_{i,t+1} = \lambda_{i,t} - \eta \lambda_{i,t}^{2L-1} + \eta \lambda_{i,t}^{L-1},
    \end{align}
    while the eigenvalues $\lambda_{i,t}$ for $i = r+1, \ldots, d_*$ remain unchanged at their initial values.
\end{lemma}

\noindent \textbf{Explain.} Lemma \ref{lemma shared spectral structure} reveals that under \textbf{Initialization} \ref{balance_init}, the spectral structure of all weight matrices is governed by a single diagonal matrix $\Sigma_t^{1/L}$, whose diagonal entries $\lambda_{i,t}$ evolve according to a unified update rule. This shared spectral structure is crucial because it implies that the optimization dynamics can be characterized entirely by tracking the evolution of these $d_*$ eigenvalues. 

Building upon this spectral characterization, we are now ready to present our main theorem, which establishes the bifurcation structure of the Hessian eigenvalues. The theorem demonstrates that even when the data covariance is perfectly balanced (i.e., $\Sigma_{xx} \simeq I_{d_\ast}$), the Hessian of the loss function exhibits a pronounced two-cluster spectral structure, with a gap that scales linearly with the network depth $L$.
\begin{theorem}[Hessian Bifurcation]
\label{theorem main result}
    Under Assumptions \ref{assumption 12} and \ref{assumption 4}, and assuming $r \leq d_*$, consider a depth-$L$ deep linear neural network trained with gradient descent with step size $\eta < \min\big\{\frac{1}{L}, \frac{1}{M^{2L-2}}\big\}$. Let $\lambda_{i,t}$ for $i = 1, \ldots, r$ denote the effective eigenvalues of the weight matrix $\Sigma_t^{1/L}$ at time $t$. Suppose that $\lambda_{i,t} \in [m_t - \delta_t, m_t + \delta_t]$ and the condition $(m_t + \delta_t)/(m_t - \delta_t) < L^{\frac{1}{2(L-1)}}$ holds.\yaoqing{What is $m_t$ and $\delta_t$?} Then, the Hessian $H_{L,t}$ exhibits a two-cluster spectral structure:
    \begin{enumerate}
        \item \textbf{(Dominant Space)} There exist $r^2$ eigenvalues of $H_{L,t}$ lying in:
        \begin{equation}
        \begin{split}
            \bigg[ &L(m_t - \delta_t)^{2(L-1)} - O(e^{-L\alpha\eta t}), \\
                   &L(m_t + \delta_t)^{2(L-1)} + O(e^{-L\alpha\eta t}) \bigg].
        \end{split}
        \end{equation}
        
        \item \textbf{(Bulk Space)} There exist $(d_\ast + d_L - 2r)r$ eigenvalues of $H_{L,t}$ lying in:
        \begin{equation}
        \begin{split}
            \bigg[ &(m_t - \delta_t)^{2(L-1)} - O(e^{-L\alpha\eta t}), \\
                   &(m_t + \delta_t)^{2(L-1)} + O(e^{-L\alpha\eta t}) \bigg].
        \end{split}
        \end{equation}
        
        \item \textbf{(Zero Space)} The remaining eigenvalues of $H_{L,t}$ are zero.
    \end{enumerate}
    
    Moreover, let $\lambda_{\text{dom}}$ denote an arbitrary eigenvalue belonging to the Dominant Space and $\lambda_{\text{bulk}}$ denote an arbitrary eigenvalue belonging to the Bulk Space. Their ratio satisfies:
\begin{equation}
    \frac{\lambda_{\text{dom}}}{\lambda_{\text{bulk}}} = \Theta(L),
\end{equation}
where $\alpha$ is a positive constant independent of $t$, provided $t$ is sufficiently large.
\end{theorem}





\noindent \textbf{Explain.} Theorem \ref{theorem main result} provides a complete characterization of the Hessian spectrum in deep linear networks. The result reveals three distinct eigenspaces: (i) a dominant space containing $r^2$ eigenvalues of order $L \cdot m_t^{2(L-1)}$, (ii) a bulk space containing $(d_\ast + d_L - 2r)r$ eigenvalues of order $m_t^{2(L-1)}$, and (iii) a zero space containing the remaining eigenvalues that vanish asymptotically. Crucially, the ratio between the dominant and bulk eigenvalues is $\Theta(L)$, demonstrating that the spectral gap is an intrinsic property of the network architecture rather than a consequence of data imbalance.

\noindent \textbf{Remark.} The condition $(m_t + \delta_t)/(m_t - \delta_t) < L^{1/(2(L-1))}$ ensures that the eigenvalues remain sufficiently concentrated to maintain a clear separation between the two clusters. This condition is naturally satisfied when the initialization is near-uniform and becomes increasingly easier to satisfy as the network depth $L$ grows.

To gain deeper insight into the mechanism underlying this spectral bifurcation, we consider a special case where the weight matrices are initialized with uniform singular values. This setting, formalized in Assumption \ref{assumption 5}, allows us to obtain a sharper characterization of the Hessian spectrum.
\begin{corollary}[Hessian Bifurcation with USI]
\label{corollary two eigenvalues}
    Under Assumptions \ref{assumption 12}, \ref{assumption 4}, and \ref{assumption 5}, and assuming $r \leq d_*$,  {\color{red}$\eta < \min\big\{\frac{1}{L}, \frac{1}{M^{2L-2}}\big\}$}, the Hessian $H_{L,t}$ has exactly two distinct nonzero eigenvalues (up to an exponentially small perturbation). Specifically:
    \begin{enumerate}
        \item \textbf{(Dominant Space)} There exist $r^2$ eigenvalues equal to $L\mu_t^{2(L-1)} + O(e^{-L\alpha\eta t})$.
        
        \item \textbf{(Bulk Space)} There exist $(d_\ast + d_L - 2r)r$ eigenvalues equal to $\mu_t^{2(L-1)} + O(e^{-L\alpha\eta t})$.
        
        \item \textbf{(Zero Space)} The remaining eigenvalues are $O(e^{-L\alpha\eta t})$.
    \end{enumerate}
    Moreover, the ratio between the dominant and bulk eigenvalues is exactly $L$, i.e.,
    \begin{align}
        \frac{\lambda_{\text{dom}}}{\lambda_{\text{bulk}}} = L,
    \end{align}
    where $\mu_t$ denotes the common value of all effective eigenvalues $\lambda_{i,t} = \mu_t$ for $i = 1, \ldots, r$ at time $t$, and $\alpha$ is a positive constant.
\end{corollary}

Corollary \ref{corollary two eigenvalues} reveals the fundamental mechanism behind the Hessian bifurcation phenomenon. Under uniform spectral initialization, the Hessian possesses exactly two distinct nonzero eigenvalues: $L\mu_t^{2(L-1)}$ in the dominant space and $\mu_t^{2(L-1)}$ in the bulk space, with a ratio of exactly $L$. This clean separation provides a transparent view of how the spectral structure emerges. {\color{red} In Section \ref{Simsec}, we discuss depth bifurcation beyond our theoretical setting via simulations.}\\
\noindent {  \textbf{Proof Sketch.} } Under Assumptions \ref{assumption 12} and \ref{assumption 4}, we first establish that all weight matrices share a common spectral structure throughout training (Lemma \ref{lemma shared spectral structure}), where the singular values of the end-to-end mapping evolve according to a unified dynamical system $\Sigma_{t+1}^{1/L} = g(\Sigma_t^{1/L})$. Analyzing this dynamics shows that the population loss converges exponentially as $L(W_t) = O(e^{-2L\alpha\eta t})$, which by the Gauss-Newton decomposition $H_{L,t} = H_{o,t} + H_{f,t}$ implies that the functional Hessian norm satisfies $\|H_{f,t}\|_2 = O(e^{-L\alpha\eta t})$ (Lemma \ref{lemma: 2norm-t}). Following \cite{singh2021analytic} Proposition 2, the outer-product Hessian admits the factorization $H_{o,t} = A_{o,t} B_o A_{o,t}^\top$, whose nonzero eigenvalues coincide with those of the Gram matrix $B_o^{1/2} A_{o,t}^\top A_{o,t} B_o^{1/2}$. By explicitly computing the weight matrix products and exploiting the Kronecker product structure, we show that these eigenvalues split into a dominant cluster of $r^2$ eigenvalues scaling as $L \cdot m_t^{2(L-1)}$ and a bulk cluster of $(d_\ast + d_L - 2r)r$ eigenvalues scaling as $m_t^{2(L-1)}$ (Lemma \ref{lemma generalize eigenvalues}). Finally, applying Weyl's inequality to $H_{L,t} = H_{o,t} + H_{f,t}$ shows that the eigenvalues of the true Hessian $H_{L,t}$ lie within $O(e^{-L\alpha\eta t})$ of those of $H_{o,t}$, yielding the stated two-cluster structure with ratio $\lambda_{\text{dom}}/\lambda_{\text{bulk}} = \Theta(L)$.

\section{Simulation}
\label{Simsec}
To {\color{red} verify} the theoretical results in
Sections~\ref{section_pre} and \ref{hessian_bifurcation},
we conduct numerical simulations under
Assumptions~\ref{assumption 12}, \ref{assumption 4}. For brevity, we provide a experiment here, and detailed experimental settings and additional results are provided in Appendix~\ref{appendix_simulation}. We consider a depth-$L$ deep linear network with input dimension $d_0 = 10$, output dimension $d_L = 16$, hidden widths $d_1=\cdots=d_{L-1} = 20$, and effective rank $r = 4$, trained on whitened data ($\Sigma_{xx} \simeq I_{d_\ast}$) with $L =3$. Figure~\ref{fig:hessian_sim2} reports the Hessian eigenvalue trajectories and training loss curves, the latter confirming convergence. The non-zero spectrum cleanly separates into a dominant cluster of dimension $r^2 = 16$ and a bulk cluster of dimension $(d_\ast + d_L - 2r)r = 72$, consistent across both depths. Moreover, the spectral gap scales with depth, yielding ratios of approximately $3$ for $L=3$, matching the linear growth of dominant eigenvalues with $L$ predicted by Theorem~\ref{theorem main result}.
\begin{figure}[h]
    \centering
    \includegraphics[width=0.5\textwidth]{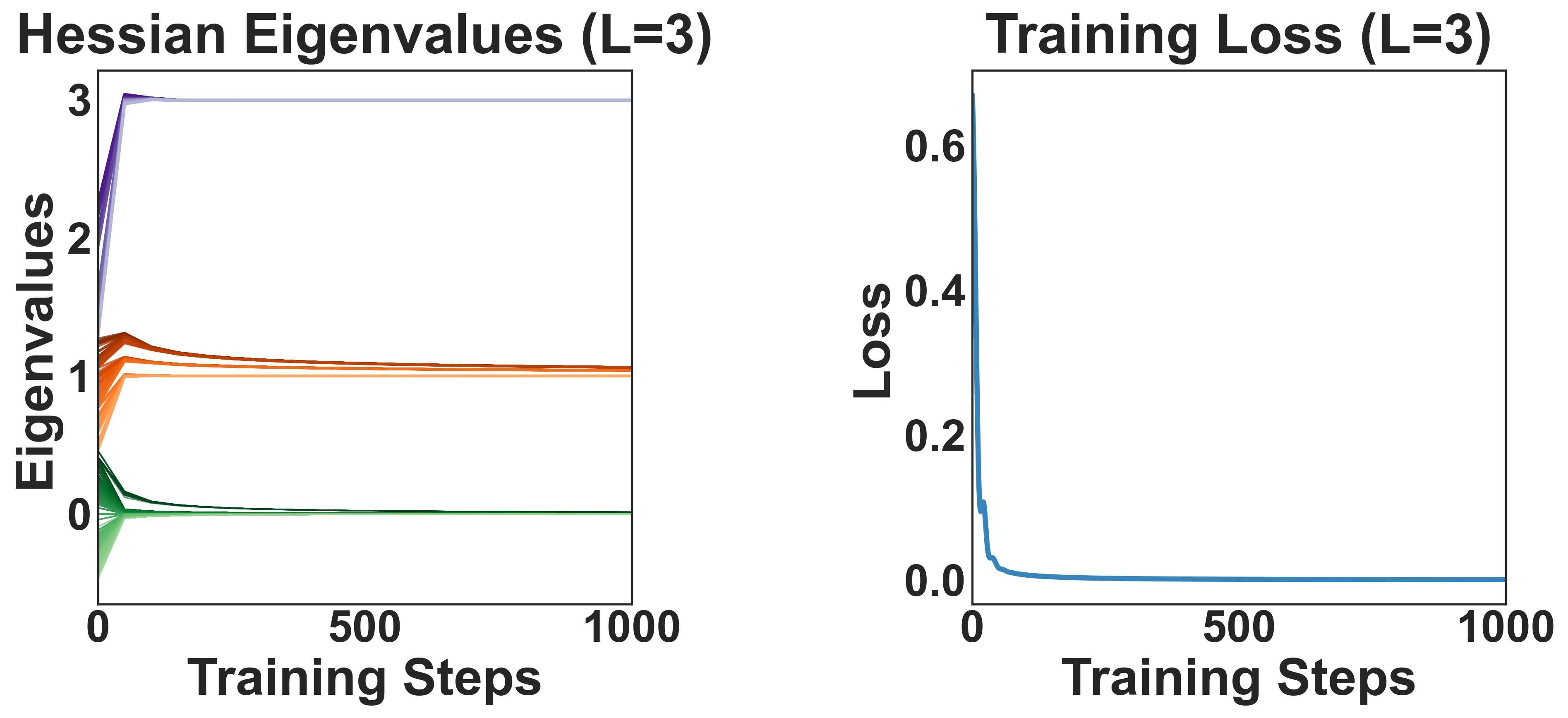}
    \caption{Hessian eigenvalue trajectories and training loss for deep linear networks with depths $L=3$.}
    \label{fig:hessian_sim2}
    \vspace{-10pt}
\end{figure}

\noindent\textbf{{\color{red}Discussion Beyond Theoretical Setup.}} To probe whether the Hessian bifurcation extends beyond the linear regime, we repeat the experiment with a $\tanh$ activation applied after each layer, keeping all other settings unchanged. As shown in Figure~\ref{fig:3panels_tanh}, the non-zero spectrum still splits cleanly into a dominant and a bulk cluster (panel a), and training converges stably (panel b). Sweeping $L=2$--$7$ (panel c), the empirical gap $\bar{\lambda}_{\rm dom}/\bar{\lambda}_{\rm bulk}$ tracks the $\text{Gap}=L$ prediction: $\bar{\lambda}_{\rm dom}$ grows linearly in $L$ while $\bar{\lambda}_{\rm bulk}$ remains nearly constant. This suggests that the depth-dependent bifurcation of Theorem~\ref{theorem main result} is not an artifact of linearity but reflects a more general structural property of deep networks.

\begin{figure}[h]
    \centering
    \includegraphics[width=0.65\textwidth]{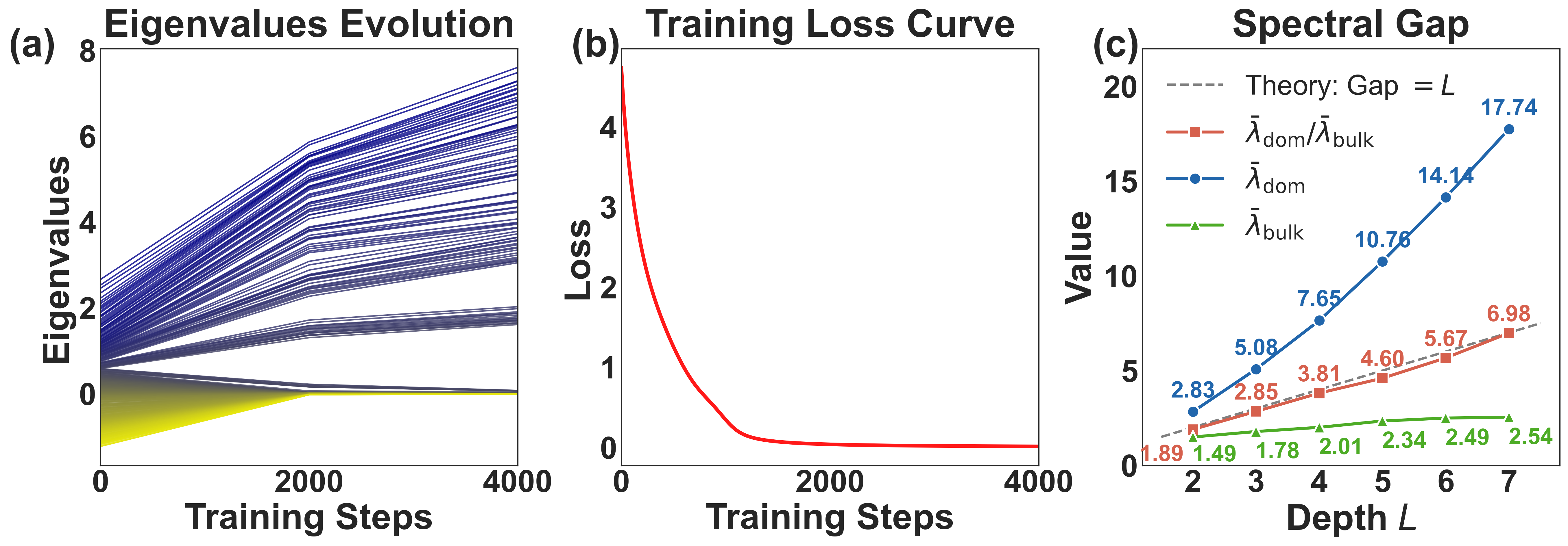}
    \caption{(a) Eigenvalue evolution ($L=3$, tanh) showing spectral bifurcation. (b) Training loss curve. (c) Spectral gap $\bar{\lambda}_{\rm dom}/\bar{\lambda}_{\rm bulk}$ vs.\ depth $L=2$--$7$, where $\bar{\lambda}_{\rm dom}$ and $\bar{\lambda}_{\rm bulk}$ are the mean eigenvalues of the dominant and bulk spaces, respectively; empirical ratio (red) matches $\text{Gap}=L$ (dashed), $\bar{\lambda}_{\rm dom}$ (blue) grows linearly, $\bar{\lambda}_{\rm bulk}$ (green) stays nearly constant.}
    \label{fig:3panels_tanh}
\end{figure}

\section{Conclusion}

In this work, we conduct a theoretical analysis of the spectral bifurcation phenomenon in the Hessian of deep linear neural networks, and demonstrate that this phenomenon is an intrinsic property of the model architecture rather than a consequence of data imbalance. Our main result (Theorem \ref{theorem main result}) shows that the Hessian exhibits a two-cluster eigenvalue structure with a gap that scales linearly with the network depth $L$, even when the data covariance is perfectly balanced. The key insight is that directions within the effective parameter space accumulate curvature contributions from all $L$ layers, while directions outside this space only contribute through a single layer, leading to an $L$-fold gap between the dominant and bulk eigenvalues. This depth-induced conditioning represents an important factor contributing to the ill-conditioned optimization landscape in deep neural networks, independent of data imbalance, suggesting that both data characteristics and model depth should be considered when designing optimization strategies.

\bibliographystyle{unsrt}

\bibliography{references}

@article{singh2021analytic,
  title={Analytic insights into structure and rank of neural network hessian maps},
  author={Singh, Sidak Pal and Bachmann, Gregor and Hofmann, Thomas},
  journal={Advances in Neural Information Processing Systems},
  volume={34},
  pages={23914--23927},
  year={2021}
}

@inproceedings{
arora2018a,
title={A Convergence Analysis of Gradient Descent for Deep Linear Neural Networks},
author={Sanjeev Arora and Nadav Cohen and Noah Golowich and Wei Hu},
booktitle={International Conference on Learning Representations},
year={2019},
url={https://openreview.net/forum?id=SkMQg3C5K7},
}

@inproceedings{bordelon2025deep,
  title={Deep Linear Network Training Dynamics from Random Initialization: Data, Width, Depth, and Hyperparameter Transfer},
  author={Bordelon, Blake and Pehlevan, Cengiz},
  booktitle={International Conference on Machine Learning},
  pages={4968--4997},
  year={2025},
  organization={PMLR}
}

@inproceedings{
song2025does,
title={Does {SGD} really happen in tiny subspaces?},
author={Minhak Song and Kwangjun Ahn and Chulhee Yun},
booktitle={The Thirteenth International Conference on Learning Representations},
year={2025},
url={https://openreview.net/forum?id=v6iLQBoIJw}
}

@inproceedings{
zhao2024theoretical,
title={Theoretical Characterisation of the Gauss Newton Conditioning in Neural Networks},
author={Jim Zhao and Sidak Pal Singh and Aurelien Lucchi},
booktitle={The Thirty-eighth Annual Conference on Neural Information Processing Systems},
year={2024},
url={https://openreview.net/forum?id=fpOnUMjLiO}
}

@article{sagun2017empirical,
  title={Empirical analysis of the hessian of over-parametrized neural networks},
  author={Sagun, Levent and Evci, Utku and Guney, V Ugur and Dauphin, Yann and Bottou, Leon},
  journal={arXiv preprint arXiv:1706.04454},
  year={2017}
}

@article{gur2018gradient,
  title={Gradient descent happens in a tiny subspace},
  author={Gur-Ari, Guy and Roberts, Daniel A and Dyer, Ethan},
  journal={arXiv preprint arXiv:1812.04754},
  year={2018}
}

@inproceedings{ghorbani2019investigation,
  title={An investigation into neural net optimization via hessian eigenvalue density},
  author={Ghorbani, Behrooz and Krishnan, Shankar and Xiao, Ying},
  booktitle={International Conference on Machine Learning},
  pages={2232--2241},
  year={2019},
  organization={PMLR}
}

@inproceedings{yao2020pyhessian,
  title={Pyhessian: Neural networks through the lens of the hessian},
  author={Yao, Zhewei and Gholami, Amir and Keutzer, Kurt and Mahoney, Michael W},
  booktitle={2020 IEEE international conference on big data (Big data)},
  pages={581--590},
  year={2020},
  organization={IEEE}
}

@inproceedings{
cohen2021gradient,
title={Gradient Descent on Neural Networks Typically Occurs at the Edge of Stability},
author={Jeremy Cohen and Simran Kaur and Yuanzhi Li and J Zico Kolter and Ameet Talwalkar},
booktitle={International Conference on Learning Representations},
year={2021},
url={https://openreview.net/forum?id=jh-rTtvkGeM}
}

@inproceedings{jin2017escape,
  title={How to escape saddle points efficiently},
  author={Jin, Chi and Ge, Rong and Netrapalli, Praneeth and Kakade, Sham M and Jordan, Michael I},
  booktitle={International conference on machine learning},
  pages={1724--1732},
  year={2017},
  organization={PMLR}
}

@inproceedings{yao2021adahessian,
  title={Adahessian: An adaptive second order optimizer for machine learning},
  author={Yao, Zhewei and Gholami, Amir and Shen, Sheng and Mustafa, Mustafa and Keutzer, Kurt and Mahoney, Michael},
  booktitle={proceedings of the AAAI conference on artificial intelligence},
  volume={35},
  number={12},
  pages={10665--10673},
  year={2021}
}

@inproceedings{
liu2024sophia,
title={Sophia: A Scalable Stochastic Second-order Optimizer for Language Model Pre-training},
author={Hong Liu and Zhiyuan Li and David Leo Wright Hall and Percy Liang and Tengyu Ma},
booktitle={The Twelfth International Conference on Learning Representations},
year={2024},
url={https://openreview.net/forum?id=3xHDeA8Noi}
}

@inproceedings{
zhang2025adam,
title={Adam-mini: Use Fewer Learning Rates To Gain More},
author={Yushun Zhang and Congliang Chen and Ziniu Li and Tian Ding and Chenwei Wu and Diederik P Kingma and Yinyu Ye and Zhi-Quan Luo and Ruoyu Sun},
booktitle={The Thirteenth International Conference on Learning Representations},
year={2025},
url={https://openreview.net/forum?id=iBExhaU3Lc}
}

@inproceedings{pennington2017geometry,
  title={Geometry of neural network loss surfaces via random matrix theory},
  author={Pennington, Jeffrey and Bahri, Yasaman},
  booktitle={International conference on machine learning},
  pages={2798--2806},
  year={2017},
  organization={PMLR}
}

@article{pennington2018spectrum,
  title={The spectrum of the fisher information matrix of a single-hidden-layer neural network},
  author={Pennington, Jeffrey and Worah, Pratik},
  journal={Advances in neural information processing systems},
  volume={31},
  year={2018}
}

@article{louart2018random,
  title={A random matrix approach to neural networks},
  author={Louart, Cosme and Liao, Zhenyu and Couillet, Romain},
  journal={The Annals of Applied Probability},
  volume={28},
  number={2},
  pages={1190--1248},
  year={2018},
  publisher={JSTOR}
}

@article{papyan2020traces,
  title={Traces of class/cross-class structure pervade deep learning spectra},
  author={Papyan, Vardan},
  journal={Journal of Machine Learning Research},
  volume={21},
  number={252},
  pages={1--64},
  year={2020}
}

@article{martin2021implicit,
  title={Implicit self-regularization in deep neural networks: Evidence from random matrix theory and implications for learning},
  author={Martin, Charles H and Mahoney, Michael W},
  journal={Journal of Machine Learning Research},
  volume={22},
  number={165},
  pages={1--73},
  year={2021}
}

@article{hodgkinson2025models,
  title={Models of heavy-tailed mechanistic universality},
  author={Hodgkinson, Liam and Wang, Zhichao and Mahoney, Michael W},
  journal={arXiv preprint arXiv:2506.03470},
  year={2025}
}

@inproceedings{keskar2017large,
  title={On Large-Batch Training for Deep Learning: Generalization Gap and Sharp Minima},
  author={Keskar, Nitish Shirish and Mudigere, Dheevatsa and Nocedal, Jorge and Smelyanskiy, Mikhail and Tang, Ping Tak Peter},
  booktitle={International Conference on Learning Representations},
  year={2017}
}

@article{li2018visualizing,
  title={Visualizing the loss landscape of neural nets},
  author={Li, Hao and Xu, Zheng and Taylor, Gavin and Studer, Christoph and Goldstein, Tom},
  journal={Advances in neural information processing systems},
  volume={31},
  year={2018}
}

@article{garipov2018loss,
  title={Loss surfaces, mode connectivity, and fast ensembling of dnns},
  author={Garipov, Timur and Izmailov, Pavel and Podoprikhin, Dmitrii and Vetrov, Dmitry P and Wilson, Andrew G},
  journal={Advances in neural information processing systems},
  volume={31},
  year={2018}
}


\appendix

\section{Proof of Main Theorem}

\subsection{Notation and Problem Setup}
\label{sec:notation_recall}

In this section, we recall the notation, problem setup, and key definitions from the main text that are essential for the subsequent theoretical analysis.

\paragraph{Notation and Matrix Operations.}
We follow standard conventions: scalars are denoted by lowercase letters (e.g., $\lambda, \eta$), vectors by bold lowercase letters (e.g., $\mathbf{x}, \mathbf{y}$), and matrices by capital letters (e.g., $W, H$). Let $I_k$ denote the $k \times k$ identity matrix.

For matrix products, we define $W^{k:l} = W^k \cdots W^l$ when $k > l$ and $W^{k:l} = (W^k)^\top \cdots (W^l)^\top$ when $k < l$. Additional operations include: $\otimes$ for Kronecker product, $\mathrm{vec}_r(\cdot)$ for row-wise matrix vectorization, and $\mathrm{rk}(W)$ for matrix rank.

For matrices $A \in \mathbb{R}^{m \times n}$ and $B \in \mathbb{R}^{p \times q}$ with $m \geq p$ and $n \geq q$, we define the zero-padding operator $\simeq$ as:
\begin{align*}
A \simeq B \quad \Leftrightarrow \quad A = \begin{bmatrix} B & \mathbf{0}_{p \times (n-q)} \\ \mathbf{0}_{(m-p) \times q} & \mathbf{0}_{(m-p) \times (n-q)} \end{bmatrix},
\end{align*}
where $\mathbf{0}_{i \times j}$ denotes an $i \times j$ zero matrix.

For matrix-matrix derivatives in row-wise vectorization form, we adopt the convention:
\begin{align*}
\frac{\partial A}{\partial B} := \frac{\partial \,\mathrm{vec}_{r}(A)}{\partial \,\mathrm{vec}_{r}(B)^{\top}}.
\end{align*}
A useful identity is:
\begin{align*}
\frac{\partial (AWB)}{\partial W} = A \otimes B^{\top}.
\end{align*}

\paragraph{Deep Linear Network Architecture.}
We consider a depth-$L$ deep linear neural network $F_W(\mathbf{x}) = W^{L:1}\mathbf{x}$, where $W^l \in \mathbb{R}^{d_l \times d_{l-1}}$ parameterizes the $l$-th layer for $l = 1, \ldots, L$. The network has input dimension $d_0$, output dimension $d_L$, and hidden layer widths $d_1, \ldots, d_{L-1}$. We denote the complete parameter set as $W := \{W^1, \ldots, W^L\}$ and define $d_\ast = \min\{d_0, d_L\}$.

The gradient of network output with respect to $W^l$ is:
\begin{align*}
\frac{\partial F}{\partial W^{l}} = W^{L:l+1} \otimes \big(\mathbf{x}^{\top}W^{1:l-1}\big) \in \mathbb{R}^{d_L \times d_l d_{l-1}}.
\end{align*}

\paragraph{Data Distribution and Loss Function.}
We consider a training dataset $\{(\mathbf{x}_i, \mathbf{y}_i)\}_{i=1}^N$ sampled i.i.d. from distribution $p$. The squared loss is $L_W(\mathbf{x}, \mathbf{y}) = \frac{1}{2} \|\mathbf{y} - F_W(\mathbf{x})\|_2^2$, and the population loss is $L(W) = \mathbb{E}_{(\mathbf{x},\mathbf{y}) \sim p}[L_W(\mathbf{x},\mathbf{y})]$. For brevity, we denote this expectation simply as $\mathbb{E}$.

We define the residual as $\boldsymbol{\delta}_{\mathbf{x},\mathbf{y}} := F_W(\mathbf{x}) - \mathbf{y}$ and introduce the following expected quantities:
\begin{align*}
\Omega := \mathbb{E}[\boldsymbol{\delta}_{\mathbf{x},\mathbf{y}}\mathbf{x}^{\top}], \quad
\Sigma_{xx} := \mathbb{E}[\mathbf{x}\mathbf{x}^{\top}], \quad
\Sigma_{yx} := \mathbb{E}[\mathbf{y}\mathbf{x}^{\top}].
\end{align*}

The gradient of the loss with respect to $W^l$ is:
\begin{align*}
\nabla_{W^l} L(W) = W^{l+1:L}\big(W^{L:1}\Sigma_{xx} - \Sigma_{yx}\big)W^{1:l-1}.
\end{align*}

\paragraph{Gauss-Newton Decomposition.}
Following \cite{singh2021analytic,zhao2024theoretical}, we decompose the Hessian $H_L = \nabla^2 L(W)$ as $H_L = H_o + H_f$, where:
\begin{align*}
H_o &= \mathbb{E} \left[ \nabla_{W} F(\mathbf{x})^\top \left[ \frac{\partial^2 L_W}{\partial F ^2} \right] \nabla_{W} F(\mathbf{x}) \right], \\
H_f &= \mathbb{E} \left[ \sum_{c=1}^{d_L} \left[ \frac{\partial L_W}{\partial F_c} \right] \nabla^2_{W} F_c(\mathbf{x}) \right].
\end{align*}
We refer to $H_o$ as the \textit{outer-product Hessian} and $H_f$ as the \textit{functional Hessian}.

\paragraph{Balanced Initialization.}
Following \cite{arora2019convergence}, we initialize the weight matrices using a balanced initialization (see Definition \ref{balance_init} in the main text). Specifically, we sample $A \sim \mathcal{D}$ over $d_L \times d_0$ matrices, compute its SVD $A = U \Sigma_0 V^{\top}$, and set:
\begin{align*}
W^L \simeq U \Sigma_0^{1/L}, \quad W^{l} \simeq \Sigma_0^{1/L} \text{ for } l=2,\ldots,L-1, \quad W^1 \simeq \Sigma_0^{1/L} V^{\top}.
\end{align*}

\paragraph{Gradient Descent Dynamics.}
We optimize using gradient descent with learning rate $\eta > 0$. Starting from $W_0$, the update rule at step $t$ is:
\begin{align*}
W_{t+1}^l = W_t^l - \eta \nabla_{W^l} L(W_t) = W_t^l - \eta \cdot W_t^{l+1:L} \big( W_t^{L:1}\Sigma_{xx} - \Sigma_{yx} \big) W_t^{1:l-1}.
\end{align*}

\paragraph{Additional Notation for Analysis.}
Let $r$ denote the effective rank of the target mapping, and define $\mathcal{I}_r \in \mathbb{R}^{d_\ast \times d_\ast}$ such that $\mathcal{I}_r \simeq I_r$. Let $U$ and $V$ denote the left and right singular vectors fixed at initialization, and let $\Sigma_t^{1/L}$ denote the dynamic diagonal matrix with entries $\lambda_{i,t}$ at step $t$.

We are here to give the proof of our theorem. First we need a few lemmas.
{\color{red}
Our goal is to characterize the structure of the full Hessian $H_{\mathcal{L}}$. Since the Hessian decomposes as $H_{\mathcal{L}} = H_o + H_f$, and the functional component $H_f$ vanishes as the loss approaches zero, the structure of the full Hessian is primarily determined by the outer-product Hessian $H_o$. Therefore, we begin our analysis from the structure of the outer-product Hessian $H_o$.  The following lemma shows that there exist two distinct cluster patters in the eigenvalues of $H_o$.
}
\begin{lemma}
\label{lemma 1}
    Under Assumptions \ref{assumption 12}, \ref{assumption 4}, and \ref{assumption 5}, and assuming $r \leq d_* = \min\{d_0, d_L\}$, the outer-product Hessian has two distinct nonzero eigenvalues $L\mu^{2(L-1)}$ and $\mu^{2(L-1)}$, where the former is exactly $L$ times the latter. Moreover, they occur with multiplicities $r^2$ and $(d_0 + d_L - 2r)r$, respectively.
\end{lemma}

\begin{proof}
    According to \cite{singh2021analytic}, for a deep linear network,
    \begin{align}
        H_o = A_o B_o A_o^\top
    \end{align}
    where $B_o = I_{d_L} \otimes \Sigma_{xx} \in \mathbb{R}^{d_L d_0 \times d_L d_0}$, and
    \begin{align}
        A_o^\top = \left( W^{L:2} \otimes I_{d_0}, \cdots, W^{L:l+1} \otimes W^{1:l-1}, \cdots, I_{d_L} \otimes W^{1:L-1} \right)
    \end{align}
    
    Under Assumption \ref{assumption 12}, $\Sigma_{xx} \simeq I_{d_*}$. When $d_* = d_0 \leq d_L$, we have $\Sigma_{xx} = I_{d_0}$, so $B_o = I_{d_L} \otimes I_{d_0} = I_{d_L d_0}$. In this case, the eigenvalues of $H_o = A_o A_o^\top$ are identical to those of $A_o^\top A_o$.
    
    Under balanced initialization (Initialization \ref{balance_init}) with Assumption \ref{assumption 5}, we have $\Sigma_0^{1/L} \simeq \mu I_r$. The weight matrices take the form:
    \begin{align}
        W^L \simeq U \Sigma_0^{1/L}, \quad W^{L-1} \simeq \Sigma_0^{1/L}, \quad \ldots, \quad W^2 \simeq \Sigma_0^{1/L}, \quad W^1 \simeq \Sigma_0^{1/L} V^\top
    \end{align}
    where $U \in \mathbb{R}^{d_L \times d_*}$ and $V \in \mathbb{R}^{d_0 \times d_*}$ have orthonormal columns. Since $\Sigma_0^{1/L} \simeq \mu I_r$ has only $r$ nonzero entries, let $U_1 \in \mathbb{R}^{d_L \times r}$ and $V_1 \in \mathbb{R}^{d_0 \times r}$ denote the first $r$ columns of $U$ and $V$, corresponding to the nonzero singular values.
    
    We first compute the weight matrix products. For $W^{L:2}$:
    \begin{align}
        W^{L:2} &= W^L W^{L-1} \cdots W^2 \simeq (U \Sigma_0^{1/L})(\Sigma_0^{1/L}) \cdots (\Sigma_0^{1/L}) = U (\Sigma_0^{1/L})^{L-1} \simeq \mu^{L-1} U_1
    \end{align}
    Therefore:
    \begin{align}
        W^{L:2}(W^{L:2})^\top = \mu^{2(L-1)} U_1 U_1^\top
    \end{align} 
    
    Similarly, for $W^{1:L-1}$:
    \begin{align}
        W^{1:L-1} &= (W^1)^\top (W^2)^\top \cdots (W^{L-1})^\top \simeq V ((\Sigma_0^{1/L})^\top)^{L-1} \simeq \mu^{L-1} V_1
    \end{align}
    Therefore:
    \begin{align}
        W^{1:L-1}(W^{1:L-1})^\top = \mu^{2(L-1)} V_1 V_1^\top
    \end{align}
    For $2 \leq l \leq L-1$:
    \begin{align}
        W^{L:l+1} &\simeq \mu^{L-l} U_1, \quad W^{1:l-1} \simeq \mu^{l-1} V_1
    \end{align}
    Therefore:
    \begin{align}
        W^{L:l+1}(W^{L:l+1})^\top &= \mu^{2(L-l)} U_1 U_1^\top, \quad W^{1:l-1}(W^{1:l-1})^\top = \mu^{2(l-1)} V_1 V_1^\top
    \end{align}
  
    Now we compute the Gram matrix $A_o^\top A_o$:
    \begin{align}
        A_o^\top A_o &= (W^{L:2} \otimes I_{d_0})(W^{L:2} \otimes I_{d_0})^\top + (I_{d_L} \otimes W^{1:L-1})(I_{d_L} \otimes W^{1:L-1})^\top \notag \\
        &\quad + \sum_{l=2}^{L-1} (W^{L:l+1} \otimes W^{1:l-1})(W^{L:l+1} \otimes W^{1:l-1})^\top
    \end{align}
    
    Using the Kronecker product property $(A \otimes B)(C \otimes D) = (AC) \otimes (BD)$:
\textcolor{red}{
\begin{align}
    (W^{L:2} \otimes I_{d_0})(W^{L:2} \otimes I_{d_0})^\top &= W^{L:2}(W^{L:2})^\top \otimes I_{d_0} =\mu^{2(L-1)} U_1 U_1^\top \otimes I_{d_0} \\
    (I_{d_L} \otimes W^{1:L-1})(I_{d_L} \otimes W^{1:L-1})^\top &= I_{d_L} \otimes W^{1:L-1}(W^{1:L-1})^\top = \mu^{2(L-1)} I_{d_L} \otimes V_1 V_1^\top \\
    (W^{L:l+1} \otimes W^{1:l-1})(W^{L:l+1} \otimes W^{1:l-1})^\top \notag 
    &= W^{L:l+1}(W^{L:l+1})^\top \otimes W^{1:l-1}(W^{1:l-1})^\top \notag \\
    &= \mu^{2(L-l)} U_1 U_1^\top \otimes \mu^{2(l-1)} V_1 V_1^\top \notag \\
    &= \mu^{2(L-1)} U_1 U_1^\top \otimes V_1 V_1^\top
\end{align}
}
    Substituting back:
    \begin{align}
        A_o^\top A_o &= \mu^{2(L-1)} U_1 U_1^\top \otimes I_{d_0} + \mu^{2(L-1)} I_{d_L} \otimes V_1 V_1^\top + (L-2) \mu^{2(L-1)} U_1 U_1^\top \otimes V_1 V_1^\top \notag \\
        &= \mu^{2(L-1)} \left[ U_1 U_1^\top \otimes I_{d_0} + I_{d_L} \otimes V_1 V_1^\top + (L-2) U_1 U_1^\top \otimes V_1 V_1^\top \right]
    \end{align}
    
    To analyze the eigenvalues, let $\{u_1, \ldots, u_r\}$ be the columns of $U_1$, extended to an orthonormal basis $\{u_1, \ldots, u_{d_L}\}$ for $\mathbb{R}^{d_L}$. Similarly, let $\{v_1, \ldots, v_r\}$ be the columns of $V_1$, extended to an orthonormal basis $\{v_1, \ldots, v_{d_0}\}$ for $\mathbb{R}^{d_0}$. The vectors $u_i \otimes v_j$ for $i = 1, \ldots, d_L$ and $j = 1, \ldots, d_0$ form an orthonormal basis for $\mathbb{R}^{d_L d_0}$.
    
    Note that $U_1 U_1^\top u_i = u_i$ if $i \leq r$ and $U_1 U_1^\top u_i = 0$ if $i > r$. Similarly, $V_1 V_1^\top v_j = v_j$ if $j \leq r$ and $V_1 V_1^\top v_j = 0$ if $j > r$. We compute the action of $A_o^\top A_o$ on these basis vectors for each case:
    
    \textbf{(Dominant space)} If $i \leq r$ and $j \leq r$:
    \begin{align}
        A_o^\top A_o (u_i \otimes v_j) &= \mu^{2(L-1)} \left[ (U_1 U_1^\top u_i) \otimes v_j + u_i \otimes (V_1 V_1^\top v_j) + (L-2) (U_1 U_1^\top u_i) \otimes (V_1 V_1^\top v_j) \right] \notag \\
        &= \mu^{2(L-1)} \left[ u_i \otimes v_j + u_i \otimes v_j + (L-2) u_i \otimes v_j \right] \notag \\
        &= L \mu^{2(L-1)} (u_i \otimes v_j)
    \end{align}

    This gives eigenvalue $L\mu^{2(L-1)}$ with multiplicity $r \times r = r^2$.
    
    \textbf{(Bulk space, part 1)} If $i \leq r$ and $j > r$:
    \begin{align}
        A_o^\top A_o (u_i \otimes v_j) &= \mu^{2(L-1)} \left[ (U_1 U_1^\top u_i) \otimes v_j + u_i \otimes (V_1 V_1^\top v_j) + (L-2) (U_1 U_1^\top u_i) \otimes (V_1 V_1^\top v_j) \right] \notag \\
        &= \mu^{2(L-1)} \left[ u_i \otimes v_j + u_i \otimes 0 + (L-2) u_i \otimes 0 \right] \notag \\
        &= \mu^{2(L-1)} (u_i \otimes v_j)
    \end{align}
    This gives eigenvalue $\mu^{2(L-1)}$ with multiplicity $r \times (d_0 - r) = r(d_0 - r)$.
    
    \textbf{(Bulk space, part 2)} If $i > r$ and $j \leq r$:
    \begin{align}
        A_o^\top A_o (u_i \otimes v_j) &= \mu^{2(L-1)} \left[ (U_1 U_1^\top u_i) \otimes v_j + u_i \otimes (V_1 V_1^\top v_j) + (L-2) (U_1 U_1^\top u_i) \otimes (V_1 V_1^\top v_j) \right] \notag \\
        &= \mu^{2(L-1)} \left[ 0 \otimes v_j + u_i \otimes v_j + (L-2) \cdot 0 \otimes v_j \right] \notag \\
        &= \mu^{2(L-1)} (u_i \otimes v_j)
    \end{align}
    This gives eigenvalue $\mu^{2(L-1)}$ with multiplicity $(d_L - r) \times r = (d_L - r)r$.
    
    \textbf{(Zero space)} If $i > r$ and $j > r$:
    \begin{align}
        A_o^\top A_o (u_i \otimes v_j) &= \mu^{2(L-1)} \left[ 0 \otimes v_j + u_i \otimes 0 + (L-2) \cdot 0 \otimes 0 \right] = 0
    \end{align}
    This gives eigenvalue $0$ with multiplicity $(d_L - r) \times (d_0 - r)$.
    
    Combining the two parts of the bulk space, the total multiplicity is:
    \begin{align}
        r(d_0 - r) + (d_L - r)r = (d_0 + d_L - 2r)r
    \end{align}
    
    We conclude that the outer-product Hessian has two distinct nonzero eigenvalues $L\mu^{2(L-1)}$ and $\mu^{2(L-1)}$, occurring with multiplicities $r^2$ and $(d_0 + d_L - 2r)r$, respectively. The total number of nonzero eigenvalues is $r^2 + (d_0 + d_L - 2r)r = (d_0 + d_L - r)r$.
\end{proof}

{\color{red}
After establishing the spectral structure of $H_o$ for a fixed weight matrices, we now turn to study the dynamic behavior of weight matrices. Lemma \ref{lemma 2} derives the update rules for the singular values and vectors, showing that the singular vectors remain invariant while the singular values manifest two distinct dynamic patterns.

\begin{lemma}
\label{lemma 2}
Under Assumptions \ref{assumption 12} and \ref{assumption 4}, the left and right singular vectors of the weight matrices $W_k$ for all layers $k \in \{1, \ldots, L\}$ remain invariant during training, regardless of the dimensions $d_0, d_L$, and $d_*$.
The dynamics of the singular values are as follows:
\begin{enumerate}
    \item For $i = 1, \ldots, r$, the singular values $\lambda_{i,t}$ follow the update rule:
    \begin{align}
        \lambda_{i,t+1} = \lambda_{i,t} - \eta \lambda_{i,t}^{2L-1} + \eta \lambda_{i,t}^{L-1} \label{eq:lambda_update}
    \end{align}
    \item The remaining $d_* - r$ singular values $\lambda_{i,t}$ for $i = r + 1, \ldots , d_*$ follow the update rule
    \begin{align}
        \lambda_{i,t+1} = \lambda_{i,t}-\eta\lambda_{i,t}^{2L-1}.
    \end{align}
\end{enumerate}
\end{lemma}

\begin{proof}
    Under Assumption \ref{assumption 12}, since $d_* \leq \min\{d_{L-1}, d_{L-2}, \ldots, d_1\}$, according to the balanced initialization procedure:
    \begin{align*}
        \forall \, 1 < k \leq L-1, \, k \in \mathbb{N}, \quad W^k_0 &\simeq \Sigma^{1/L} \\
        W^1_0 &\simeq \Sigma^{1/L} V^\top \\
        W^L_0 &\simeq U \Sigma^{1/L}
    \end{align*}
    
    Then we consider three cases.
    
    \begin{enumerate}
        \item For $1 < k < L$, considering the dynamic behavior between step $t = 0$ and step $t = 1$:
        \begin{align}
            \nabla_{W^k} L(W^k_0) &= W^{k+1:L}_0 (W^{L:1}_0 \Sigma_{xx} - \Sigma_{yx}) W^{1:k-1}_0 \notag \\
            &\simeq \Sigma^{(L-k)/L} U^\top (U \Sigma V^\top \Sigma_{xx} - \Sigma_{yx}) V \Sigma^{(k-1)/L} \label{eq 1}
        \end{align}
        
        Under Assumption \ref{assumption 12}, $\Sigma_{xx} \simeq I_r$. Since the balanced initialization chooses $V$ such that its column space lies within the support of $\Sigma_{xx}$, therefore, 
        \begin{align}
            V^\top \Sigma_{xx} = V^\top
        \end{align}
        which implies $U \Sigma V^\top \Sigma_{xx} = U \Sigma V^\top$.

        According to Assumption \ref{assumption 4}, $\Sigma_{yx} = U \mathcal{I}_rV^\top$, where $\mathcal{I}_r\in \mathbb{R}^{d_*\times d_*}$ (i.e, $\mathcal{I}_r = \text{diag}(\underbrace{1, \dots, 1}_{r}, \underbrace{0, \dots, 0}_{d_* - r})$). Note that $U^\top U = I_{d_*}$ and $V^\top V = I_{d_*}$. Substituting into \cref{eq 1}:
        \begin{align*}
            \nabla_{W^k} L(W^k_0) &\simeq \Sigma^{(L-k)/L} U^\top (U \Sigma V^\top - U V^\top) V \mathcal{I}_r\Sigma^{(k-1)/L} \\
            &= \Sigma^{(L-k)/L} U^\top U \left(\Sigma -  \mathcal{I}_r\right) V^\top V \Sigma^{(k-1)/L} \\
            &= \Sigma^{(L-k)/L} \left(\Sigma -  \mathcal{I}_r\right) \Sigma^{(k-1)/L} \\
            &= \Sigma^{(2L-1)/L} - \mathcal{I}_r\Sigma^{(L-1)/L}
        \end{align*}
        
        Since $\Sigma$ is a diagonal matrix, $\Sigma^{(2L-1)/L} - \mathcal{I}_r\Sigma^{(L-1)/L}$ is also diagonal. The update rule $W^k_{t+1} = W^k_t - \eta \nabla_{W^k} L(W^k_t)$ only changes the diagonal elements.

        
        For example, suppose  
\begin{equation}
    W^k_0 \simeq \Sigma^{1/L} = 
    \begin{bmatrix} 
        \lambda_{1,0} & & 0 \\
         & \ddots & \\
        0 & & \lambda_{d_*,0}
    \end{bmatrix}
\end{equation}
Then for the update rule $W^k_{t+1} = W^k_t - \eta \nabla_{W^k} L(W^k_t)$:
\begin{equation}
    W^k_1 \simeq
    \left[
    \begin{array}{c|c}
        \begin{matrix} 
            \lambda_{1,0} - \eta \lambda_{1,0}^{2L-1} + \mathbf{\eta \lambda_{1,0}^{L-1}} & & 0 \\
             & \ddots & \\
            0 & & \lambda_{r,0} - \eta \lambda_{r,0}^{2L-1} + \mathbf{\eta \lambda_{r,0}^{L-1}}
        \end{matrix} & \mathbf{0} \\
        \hline
        \mathbf{0} & \begin{matrix} 
            \lambda_{r+1,0} - \eta \lambda_{r+1,0}^{2L-1} & & 0 \\
             & \ddots & \\
            0 & & \lambda_{d_*,0} - \eta \lambda_{d_*,0}^{2L-1}
        \end{matrix}
    \end{array}
    \right].
\end{equation}

        Thus the $i$-th singular value follows the dynamic
        \begin{equation} \label{eq2}
    \lambda_{i,t+1} = 
    \begin{cases}
        \lambda_{i,t} - \eta \lambda_{i,t}^{2L-1} + \eta \lambda_{i,t}^{L-1} & \text{if } 1 \le i \le r, \\[10pt]
       \lambda_{i,t} - \eta \lambda_{i,t}^{2L-1} & \text{if } r < i \le d_*.
    \end{cases}
\end{equation}
        and both the left and right singular vectors remain invariant.
        
        \item For $k = 1$, considering the dynamic behavior between step $t = 0$ and step $t = 1$:
        \begin{align*}
            \nabla_{W^1} L(W^1_0) &= W^{2:L}_0 (W^{L:1}_0 \Sigma_{xx} - \Sigma_{yx}) \\
            &\simeq \Sigma^{(L-1)/L} U^\top (U \Sigma V^\top - U \mathcal{I}_rV^\top) \\
            &= \Sigma^{(L-1)/L} U^\top U (\Sigma - \mathcal{I}_r) V^\top \\
            &= \Sigma^{(L-1)/L} (\Sigma - \mathcal{I}_r) V^\top \\
            &= (\Sigma^{(2L-1)/L} - \mathcal{I}_r\Sigma^{(L-1)/L}) V^\top
        \end{align*}
        
        For the update rule $W^1_{t+1} = W^1_t - \eta \nabla_{W^1} L(W^1_t)$, since $W^1_0 \simeq \Sigma^{1/L} V^\top$:
        \begin{align*}
            W^1_1 &\simeq \Sigma^{1/L} V^\top - \eta (\Sigma^{(2L-1)/L} - \mathcal{I}_r\Sigma^{(L-1)/L}) V^\top \\
            &= \left[ \Sigma^{1/L} - \eta (\Sigma^{(2L-1)/L} - \mathcal{I}_r\Sigma^{(L-1)/L}) \right] V^\top
        \end{align*}
        
        We can separate $V^\top$, thus the singular values also follow \cref{eq2} and both the left and right singular vectors (determined by $V^\top$) remain invariant.
        
        \item For $k = L$, considering the dynamic behavior between step $t = 0$ and step $t = 1$:
        \begin{align*}
            \nabla_{W^L} L(W^L_0) &= (W^{L:1}_0 \Sigma_{xx} - \Sigma_{yx}) W^{1:L-1}_0 \\
            &\simeq (U \Sigma V^\top - U \mathcal{I}_rV^\top) V \Sigma^{(L-1)/L} \\
            &= U (\Sigma - \mathcal{I}_r) V^\top V \Sigma^{(L-1)/L} \\
            &= U (\Sigma - \mathcal{I}_r) \Sigma^{(L-1)/L} \\
            &= U (\Sigma^{(2L-1)/L} -\mathcal{I}_r \Sigma^{(L-1)/L})
        \end{align*}
        For the update rule $W^L_{t+1} = W^L_t - \eta \nabla_{W^L} L(W^L_t)$, since $W^1_0 \simeq U\Sigma^{1/L} $:
        \begin{align*}
            W^L_1 &\simeq U\Sigma^{1/L}  - \eta U(\Sigma^{(2L-1)/L} - \mathcal{I}_r\Sigma^{(L-1)/L}) \\
            &= U\left[ \Sigma^{1/L} - \eta (\Sigma^{(2L-1)/L} - \mathcal{I}_r\Sigma^{(L-1)/L}) \right] 
        \end{align*}
        We can separate $U$, thus the singular values also follow \cref{eq2} and both the left and right singular vectors (determined by $U$) remain invariant.
    \end{enumerate}
    Combining these cases completes the proof.
\end{proof}
}

{\color{red}
The update rules derived in \cref{lemma 2} indicate that the two singular value components evolve differently. The following lemma reveals that they exhibit distinct asymptotic behavior as well.
\begin{lemma}
\label{lemma eigenvalues convergence}
    Under Assumption \ref{assumption 12} and \ref{assumption 4}, for initialization $\lambda_{i,0} \in (0, M]$, $i = 1, 2, \ldots, r$, where $M = \max_{i = 1, 2, \ldots, r}\{1, \lambda_{i,0}\}$ and proper $\eta < \min\left\{\frac{1}{L}, \frac{1}{M^{2L-2}}\right\}$, the asymptotic behavior of $\lambda_{i,t}$ will converge to $1$ for all $i = 1, 2, \ldots, r$, i.e., $\lim_{t \to \infty} \lambda_{i,t} = 1$. The remaining $d_* - r$ singular values $\lambda_{i,t}$ for $i = r+1, \ldots, d_*$ converge to 0, i.e., $\lim_{t \to \infty} \lambda_{i,t} = 0$. 
\end{lemma}

\begin{proof}
    According to Lemma \ref{lemma 2}, the gradient of the loss with respect to $W^k$ involves $\Sigma_{xx}$. Under Assumption \ref{assumption 12}, $\Sigma_{xx} \simeq I_r$, which means only the first $r$ directions of the input space contribute to the gradient.
    
    More precisely, for $1 < k < L$, the gradient is:
    \begin{align*}
        \nabla_{W^k} L(W^k_t) &= W^{k+1:L}_t (W^{L:1}_t \Sigma_{xx} - \Sigma_{yx}) W^{1:k-1}_t
    \end{align*}
    
    Since $\Sigma_{xx} \simeq I_r$ and the balanced initialization chooses $V$ such that its first $r$ rows $V_1 \in \mathbb{R}^{r \times d_*}$ satisfy $V_1^\top V_1 = I_r$ (with $V_2 = \mathbf{0}$), the effective gradient only updates the first $r$ eigenvalues of $\Sigma^{1/L}$.
    
    For the first $r$ eigenvalues, let $i = 1, 2, \ldots, r$ be fixed and arbitrary. The update rule is:
    \begin{align}
        \lambda_{i,t+1} = \lambda_{i,t} - \eta \lambda_{i,t}^{2L-1} + \eta \lambda_{i,t}^{L-1} \label{eq: eigenvalue update}
    \end{align}
    
    Now we prove convergence for $i = 1, 2, \ldots, r$. First, we show that for all $t \in \mathbb{N}$ and $i = 1, 2, \ldots, r$, $\lambda_{i,t}$ is bounded by $(0, M]$. The initialization satisfies this condition. For induction, assume $\lambda_{i,t} \in (0, M]$. We consider the following cases:
    
    \textbf{Case 1: $\lambda_{i,t} \in (0, 1)$.} In this regime, the update rule \cref{eq: eigenvalue update} implies $\lambda_{i,t+1} \geq \lambda_{i,t} > 0$ since $0 < \lambda_{i,t} \leq 1$.
    
    If $\lambda_{i,t+1} > 1$, it implies: \tianyu{ We better mention the definition of  $M $}
    \begin{align*}
        \eta > \frac{1 - \lambda_{i,t}}{\lambda_{i,t}^{L-1}(1 - \lambda_{i,t}^L)}
    \end{align*}
    
    Note that $1 - \lambda_{i,t}^L = (1 - \lambda_{i,t})(1 + \lambda_{i,t} + \lambda_{i,t}^2 + \cdots + \lambda_{i,t}^{L-1})$. Using the fact $\lambda_{i,t} < 1$, therefore $\eta > \frac{1 - \lambda_{i,t}}{\lambda_{i,t}^{L-1}(1 - \lambda_{i,t}^L)}$ implies:
    \begin{align*}
        \eta > \frac{1}{\lambda_{i,t}^{L-1}(1 + \lambda_{i,t} + \lambda_{i,t}^2 + \cdots + \lambda_{i,t}^{L-1})} \geq \frac{1}{L}
    \end{align*}
    
    Since $\eta < \min\left\{\frac{1}{L}, \frac{1}{M^{2L-2}}\right\}$ and $M = \max_{i = 1, 2, \ldots, r}\{1, \lambda_{i,0}\}$, thus $0 < \lambda_{i,t+1} \leq 1\leq M$ as well. 
    
    \textbf{Case 2: $\lambda_{i,t}=1 $.} In this case, the update rule \cref{eq: eigenvalue update} implies that $\lambda_{i,t+k}=1 $ for all $k\in \mathbb{N}$. Therefore, $0<\lambda_{i,t+1}\leq M$ also holds.

    \textbf{Case 3: $\lambda_{i,t} \in (1, M]$.} In this case, the update rule \cref{eq: eigenvalue update} implies $\lambda_{i,t+1} < \lambda_{i,t}$ since $\lambda_{i,t} > 1$ and $L \geq 1$. Using the induction assumption $\lambda_{i,t} \leq M$ yields $\lambda_{i,t+1} \leq M$ as well. Moreover, since $\eta < \min\left\{\frac{1}{L}, \frac{1}{M^{2L-2}}\right\} \leq \frac{1}{M^{2L-2}}$:
    \begin{align*}
        \lambda_{i,t+1} &= \lambda_{i,t}(1 - \eta \lambda_{i,t}^{2L-2} + \eta \lambda_{i,t}^{L-2}) 
        > \lambda_{i,t}(1 - \eta \lambda_{i,t}^{2L-2}) \geq \lambda_{i,t}(1 - \eta M^{2L-2}) > 0
    \end{align*}
    
    Therefore, in both cases, we have shown $\lambda_{i,t+1} \in (0, M]$. By the principle of induction, it can be concluded that for all $t \in \mathbb{N}$, $\lambda_{i,t}$ is bounded by $(0, M]$.
    
    Next we will prove $\lim_{t \to \infty} \lambda_{i,t} = 1$ for all $i=1,2,\cdots,r$. 
    
    If there exists some $t\in \mathbb{N}$ such that $\lambda_{i,t}=1 $. Then in this case, the update rule \cref{eq: eigenvalue update} implies that $\lambda_{i,t+k}=1 $ for all $k\in \mathbb{N}$. Therefore, $\lim_{t \to \infty} \lambda_{i,t} = 1$. 

    Then suppose that for all $t\in \mathbb{N}, \,\lambda_{i,t}\neq 1$, Based on the initial value $\lambda_{i,0}$, we consider the following two cases:

    \textbf{Case 1: The sequence enters or starts in $(0, 1)$.}
    Suppose there exists some time $T \ge 0$ such that $\lambda_{i,T} \in (0, 1)$. Then, as shown above, $\lambda_{i,t} \in (0, 1)$ for all $t \ge T$.
    Therefore, the sequence $\{\lambda_{i,t}\}_{t=T}^\infty$ is non-decreasing and has a upper bound $1$. By the Monotone Convergence Theorem, it converges to a limit $\lambda^* \in [0, 1]$.
    Taking the limit on both sides of \cref{eq: eigenvalue update} yields:
    \begin{align}
        \lambda^* = \lambda^* + \eta (\lambda^*)^{L-1}(1 - (\lambda^*)^L).
    \end{align}
    which implies that $\lambda^*=0$ or $\lambda^*=1$. Since $\lambda_{i,T}>0$ and the sequence $\{\lambda_{i,t}\}_{t=T}^\infty$ is non-decreasing, $\lambda^*=0$ is impossible. Therefore $\lambda^* = 1$, that is: $\lim_{t\to \infty} \lambda_{i,t}=1$.

    \textbf{Case 2: The sequence never enters $(0, 1)$.}
    Suppose that for all $t \in \mathbb{N}$, $\lambda_{i,t} \in (1, M]$. 
    In this regime, the sequence $\{\lambda_{i,t}\}_{t=1}^\infty$ is strictly decreasing and has a lower bound below by $1$. Similarly, by the Monotone Convergence Theorem, it can be shown that $\{\lambda_{i,t}\}_{t=1}^\infty$ converges to a limit $1$.
    
    Therefore, for all cases, for any initialization $\lambda_{i,0} \in (0, M]$, the sequence converges to $1$. Since $i$ is arbitrary, it holds for all $i = 1, \ldots, r$.

    Finally, let $i = r+1, \ldots, d_*$ be fixed and arbitrary, we will show $\lambda_{i,t}$ converges to 0, i.e., $\lim_{t \to \infty} \lambda_{i,t} = 0$.

    According to Lemma \ref{lemma 2}, the update rule \cref{eq2} is:
    \begin{align}
     \lambda_{i,t+1}=\lambda_{i,t} - \eta \lambda_{i,t}^{2L-1}
    \end{align}
    Therefore, the sequence $\{\lambda_{i,t}\}$ is strictly decreasing and has a lower bound below by $0$. Applying Monotone Convergence Theorem again shows it converges to a limit, denoted by $\tilde{\lambda}$. Then, taking limit for both sides of update rule \cref{eq2} simultaneously yields
    \begin{align*}
        \tilde{\lambda} = \tilde{\lambda} - \eta (\tilde{\lambda})^{2L-1}
    \end{align*} 
    Since $L\geq 1$, we can conclude that $\tilde{\lambda}'=0$, which implies $\lim_{t\to \infty}\lambda_{i,t}=0$.

    Since $i$ is arbitrary, it holds for all $i = r+1, \ldots, d_*$, which completes the proof.
\end{proof}
}

{\color{red}
To establish a linear convergence rate for the loss, we need a uniform lower bound for the first $r$-th singular values. The following lemma establishes that the the first $r$-th singular values are strictly bounded away from zero during training
\begin{lemma}
\label{lemma eigenvalues lower bound}
     Under Assumption \ref{assumption 12} and \ref{assumption 4}, for initialization $\lambda_{i,0} \in (0, M]$, $i = 1, 2, \ldots, r$, where $M = \max_{i = 1, 2, \ldots, r}\{1, \lambda_{i,0}\}$. Suppose the step size satisfies $\eta < \min \left\{ \frac{1}{L}, \frac{1}{M^{2L-2}} \right\}$, then the sequence $\{\lambda_{i,t}\}_{t=0}^\infty$ is strictly bounded away from zero. That is, there exists a constant $C_{\min} > 0$ such that $\lambda_{i,t} \geq C_{\min}$ for all $t \in \mathbb{N}$ and $i = 1, 2, \ldots, r$. Furthermore, for all $t \in \mathbb{N}$, $\min_{i = 1, \ldots, r} (\lambda_{i,t}^{2L-2})$ has a positive lower bound $\alpha > 0$ that is independent of $t$.
\end{lemma}

\begin{proof}
    According to \cref{lemma eigenvalues convergence}, for all $i =1, \ldots, r$, the sequence $\{\lambda_{i,t}\}$ converges to $1$, i.e., $\lim_{t \to \infty} \lambda_{i,t} = 1$.
    
    By the definition of the limit, taking $\epsilon = \frac{1}{2}$, there exists $T_i > 0$ such that for all $t \geq T_i$, we have
    \begin{align}
        |\lambda_{i,t} - 1| < \frac{1}{2} 
    \end{align}
    which implies $\lambda_{i,t} > \frac{1}{2}$.

    Recall that as shown in \cref{lemma eigenvalues convergence}, $\lambda_{i,t} > 0$ holds for all $t\in \mathbb{N}$ and $i=1,2,\cdots, r$. Let $T = \max_{i=1,\ldots,r} T_i$. Note that T is finite, thus, the set 
    \begin{align}
        S_{\text{finite}} = \{ \lambda_{i,t} \mid 1 \le i \le r, \, 0 \le t < T \}.
    \end{align}
    is also finite, which implies $c_{\text{finite}} = \min(S_{\text{finite}})> 0$. 
    
    Let $C_{\min} = \min \left\{ c_{\text{finite}}, \frac{1}{2} \right\}$. It follows that $C_{\min} > 0$ and $\lambda_{i,t} \geq C_{\min}$ for all $t\in \mathbb{N}$ and $i=1,2,\cdots r$.
    
    Consequently, $\min_{i=1,\ldots,r} (\lambda_{i,t}^{2L-2}) \ge C_{\min}^{2L-2}:=\alpha > 0$ and $\alpha$ is independent of $t$, which completes the proof.
\end{proof}
}

{\color{red}
With the dynamics of the weights and $H_o$ analyzed, we proceed to consider the functional Hessian $H_f$. Before the analysis, we introduce a useful inequality.
}
\begin{lemma}
\label{lemma:block_matrix_norm}
Let $H_f$ be a block matrix partitioned into blocks ${\color{red}H_f^{kl}}$, where $k, l \in \{1, \dots, L\}$. Then the spectral norm of $H_f$ satisfies
\begin{equation*}
    \|H_f\|_2 \le \sqrt{\sum_{k,l} \|{\color{red}H_f^{kl}}\|_2^2}.
\end{equation*}
\end{lemma}

\begin{proof}
For any unit vector $\mathbf{x}$ with $\|\mathbf{x}\|_2=1$, we partition $\mathbf{x}$ conformably with the blocks of $H_f$ as $\mathbf{x} = [\mathbf{x}_1^\top, \dots, \mathbf{x}_L^\top]^\top$. The condition $\|\mathbf{x}\|_2=1$ implies that $\sum_{l=1}^L \|\mathbf{x}_l\|_2^2 = 1$.

Let $\mathbf{y} = H_f \mathbf{x}$. The $k$-th block of $\mathbf{y}$, denoted by $\mathbf{y}_k$, is given by:
\begin{equation*}
    \mathbf{y}_k = \sum_{l=1}^L {\color{red}H_f^{kl}} \, \mathbf{x}_l.
\end{equation*}

By the triangle inequality and the compatibility of the spectral norm, the norm of $\mathbf{y}_k$ is bounded by:
\begin{equation*}
    \|\mathbf{y}_k\|_2 = \left\| \sum_{l=1}^L {\color{red}H_f^{kl}} \, \mathbf{x}_l \right\|_2 \le \sum_{l=1}^L \|{\color{red}H_f^{kl}} \, \mathbf{x}_l\|_2 \le \sum_{l=1}^L \|{\color{red}H_f^{kl}}\|_2 \|\mathbf{x}_l\|_2.
\end{equation*}

We compute the squared norm of the entire vector $\mathbf{y}$:
\begin{equation*}
    \|H_f \mathbf{x}\|_2^2 = \|\mathbf{y}\|_2^2 = \sum_{k=1}^L \|\mathbf{y}_k\|_2^2 \le \sum_{k=1}^L \left( \sum_{l=1}^L \|{\color{red}H_f^{kl}}\|_2 \|\mathbf{x}_l\|_2 \right)^2.
\end{equation*}

Applying the Cauchy-Schwarz inequality to the inner sum:
\begin{align*}
    \left( \sum_{l=1}^L \|{\color{red}H_f^{kl}}\|_2 \|\mathbf{x}_l\|_2 \right)^2 &\le \left( \sum_{l=1}^L \|{\color{red}H_f^{kl}}\|_2^2 \right) \left( \sum_{l=1}^L \|\mathbf{x}_l\|_2^2 \right) 
    = \left( \sum_{l=1}^L \|{\color{red}H_f^{kl}}\|_2^2 \right) \cdot 1.
\end{align*}

Substituting this back:
\begin{equation*}
    \|H_f \mathbf{x}\|_2^2 \le \sum_{k=1}^L \sum_{l=1}^L \|{\color{red}H_f^{kl}}\|_2^2.
\end{equation*}

Since the spectral norm is defined as $\|H_f\|_2 = \sup_{\|\mathbf{x}\|_2=1} \|H_f \mathbf{x}\|_2$, taking the square root of both sides yields the desired upper bound.
\end{proof}

\tianyu{I think it will be better for us   to  mention the specific form of $H_f$ here to help people recall.}
{\color{red}
After introducing the inequality, we begin to analyze the functional Hessian $H_f$. The following \cref{lemma H_f decay with loss} bounds the 2-norm of the functional Hessian $\|H_f\|_2$ in terms of the loss.
}
\begin{lemma}
\label{lemma H_f decay with loss}
    Under Assumption \ref{assumption 12} and \ref{assumption 4}, when $L_W(\mathbf{x}, \mathbf{y}) = \frac{1}{2} \|\mathbf{y} - \hat{\mathbf{y}}\|_2^2 < \epsilon$, the 2-norm of functional Hessian $\|H_f\|_2$ is bounded by
    \begin{align*}
        \|H_f\|_2 \leq \|\Sigma_t^{1/L}\|_2^{L-2} \sqrt{2L(L-1)r} \, \epsilon^{1/2} = O(\epsilon^{\frac{1}{2}})
    \end{align*}
\end{lemma}

\begin{proof}
    According to \cite{singh2021analytic}, the functional Hessian $H_f$ can be written as a block matrix, where
    \begin{align*}
        &\forall \, k < l, \quad H_f^{kl} = W_t^{k+1:l-1} \otimes W_t^{k-1:1} \Omega^\top W_t^{L:l+1} \\
        &\forall \, k > l, \quad H_f^{kl} = W_t^{k+1:L} \Omega W_t^{1:l-1} \otimes W_t^{k-1:l+1}
    \end{align*}
    and the diagonal entries are $0$.
    
    By Lemma \ref{lemma 2}, the eigenvectors of the weight matrices are invariant during training. Under balanced initialization with $\Sigma_t^{1/L} = \text{diag}(\lambda_{1,t}, \ldots, \lambda_{d_*,t})$, we have:
    \begin{align*}
        W_t^k &\simeq \Sigma_t^{1/L}, \quad \text{for } 1 < k < L \\
        W_t^1 &\simeq \Sigma_t^{1/L} V^\top \\
        W_t^L &\simeq U \Sigma_t^{1/L}
    \end{align*}
    
    Note that for $k < l$:
    {\color{red}
    \begin{align*}
        &W_t^{k+1:l-1} \otimes W_t^{k-1:1} \Omega^\top W_t^{L:l+1} \\
        &\simeq (\Sigma_t^{1/L})^\top \cdots (\Sigma_t^{1/L})^\top  \otimes \Sigma_t^{1/L} \cdots \left( \Sigma_t^{1/L} V^\top \right) \Omega^\top \left( U \Sigma_t^{1/L} \right) \cdots \Sigma_t^{1/L} \\
        &= \Sigma_t^{(l-k-1)/L} \otimes \Sigma_t^{(k-1)/L} V^\top \Omega^\top U \Sigma_t^{(L-l)/L}
    \end{align*}
    }
    And for $k > l$:
    {\color{red}
    \begin{align*}
        &W_t^{k+1:L} \Omega W_t^{1:l-1} \otimes W_t^{k-1:l+1} \\
        &\simeq \left( U \Sigma_t^{1/L} \right)^\top \cdots (\Sigma_t^{1/L})^\top \Omega \,(\Sigma_t^{1/L} V^\top)^\top \Sigma_t^{1/L} \cdots (\Sigma_t^{1/L})^\top \otimes \Sigma_t^{1/L} \cdots \Sigma_t^{1/L} \\
        &=  \Sigma_t^{(L-k)/L} U^\top \Omega V \Sigma_t^{(l-1)/L} \otimes \Sigma_t^{(k-l-1)/L}
    \end{align*}
    }
    For $k < l$, we have $(l-k-1) + (k-1) + (L-l) = L-2$. For $k > l$, we have $(L-k) + (l-1) + (k-l-1) = L-2$. Using the submultiplicativity of the spectral norm $\|AB\|_2 \leq \|A\|_2 \|B\|_2$ and the property $\|A \otimes B\|_2 = \|A\|_2 \|B\|_2$, along with the fact that $\|\Sigma_t^{s/L}\|_2 = \|\Sigma_t^{1/L}\|_2^s$ and $\|U\|_2 = \|V\|_2 = 1$ (since $U$ and $V$ have orthonormal columns), we obtain:
    \begin{align*}
        \|H_f^{kl}\|_2 \leq \|\Sigma_t^{1/L}\|_2^{L-2} \cdot \|\Omega\|_2
    \end{align*}
    
    Recall that under Assumption \ref{assumption 12}, $\Sigma_{xx} \simeq I_r$. Since the balanced initialization chooses $V$ such that its column space lies within the support of $\Sigma_{xx}$, we have:
    \begin{align*}
        \mathbb{E}[\|\mathbf{x}\|_2^2] = \mathbb{E}[\text{Tr}(\mathbf{x} \mathbf{x}^\top)] = \text{Tr}(\mathbb{E}[\mathbf{x} \mathbf{x}^\top]) = \text{Tr}(\Sigma_{xx}) = r
    \end{align*}
    
    Recall that $\Omega = \mathbb{E}[\boldsymbol{\delta}_{\mathbf{x},\mathbf{y}} \mathbf{x}^\top]$, where $\boldsymbol{\delta}_{\mathbf{x},\mathbf{y}} = \hat{\mathbf{y}} - \mathbf{y}$. Using Jensen's inequality, we have:
    \begin{align*}
        \|\Omega\|_2 = \|\mathbb{E}[\boldsymbol{\delta}_{\mathbf{x},\mathbf{y}} \mathbf{x}^\top]\|_2 \leq \mathbb{E}[\|\boldsymbol{\delta}_{\mathbf{x},\mathbf{y}} \mathbf{x}^\top\|_2] = \mathbb{E}[\|\boldsymbol{\delta}_{\mathbf{x},\mathbf{y}}\|_2 \|\mathbf{x}\|_2]
    \end{align*}
    
    Applying the Cauchy-Schwarz inequality for expectations yields:
    \begin{align*}
        \mathbb{E}[\|\boldsymbol{\delta}_{\mathbf{x},\mathbf{y}}\|_2 \|\mathbf{x}\|_2] \leq \sqrt{\mathbb{E}[\|\boldsymbol{\delta}_{\mathbf{x},\mathbf{y}}\|_2^2] \cdot \mathbb{E}[\|\mathbf{x}\|_2^2]} < \sqrt{2r\epsilon}
    \end{align*}
    which implies $\|\Omega\|_2 \leq \sqrt{2r\epsilon}$.
    
    Therefore, using $\|H_f\|_2 \leq \sqrt{\sum_{k,l} \|H_f^{kl}\|_2^2}$:
    \begin{align*}
        \|H_f\|_2 \leq \|\Sigma_t^{1/L}\|_2^{L-2} \sqrt{2L(L-1)r\epsilon} = \|\Sigma_t^{1/L}\|_2^{L-2} \sqrt{2L(L-1)r} \, \epsilon^{1/2}
    \end{align*}
    which completes the proof.
\end{proof}

{\color{red}
\cref{lemma H_f decay with loss} shows that $\|H_f\|_2$ scales with the square root of the loss. Therefore, to establish the decay rate of $H_f$, it suffices to determine the convergence rate of the population loss. The following lemma establishes the linear convergence rate by using the bound derived in \cref{lemma eigenvalues lower bound}.
}
\begin{lemma}\label{lemma:loss_convergence}
    Under Assumption \ref{assumption 12} and \ref{assumption 4}, and assuming the first $r$ rows of $V$, denoted $V_1 \in \mathbb{R}^{r \times d_*}$, satisfy $V_1 V_1^\top = I_r$, the population loss $L(W_t)$ converges linearly to its minimum possible value. Specifically, let {\color{red}$L_{\min} = \frac{1}{2} \mathbb{E}[\|\mathbf{y}\|_2^2] - \frac{r}{2}$} be the irreducible error, the excess loss $\tilde{L}_t = L(W_t) - L_{\min}$ satisfies:
    \begin{align}
        \tilde{L}_t \leq \tilde{L}_0 \cdot e^{-2 L\alpha \eta t} = O\left(e^{-2 L\alpha \eta t}\right)
    \end{align}
    {\color{red}
    where $\alpha > 0$ is the uniform lower bound established in \cref{lemma eigenvalues lower bound}, satisfying $\min_{1,\cdots,r}(\lambda_{i,t}^{2L-2}) \geq \alpha$ for all $t \in \mathbb{N}$. \tianyu{This notation may be strange, why there is a $t$  in min here?}
    }
\end{lemma}

\begin{proof}
    Recall that the population loss is $L(W_t) = \frac{1}{2} \mathbb{E}_{\mathbf{x}, \mathbf{y} \sim \mu} \left[ \|\mathbf{y} - \hat{\mathbf{y}}_t\|_2^2 \right]$. Thus,
    \begin{align*}
        L(W_t) &= \frac{1}{2} \mathbb{E}_{\mathbf{x}, \mathbf{y} \sim \mu} \left[ \mathbf{y}^\top \mathbf{y} - 2\mathbf{y}^\top \hat{\mathbf{y}}_t + \hat{\mathbf{y}}_t^\top \hat{\mathbf{y}}_t \right] \\
        &= \frac{1}{2} \mathbb{E}_{\mathbf{x}, \mathbf{y} \sim \mu}[\|\mathbf{y}\|^2] - \mathbb{E}_{\mathbf{x}, \mathbf{y} \sim \mu}[\text{Tr}(\mathbf{y}^\top \hat{\mathbf{y}}_t)] + \frac{1}{2} \mathbb{E}_{\mathbf{x}, \mathbf{y} \sim \mu}[\text{Tr}(\hat{\mathbf{y}}_t \hat{\mathbf{y}}_t^\top)]
    \end{align*}
    
    Note that
    \begin{align*}
        \mathbb{E}_{\mathbf{x}, \mathbf{y} \sim \mu}[\text{Tr}(\mathbf{y}^\top \hat{\mathbf{y}}_t)] &= \mathbb{E}_{\mathbf{x}, \mathbf{y} \sim \mu}[\text{Tr}(\mathbf{y}^\top W^{L:1}_t \mathbf{x})] \\
        &= \text{Tr}(W^{L:1}_t \mathbb{E}_{\mathbf{x}, \mathbf{y} \sim \mu}[\mathbf{x} \mathbf{y}^\top]) \\
        &= \text{Tr}(W^{L:1}_t \Sigma_{yx}^\top)
    \end{align*}
    and
    \begin{align*}
        \mathbb{E}_{\mathbf{x}, \mathbf{y} \sim \mu}[\text{Tr}(\hat{\mathbf{y}}_t \hat{\mathbf{y}}_t^\top)] &= \mathbb{E}_{\mathbf{x}, \mathbf{y} \sim \mu}[\text{Tr}(W^{L:1}_t \mathbf{x} \mathbf{x}^\top (W^{L:1}_t)^\top)] \\
        &= \text{Tr}(W^{L:1}_t \mathbb{E}_{\mathbf{x}, \mathbf{y} \sim \mu}[\mathbf{x} \mathbf{x}^\top] (W^{L:1}_t)^\top) \\
        &= \text{Tr}(W^{L:1}_t \Sigma_{xx} (W^{L:1}_t)^\top)
    \end{align*}
    
    By Lemma \ref{lemma 2}, the eigenvectors of the weight matrices are invariant during training. Thus, $W^{L:1}_t = U \Sigma_t V^\top$, where $\Sigma_t = \text{diag}(\lambda_{1,t}^L, \ldots, \lambda_{d_*,t}^L)$.
    
    Under Assumption \ref{assumption 4}, $\Sigma_{yx} = U{\color{red}\mathcal{I}_r} V^\top$, so $\Sigma_{yx}^\top = V {\color{red}\mathcal{I}_r}U^\top$. Since $V^\top V = I_{d_*}$ and $U^\top U = I_{d_*}$:
    \begin{align*}
        \text{Tr}(W^{L:1}_t \Sigma_{yx}^\top) &= \text{Tr}(U \Sigma_t V^\top V{\color{red}\mathcal{I}_r} U^\top) = \text{Tr}(U \Sigma_t {\color{red}\mathcal{I}_r}U^\top) = \text{Tr}(\Sigma_t{\color{red}\mathcal{I}_r}) = \sum_{i=1}^{{\color{red}r}} \lambda_{i,t}^L
    \end{align*}
    
    Under Assumption \ref{assumption 12}, $\Sigma_{xx} \simeq I_r$. Partition $V \in \mathbb{R}^{d_0 \times d_*}$ by rows as $V = \begin{pmatrix} V_1 \\ V_2 \end{pmatrix}$ where $V_1 \in \mathbb{R}^{r \times d_*}$. Then:
    \begin{align*}
        V^\top \Sigma_{xx} V = V_1^\top V_1
    \end{align*}
    
    Under the alignment condition $V_1 V_1^\top = I_r$, the matrix $P := V_1^\top V_1 \in \mathbb{R}^{d_* \times d_*}$ is a projection matrix with rank $r$. Choosing coordinates such that $P = \text{diag}(I_r, \mathbf{0}_{(d_* - r) \times (d_* - r)})$, we have:
    \begin{align*}
        \text{Tr}(W^{L:1}_t \Sigma_{xx} (W^{L:1}_t)^\top) &= \text{Tr}(U \Sigma_t V^\top \Sigma_{xx} V \Sigma_t U^\top) 
        = \text{Tr}(\Sigma_t P \Sigma_t) = \sum_{i=1}^{r} \lambda_{i,t}^{2L}
    \end{align*}
    
    Substituting back into the loss equation:
    \begin{align*}
        L(W_t) &= \frac{1}{2} \mathbb{E}_{\mathbf{x}, \mathbf{y} \sim \mu}[\|\mathbf{y}\|^2] - \sum_{i=1}^{{\color{red}r}} \lambda_{i,t}^L + \frac{1}{2} \sum_{i=1}^{r} \lambda_{i,t}^{2L}
    \end{align*}
    
    By Lemma \ref{lemma eigenvalues convergence}, the first $r$ eigenvalues $\lambda_{1,t}, \ldots, \lambda_{r,t}$ are updated according to the gradient dynamics and converge to $1$, while the remaining $d_* - r$ eigenvalues $\lambda_{r+1,t}, \ldots, \lambda_{d_*,t}$ remain at their initial values $\lambda_{r+1,0}, \ldots, \lambda_{d_*,0}$.
    
    {\color{red} Note that:}
    {\color{red}
    \begin{align*}
       L(W_t) &= \frac{1}{2} \mathbb{E}[\|\mathbf{y}\|^2] - \sum_{i=1}^{r} \lambda_{i,t}^L + \frac{1}{2} \sum_{i=1}^{r} \lambda_{i,t}^{2L}
       \\&= \frac{1}{2} \mathbb{E}[\|\mathbf{y}\|^2]+ \frac{1}{2} \sum_{i=1}^{r} (\lambda_{i,t}^{2L} - 2\lambda_{i,t}^L + 1) - \frac{r}{2} \\
        &= \frac{1}{2} \mathbb{E}[\|\mathbf{y}\|^2]  - \frac{r}{2} + \frac{1}{2} \sum_{i=1}^{r} (1 - \lambda_{i,t}^L)^2
    \end{align*}
    }
    
    The minimum loss is achieved when $\lambda_{i,t}^L = 1$ for $i = 1, \ldots, r$:
    {\color{red}
    \begin{align*}
        L_{\min} = \frac{1}{2} \mathbb{E}[\|\mathbf{y}\|^2] - \frac{r}{2}
    \end{align*}
    }
    
    Thus, the excess loss is:
    \begin{align*}
        \tilde{L}_t = L(W_t) - L_{\min} = \frac{1}{2} \sum_{i=1}^{r} (1 - \lambda_{i,t}^L)^2
    \end{align*}
    
    According to Lemma \ref{lemma 2}, for each $i = 1, \ldots, r$, the eigenvalue update rule is:
    \begin{align*}
        \lambda_{i,t+1} = \lambda_{i,t} - \eta \lambda_{i,t}^{2L-1} + \eta \lambda_{i,t}^{L-1}
    \end{align*}
    
    Note that the function $f(x) = x^L$ is convex when $x > 0$. Thus, by the definition of convex function, for all $1 \leq i \leq r$:
    \begin{align*}
        \lambda_{i,t+1}^L \geq \lambda_{i,t}^L + L \lambda_{i,t}^{L-1} (\lambda_{i,t+1} - \lambda_{i,t}) = \lambda_{i,t}^L + L \eta \lambda_{i,t}^{L-1} (\lambda_{i,t}^{L-1} - \lambda_{i,t}^{2L-1})
    \end{align*}
    which implies:
    \begin{align*}
        1 - \lambda_{i,t+1}^L &\leq 1 - \left[ \lambda_{i,t}^L + L \eta \lambda_{i,t}^{L-1} (\lambda_{i,t}^{L-1} - \lambda_{i,t}^{2L-1}) \right] \\
        &= (1 - \lambda_{i,t}^L) - L \eta \lambda_{i,t}^{2L-2} (1 - \lambda_{i,t}^L) \\
        &= (1 - L \eta \lambda_{i,t}^{2L-2}) (1 - \lambda_{i,t}^L)
    \end{align*}
    
    Therefore,
    \begin{align*}
        (1 - \lambda_{i,t+1}^L)^2 \leq (1 - L \eta \lambda_{i,t}^{2L-2})^2 (1 - \lambda_{i,t}^L)^2
    \end{align*}
    
    Using the fact $(1 - x)^2 \leq e^{-2x}$ for $x \in [0, 1]$ yields:
    \begin{align*}
        (1 - \lambda_{i,t+1}^L)^2 \leq e^{-2 L \eta \lambda_{i,t}^{2L-2}} (1 - \lambda_{i,t}^L)^2
    \end{align*}
    
    {\color{red}
    Recall that from \cref{lemma eigenvalues lower bound}, $\alpha$ is a positive lower bound of $\min_{i = 1, \ldots, r} (\lambda_{i,t}^{2L-2})$ for all $t\in \mathbb{N}$ \tianyu{need to mention Lemma~\ref{lemma eigenvalues lower bound}}. Summing $i$ from $1$ to $r$ yields:}
    \begin{align*}
        \tilde{L}_{t+1} = \frac{1}{2} \sum_{i=1}^{r} (1 - \lambda_{i,t+1}^L)^2 \leq e^{-2 L \alpha \eta} \cdot \frac{1}{2} \sum_{i=1}^{r} (1 - \lambda_{i,t}^L)^2 = e^{-2 L \alpha \eta} \tilde{L}_t
    \end{align*}
    
    Repeating this process yields:
    \begin{align*}
        \tilde{L}_t \leq e^{-2 L \alpha \eta t} \tilde{L}_0 = O\left( e^{-2 L \alpha \eta t} \right)
    \end{align*}
    which completes the proof.
\end{proof}





{\color{red}
Combining the bound on $\|H_f\|_2$ and the linear convergence, we can state the decay rate of the functional Hessian spectral norm with respect to training time $t$, which is shown in \cref{lemma: 2norm-t}.
}
\begin{lemma}\label{lemma: 2norm-t}
    Under Assumptions \ref{assumption 12} and \ref{assumption 4}, and assuming $r \leq d_*$, the spectral norm of the functional Hessian $\|H_{f,t}\|_2$ decays exponentially with respect to $t$. Specifically:
    \begin{align}
        \|H_{f,t}\|_2 = O\left( \epsilon^{1/2} \right) = O\left( e^{-L \alpha \eta t} \right),
    \end{align}
    {\color{red}
    where $\alpha > 0$ is the uniform lower bound established in \cref{lemma eigenvalues lower bound}, satisfying $\min_{1,\cdots,r}(\lambda_{i,t}^{2L-2}) \geq \alpha$ for all $t \in \mathbb{N}$.}
\end{lemma}

\begin{proof}
    Under Assumption \ref{assumption 12}, $\Sigma_{xx} \simeq I_r$, which implies that only the first $r$ eigenvalue directions are effectively updated during gradient descent (as shown in Lemma \ref{lemma eigenvalues convergence}). The remaining $d_* - r$ eigenvalues remain at their initial values and do not contribute to the loss decrease.
    
    By Lemma \ref{lemma H_f decay with loss}, the functional Hessian norm satisfies
    \begin{align*}
        \|H_{f,t}\|_2 \leq \|\Sigma_t^{1/L}\|_2^{L-2} \sqrt{2L(L-1)r} \, \epsilon^{1/2}
    \end{align*}
    {\color{red}
    where $\epsilon = L(W_t)$ is the population excess loss at time $t$ and $\|\Sigma_t^{1/L}\|_2^{L-2}$ is bounded.
    }
    
    By Lemma \ref{lemma:loss_convergence}, the excess loss satisfies
    \begin{align*}
        \tilde{L}_t \leq \tilde{L}_0 \cdot e^{-2L\alpha\eta t}
    \end{align*}
    {\color{red}
    where $\alpha$ is a positive lower bound of $\min_{i = 1, \ldots, r} (\lambda_{i,t}^{2L-2})$ for all $t\in \mathbb{N}$.
    }
    
    Combining these two results yields
    \begin{align*}
        \|H_{f,t}\|_2 = O\left( \epsilon^{1/2} \right) = O\left( e^{-L \alpha \eta t} \right)
    \end{align*}
    which completes the proof.
\end{proof}

{\color{red}
We now return back to the analysis of the outer-product Hessian $H_{o,t}$. Unlike the simplified setting where layer weights are identical scalars, we consider a generalized case where the singular values $\lambda_{i,t}$ may vary within a small interval. The following lemma establishes that the eigenvalues of $H_{o,t}$ still exhibit a distinct two-cluster structure.}

\begin{lemma}\label{lemma generalize eigenvalues}
    Under Assumptions \ref{assumption 12} and \ref{assumption 4}, and assuming $r \leq d_* = \min\{d_0, d_L\}$, suppose that at time $t$, $\Sigma_t^{1/L} = \text{diag}(\lambda_{1,t}, \lambda_{2,t}, \ldots, \lambda_{r,t}, 0, \ldots, 0)$. Suppose $\lambda_{i,t} \in [m_t - \delta_t, m_t + \delta_t]$ for $i = 1, 2, \ldots, r$ where $\delta_t > 0$. As long as $\frac{m_t + \delta_t}{m_t - \delta_t} < L^{\frac{1}{2(L-1)}}$, the outer product Hessian $H_{o,t}$ exhibits a three-part spectral structure:
    \begin{enumerate}
        \item \textbf{(Dominant space)} $r^2$ eigenvalues lying in $[L(m_t - \delta_t)^{2(L-1)}, L(m_t + \delta_t)^{2(L-1)}]$.
        \item \textbf{(Bulk space)} $(d_0 + d_L - 2r)r$ eigenvalues lying in $[(m_t - \delta_t)^{2(L-1)}, (m_t + \delta_t)^{2(L-1)}]$.
        \item \textbf{(Zero space)} $(d_L - r)(d_0 - r)$ eigenvalues equal to zero.
    \end{enumerate}
    The total number of nonzero eigenvalues is $(d_0 + d_L - r)r$.
\end{lemma}

\begin{proof}
    According to \cite{singh2021analytic}, for a deep linear network at time $t$,
    \begin{align*}
        H_{o,t} = A_{o,t} B_o A_{o,t}^\top
    \end{align*}
    where $B_o = I_{d_L} \otimes \Sigma_{xx} \in \mathbb{R}^{d_L d_0 \times d_L d_0}$, and
    \begin{align*}
        A_{o,t}^\top = \left( W_t^{L:2} \otimes I_{d_0}, \ldots, W_t^{L:l+1} \otimes W_t^{1:l-1}, \ldots, I_{d_L} \otimes W_t^{1:L-1} \right)
    \end{align*}
    
    Under Assumption \ref{assumption 12}, $\Sigma_{xx} \simeq I_{d_*}$. When $d_* = d_0 \leq d_L$, we have $\Sigma_{xx} = I_{d_0}$, so $B_o = I_{d_L} \otimes I_{d_0} = I_{d_L d_0}$. In this case, the eigenvalues of $H_{o,t} = A_{o,t} A_{o,t}^\top$ are identical to those of $A_{o,t}^\top A_{o,t}$.
    
    By Lemma \ref{lemma 2}, the eigenvectors of the weight matrices are invariant during training. Thus, at time $t$, we have $W_t^{L:1} = U \Sigma_t V^\top$, where $\Sigma_t = \text{diag}(\lambda_{1,t}^L, \ldots, \lambda_{r,t}^L, 0, \ldots, 0)$ and $\Sigma_t^{1/L} = \text{diag}(\lambda_{1,t}, \ldots, \lambda_{r,t}, 0, \ldots, 0)$.
    
    Let $U_1 \in \mathbb{R}^{d_L \times r}$ and $V_1 \in \mathbb{R}^{d_0 \times r}$ denote the first $r$ columns of $U$ and $V$, corresponding to the nonzero singular values. We compute the weight matrix products at time $t$:
    \begin{align*}
        W_t^{L:2}(W_t^{L:2})^\top &{\color{red}=} U_1 \, \text{diag}(\lambda_{1,t}^{2(L-1)}, \ldots, \lambda_{r,t}^{2(L-1)}) \, U_1^\top \\
        W_t^{1:L-1}(W_t^{1:L-1})^\top &{\color{red}=}  V_1 \, \text{diag}(\lambda_{1,t}^{2(L-1)}, \ldots, \lambda_{r,t}^{2(L-1)}) \, V_1^\top
    \end{align*}
    And for all $2 \leq l \leq L-1$,
    \begin{align*}
        W_t^{L:l+1}(W_t^{L:l+1})^\top &{\color{red}=}  U_1 \, \text{diag}(\lambda_{1,t}^{2(L-l)}, \ldots, \lambda_{r,t}^{2(L-l)}) \, U_1^\top \\
        W_t^{1:l-1}(W_t^{1:l-1})^\top &{\color{red}=}  V_1 \, \text{diag}(\lambda_{1,t}^{2(l-1)}, \ldots, \lambda_{r,t}^{2(l-1)}) \, V_1^\top
    \end{align*}
    
    Computing the Gram matrix $A_{o,t}^\top A_{o,t}$:
    \begin{align*}
        A_{o,t}^\top A_{o,t} &= W_t^{L:2}(W_t^{L:2})^\top \otimes I_{d_0} + I_{d_L} \otimes W_t^{1:L-1}(W_t^{1:L-1})^\top \\
        &\quad + \sum_{l=2}^{L-1} W_t^{L:l+1}(W_t^{L:l+1})^\top \otimes W_t^{1:l-1}(W_t^{1:l-1})^\top
    \end{align*}
    
    Let $\Lambda_U^{(l)} = \text{diag}(\lambda_{1,t}^{2(L-l)}, \ldots, \lambda_{r,t}^{2(L-l)})$ and $\Lambda_V^{(l)} = \text{diag}(\lambda_{1,t}^{2(l-1)}, \ldots, \lambda_{r,t}^{2(l-1)})$. Substituting:
    \begin{align*}
        A_{o,t}^\top A_{o,t} &= U_1 \Lambda_U^{(1)} U_1^\top \otimes I_{d_0} + I_{d_L} \otimes V_1 \Lambda_V^{(L)} V_1^\top + \sum_{l=2}^{L-1} U_1 \Lambda_U^{(l)} U_1^\top \otimes V_1 \Lambda_V^{(l)} V_1^\top
    \end{align*}
    
    To analyze the eigenvalues, let $\{u_1, \ldots, u_r\}$ be the columns of $U_1$, extended to an orthonormal basis $\{u_1, \ldots, u_{d_L}\}$ for $\mathbb{R}^{d_L}$. Similarly, let $\{v_1, \ldots, v_r\}$ be the columns of $V_1$, extended to an orthonormal basis $\{v_1, \ldots, v_{d_0}\}$ for $\mathbb{R}^{d_0}$. The vectors $u_i \otimes v_j$ form an orthonormal basis for $\mathbb{R}^{d_L d_0}$.
    
    Note that $U_1 \Lambda_U^{(l)} U_1^\top u_i = \lambda_{i,t}^{2(L-l)} u_i$ if $i \leq r$ and $U_1 \Lambda_U^{(l)} U_1^\top u_i = 0$ if $i > r$. Similarly for $V_1 \Lambda_V^{(l)} V_1^\top$. We compute the action of $A_{o,t}^\top A_{o,t}$ on these basis vectors for each case:
    
    \textbf{(Dominant space)} If $i \leq r$ and $j \leq r$:
    \begin{align*}
        A_{o,t}^\top A_{o,t} (u_i \otimes v_j) &= \lambda_{i,t}^{2(L-1)} (u_i \otimes v_j) + \lambda_{j,t}^{2(L-1)} (u_i \otimes v_j) + \sum_{l=2}^{L-1} \lambda_{i,t}^{2(L-l)} \lambda_{j,t}^{2(l-1)} (u_i \otimes v_j) \\
        &= \left( \sum_{l=1}^{L} \lambda_{i,t}^{2(L-l)} \lambda_{j,t}^{2(l-1)} \right) (u_i \otimes v_j)
    \end{align*}
    Denote the eigenvalue as $\nu_{i,j} = \sum_{l=1}^{L} \lambda_{i,t}^{2(L-l)} \lambda_{j,t}^{2(l-1)}$. When $i = j$, this equals $L \lambda_{i,t}^{2(L-1)}$. When $i \neq j$ and $\lambda_{i,t} \neq \lambda_{j,t}$, this equals $\frac{\lambda_{i,t}^{2L} - \lambda_{j,t}^{2L}}{\lambda_{i,t}^2 - \lambda_{j,t}^2}$. In all cases, since $\lambda_{i,t}, \lambda_{j,t} \in [m_t - \delta_t, m_t + \delta_t]$:
    \begin{align*}
        L(m_t - \delta_t)^{2(L-1)} \leq \nu_{i,j} \leq L(m_t + \delta_t)^{2(L-1)}
    \end{align*}
    This gives eigenvalue multiplicity $r \times r = r^2$.
    
    \textbf{(Bulk space, part 1)} If $i \leq r$ and $j > r$:
    \begin{align*}
        A_{o,t}^\top A_{o,t} (u_i \otimes v_j) &= \lambda_{i,t}^{2(L-1)} (u_i \otimes v_j) + 0 + 0 = \lambda_{i,t}^{2(L-1)} (u_i \otimes v_j)
    \end{align*}
    The eigenvalue is $\nu_{i,j} = \lambda_{i,t}^{2(L-1)} \in [(m_t - \delta_t)^{2(L-1)}, (m_t + \delta_t)^{2(L-1)}]$. This gives eigenvalue multiplicity $r \times (d_0 - r) = r(d_0 - r)$.
    
    \textbf{(Bulk space, part 2)} If $i > r$ and $j \leq r$:
    \begin{align*}
        A_{o,t}^\top A_{o,t} (u_i \otimes v_j) &= 0 + \lambda_{j,t}^{2(L-1)} (u_i \otimes v_j) + 0 = \lambda_{j,t}^{2(L-1)} (u_i \otimes v_j)
    \end{align*}
    The eigenvalue is $\nu_{i,j} = \lambda_{j,t}^{2(L-1)} \in [(m_t - \delta_t)^{2(L-1)}, (m_t + \delta_t)^{2(L-1)}]$. This gives eigenvalue multiplicity $(d_L - r) \times r = (d_L - r)r$.
    
    \textbf{(Zero space)} If $i > r$ and $j > r$:
    \begin{align*}
        A_{o,t}^\top A_{o,t} (u_i \otimes v_j) = 0
    \end{align*}
    This gives eigenvalue $0$ with multiplicity $(d_L - r) \times (d_0 - r) = (d_L - r)(d_0 - r)$.
    
    Combining the two parts of the bulk space, the total multiplicity is:
    \begin{align*}
        r(d_0 - r) + (d_L - r)r = (d_0 + d_L - 2r)r
    \end{align*}
    
    The condition $\frac{m_t + \delta_t}{m_t - \delta_t} < L^{\frac{1}{2(L-1)}}$ is equivalent to
    \begin{align*}
        (m_t + \delta_t)^{2(L-1)} < L (m_t - \delta_t)^{2(L-1)}
    \end{align*}
    which guarantees a spectral gap between the dominant space and bulk space:
    \begin{align*}
        \min_{i,j \leq r} \nu_{i,j} \geq L(m_t - \delta_t)^{2(L-1)} > (m_t + \delta_t)^{2(L-1)} \geq \max_{\substack{i \leq r, j > r \\ \text{or } i > r, j \leq r}} \nu_{i,j}
    \end{align*}
    
    In summary, the outer product Hessian $H_{o,t}$ has:
    \begin{itemize}
        \item $r^2$ eigenvalues in the dominant space, lying in $[L(m_t - \delta_t)^{2(L-1)}, L(m_t + \delta_t)^{2(L-1)}]$;
        \item $(d_0 + d_L - 2r)r$ eigenvalues in the bulk space, lying in $[(m_t - \delta_t)^{2(L-1)}, (m_t + \delta_t)^{2(L-1)}]$;
        \item $(d_L - r)(d_0 - r)$ zero eigenvalues.
    \end{itemize}
    The total number of nonzero eigenvalues is $r^2 + (d_0 + d_L - 2r)r = (d_0 + d_L - r)r$.
\end{proof}

{\color{red}
In addition to the eigenvalues, we also characterize the eigenvectors for the generalized case.
}

\begin{lemma}
\label{lemma eigenvectors outer product hessian}
    Under Assumptions \ref{assumption 12} and \ref{assumption 4}, and assuming $r \leq d_* = \min\{d_0, d_L\}$, suppose that $\{u_1, u_2, \ldots, u_{d_L}\}$ and $\{v_1, v_2, \ldots, v_{d_0}\}$ are the columns of initialized matrices $U$ (extended to an orthonormal basis of $\mathbb{R}^{d_L}$) and $V$ (extended to an orthonormal basis of $\mathbb{R}^{d_0}$), respectively. Then the eigenvectors of the outer product Hessian at time $t$ corresponding to non-zero eigenvalues can be written as
    \begin{align*}
        A_{o,t} (u_i \otimes v_j), \quad \text{for } (i \leq r \text{ and } j \leq d_0) \text{ or } (i > r \text{ and } j \leq r)
    \end{align*}
    where $A_{o,t}^\top = \left( W_t^{L:2} \otimes I_{d_0}, \ldots, W_t^{L:l+1} \otimes W_t^{1:l-1}, \ldots, I_{d_L} \otimes W_t^{1:L-1} \right)$ as defined in \cite{singh2021analytic}. Moreover:
    \begin{itemize}
        \item The $r^2$ eigenvectors $A_{o,t}(u_i \otimes v_j)$ for $i \leq r$ and $j \leq r$ span the dominant eigenspace.
        \item The $(d_0 + d_L - 2r)r$ eigenvectors $A_{o,t}(u_i \otimes v_j)$ for ($i \leq r$ and $j > r$) or ($i > r$ and $j \leq r$) span the bulk eigenspace.
    \end{itemize}
\end{lemma}

\begin{proof}
    Recall from Lemma \ref{lemma generalize eigenvalues} that under Assumption \ref{assumption 12}, $\Sigma_{xx} \simeq I_{d_*}$. When $d_* = d_0 \leq d_L$, we have $\Sigma_{xx} = I_{d_0}$, so $B_o = I_{d_L d_0}$. The eigenvalues of $H_{o,t} = A_{o,t} A_{o,t}^\top$ are identical to those of $A_{o,t}^\top A_{o,t}$.
    
    Let $U_1 \in \mathbb{R}^{d_L \times r}$ and $V_1 \in \mathbb{R}^{d_0 \times r}$ denote the first $r$ columns of $U$ and $V$, corresponding to the nonzero singular values. The Gram matrix $A_{o,t}^\top A_{o,t}$ can be expressed as:
    \begin{align*}
        A_{o,t}^\top A_{o,t} &= U_1 \Lambda_U^{(1)} U_1^\top \otimes I_{d_0} + I_{d_L} \otimes V_1 \Lambda_V^{(L)} V_1^\top + \sum_{l=2}^{L-1} U_1 \Lambda_U^{(l)} U_1^\top \otimes V_1 \Lambda_V^{(l)} V_1^\top
    \end{align*}
    where $\Lambda_U^{(l)} = \text{diag}(\lambda_{1,t}^{2(L-l)}, \ldots, \lambda_{r,t}^{2(L-l)})$ and $\Lambda_V^{(l)} = \text{diag}(\lambda_{1,t}^{2(l-1)}, \ldots, \lambda_{r,t}^{2(l-1)})$.
    
    We show that $u_i \otimes v_j$ are eigenvectors of $A_{o,t}^\top A_{o,t}$. Using the Kronecker product property $(A \otimes B)(x \otimes y) = (Ax) \otimes (By)$, and noting that:
    \begin{enumerate}
        \item $U_1 \Lambda_U^{(l)} U_1^\top u_i = \lambda_{i,t}^{2(L-l)} u_i$ if $i \leq r$, and $U_1 \Lambda_U^{(l)} U_1^\top u_i = \mathbf{0}$ if $i > r$.
        \item $V_1 \Lambda_V^{(l)} V_1^\top v_j = \lambda_{j,t}^{2(l-1)} v_j$ if $j \leq r$, and $V_1 \Lambda_V^{(l)} V_1^\top v_j = \mathbf{0}$ if $j > r$.
    \end{enumerate}
    
    We analyze the following cases:
    
    \textbf{Case 1 (Dominant space):} If $i \leq r$ and $j \leq r$, then
    \begin{align*}
        A_{o,t}^\top A_{o,t} (u_i \otimes v_j) &= \lambda_{i,t}^{2(L-1)} (u_i \otimes v_j) + \lambda_{j,t}^{2(L-1)} (u_i \otimes v_j) + \sum_{l=2}^{L-1} \lambda_{i,t}^{2(L-l)} \lambda_{j,t}^{2(l-1)} (u_i \otimes v_j) \\
        &= \left( \sum_{l=1}^{L} \lambda_{i,t}^{2(L-l)} \lambda_{j,t}^{2(l-1)} \right) (u_i \otimes v_j)
    \end{align*}
    This gives $r \times r = r^2$ eigenvectors spanning the dominant eigenspace.
    
    \textbf{Case 2 (Bulk space, part 1):} If $i \leq r$ and $j > r$, then
    \begin{align*}
        A_{o,t}^\top A_{o,t} (u_i \otimes v_j) &= \lambda_{i,t}^{2(L-1)} (u_i \otimes v_j) + 0 + 0 = \lambda_{i,t}^{2(L-1)} (u_i \otimes v_j)
    \end{align*}
    This gives $r \times (d_0 - r) = r(d_0 - r)$ eigenvectors.
    
    \textbf{Case 3 (Bulk space, part 2):} If $i > r$ and $j \leq r$, then
    \begin{align*}
        A_{o,t}^\top A_{o,t} (u_i \otimes v_j) &= 0 + \lambda_{j,t}^{2(L-1)} (u_i \otimes v_j) + 0 = \lambda_{j,t}^{2(L-1)} (u_i \otimes v_j)
    \end{align*}
    This gives $(d_L - r) \times r = (d_L - r)r$ eigenvectors.
    
    \textbf{Case 4 (Zero space):} If $i > r$ and $j > r$, then
    \begin{align*}
        A_{o,t}^\top A_{o,t} (u_i \otimes v_j) = 0
    \end{align*}
    This gives $(d_L - r) \times (d_0 - r)$ eigenvectors with zero eigenvalue.
    
    Combining Cases 2 and 3, the bulk eigenspace has $r(d_0 - r) + (d_L - r)r = (d_0 + d_L - 2r)r$ eigenvectors.
    
    By the standard relationship between eigenvectors of $M^\top M$ and $M M^\top$, the eigenvectors of $H_{o,t} = A_{o,t} A_{o,t}^\top$ corresponding to non-zero eigenvalues are $A_{o,t} (u_i \otimes v_j)$. Among these $(d_0 + d_L - r)r$ eigenvectors:
    \begin{itemize}
        \item $r^2$ eigenvectors (for $i \leq r$ and $j \leq r$) correspond to the dominant eigenspace with eigenvalues of order $L \lambda_{i,t}^{2(L-1)}$.
        \item $(d_0 + d_L - 2r)r$ eigenvectors (for ($i \leq r$ and $j > r$) or ($i > r$ and $j \leq r$)) correspond to the bulk eigenspace with eigenvalues of order $\lambda_{i,t}^{2(L-1)}$ or $\lambda_{j,t}^{2(L-1)}$.
    \end{itemize}
    This completes the proof.
\end{proof}


{\color{red}
Then we consider the gap between the idealized outer-product Hessian $H_{o,t}$ and the true Hessian $H_{\mathcal{L},t}$. By treating $H_{f,t}$ as a perturbation and applying Weyl's inequality, we show that the eigenvalues of the true Hessian remain close to those of $H_{o,t}$, preserving the bifurcated structure. 
}
\begin{lemma}
\label{lemma true hessian eigenvalues}
    Under Assumption \ref{assumption 12} and \ref{assumption 4}, let $f_{i,j,t}(k)$ be defined as the $k$th largest value of $\sum_{l=1}^L \lambda_{i,t}^{2(L-l)} \lambda_{j,t}^{2(l-1)}$ for $i = 1, \ldots, d_L$ and $j = 1, \ldots, r$. Then for the true Hessian $H_{L,t}$ at time $t$, its eigenvalues can be bounded as:
    \begin{align*}
        f_{i,j,t}(k) - O(e^{-L\alpha\eta t}) \leq \lambda_k(H_{L,t}) \leq f_{i,j,t}(k) + O(e^{-L\alpha\eta t})
    \end{align*}
    where $\lambda_k$ represents the $k$th largest eigenvalue, and {\color{red}
     $\alpha > 0$ is the uniform lower bound established in \cref{lemma eigenvalues lower bound}, satisfying $\min_{1,\cdots,r}(\lambda_{i,t}^{2L-2}) \geq \alpha$ for all $t \in \mathbb{N}$.}, which is a positive constant that is independent of $t$.
\end{lemma}

\begin{proof}
    Recall that $H_{L,t} = H_{o,t} + H_{f,t}$. As shown in Lemma \ref{lemma: 2norm-t}, $\|H_{f,t}\|_2 = O(e^{-L\alpha\eta t})$, which implies that $\lambda_{\max}(H_{f,t}) = O(e^{-L\alpha\eta t})$ for the functional Hessian at time $t$.
    
    Applying Weyl's Inequality yields:
    \begin{align*}
        \lambda_k(H_{o,t}) - \lambda_{\max}(H_{f,t}) \leq \lambda_k(H_{L,t}) \leq \lambda_k(H_{o,t}) + \lambda_{\max}(H_{f,t})
    \end{align*}
    
    By Lemma \ref{lemma generalize eigenvalues}, $\lambda_k(H_{o,t}) = f_{i,j,t}(k)$. Therefore:
    \begin{align*}
        f_{i,j,t}(k) - O(e^{-L\alpha\eta t}) \leq \lambda_k(H_{L,t}) \leq f_{i,j,t}(k) + O(e^{-L\alpha\eta t})
    \end{align*}
    which completes the proof.
\end{proof}

\begin{lemma-restated}[Shared Spectral Structure]
    Under Assumptions \ref{assumption 12} and \ref{assumption 4}, all weight matrices $W^k_t$ for $k = 1, \ldots, L$ share the same spectral structure $\Sigma_t^{1/L} = \text{diag}(\lambda_{1,t}, \lambda_{2,t}, \ldots, \lambda_{d_*,t})$ at any time $t \geq 0$. Specifically:
    \begin{enumerate}
        \item For $1 < k < L$: $W^k_t \simeq \Sigma_t^{1/L}$
        \item For $k = 1$: $W^1_t \simeq \Sigma_t^{1/L} V^\top$
        \item For $k = L$: $W^L_t \simeq U \Sigma_t^{1/L}$
    \end{enumerate}
    where $U \in \mathbb{R}^{d_L \times d_*}$ and $V \in \mathbb{R}^{d_0 \times d_*}$ are the left and right singular vector matrices from the balanced initialization, which remain constant throughout training. Furthermore, the eigenvalues $\lambda_{i,t}$ for $i = 1, \ldots, r$ evolve according to:
    \begin{align}
        \lambda_{i,t+1} = \lambda_{i,t} - \eta \lambda_{i,t}^{2L-1} + \eta \lambda_{i,t}^{L-1}
    \end{align}
    while the eigenvalues $\lambda_{i,t}$ for $i = r+1, \ldots, d_*$ remain unchanged at their initial values.
\end{lemma-restated}

\begin{proof}
    Under Assumption \ref{assumption 12}, since $d_* \leq \min\{d_{L-1}, d_{L-2}, \ldots, d_1\}$, the balanced initialization procedure yields:
    \begin{align*}
        W^k_0 &\simeq \Sigma^{1/L}, \quad \text{for } 1 < k < L \\
        W^1_0 &\simeq \Sigma^{1/L} V^\top \\
        W^L_0 &\simeq U \Sigma^{1/L}
    \end{align*}
    where $A = U \Sigma V^\top$ is the SVD of the initialization matrix $A$, and $\Sigma^{1/L} = \text{diag}(\lambda_{1,0}, \lambda_{2,0}, \ldots, \lambda_{d_*,0})$ with $\lambda_{i,0}$ being the $L$-th root of the $i$-th singular value.
    
    By Lemma \ref{lemma 2}, the eigenvectors of $W^k_t$ are invariant during training. For $1 < k < L$, the gradient is:
    \begin{align*}
        \nabla_{W^k} L(W^k_t) = \Sigma_t^{(2L-1)/L} - \Sigma_t^{(L-1)/L}
    \end{align*}
    which is a diagonal matrix. Therefore, the update $W^k_{t+1} = W^k_t - \eta \nabla_{W^k} L(W^k_t)$ only modifies the diagonal entries, preserving the diagonal structure.
    
    For $k = 1$, the gradient is:
    \begin{align*}
        \nabla_{W^1} L(W^1_t) = (\Sigma_t^{(2L-1)/L} - \Sigma_t^{(L-1)/L}) V^\top
    \end{align*}
    Thus:
    \begin{align*}
        W^1_{t+1} &= W^1_t - \eta \nabla_{W^1} L(W^1_t) \\
        &= \Sigma_t^{1/L} V^\top - \eta (\Sigma_t^{(2L-1)/L} - \Sigma_t^{(L-1)/L}) V^\top \\
        &= \left[ \Sigma_t^{1/L} - \eta (\Sigma_t^{(2L-1)/L} - \Sigma_t^{(L-1)/L}) \right] V^\top \\
        &= \Sigma_{t+1}^{1/L} V^\top
    \end{align*}
    where the right factor $V^\top$ remains unchanged.
    
    For $k = L$, the gradient is:
    \begin{align*}
        \nabla_{W^L} L(W^L_t) = U (\Sigma_t^{(2L-1)/L} - \Sigma_t^{(L-1)/L})
    \end{align*}
    Thus:
    \begin{align*}
        W^L_{t+1} &= W^L_t - \eta \nabla_{W^L} L(W^L_t) \\
        &= U \Sigma_t^{1/L} - \eta U (\Sigma_t^{(2L-1)/L} - \Sigma_t^{(L-1)/L}) \\
        &= U \left[ \Sigma_t^{1/L} - \eta (\Sigma_t^{(2L-1)/L} - \Sigma_t^{(L-1)/L}) \right] \\
        &= U \Sigma_{t+1}^{1/L}
    \end{align*}
    where the left factor $U$ remains unchanged.
    
    In all cases, the updated weight matrix shares the same diagonal structure $\Sigma_{t+1}^{1/L}$, where the $i$-th diagonal entry evolves according to:
    \begin{align*}
        \lambda_{i,t+1} = \lambda_{i,t} - \eta \lambda_{i,t}^{2L-1} + \eta \lambda_{i,t}^{L-1}
    \end{align*}
    
    By Lemma \ref{lemma eigenvalues convergence}, under Assumption \ref{assumption 12}, only the first $r$ eigenvalues $\lambda_{1,t}, \ldots, \lambda_{r,t}$ are effectively updated since $\Sigma_{xx} \simeq I_r$. The remaining eigenvalues $\lambda_{r+1,t}, \ldots, \lambda_{d_*,t}$ correspond to directions orthogonal to the support of $\Sigma_{xx}$, and thus their gradients are zero:
    \begin{align*}
        \lambda_{i,t+1} = \lambda_{i,t} = \lambda_{i,0}, \quad \text{for } i = r+1, \ldots, d_*
    \end{align*}
    
    By induction on $t$, all weight matrices $W^k_t$ share the same spectral structure $\Sigma_t^{1/L}$ at any time $t \geq 0$, which completes the proof.
\end{proof}

{\color{red}
Finally, we are equipped with all the results we need. we are now ready to state our main result regarding the Hessian bifurcation phenomenon in deep linear networks
}

\begin{theorem-restated}\ref{theorem main result}(Hessian Bifurcation)
    Under Assumptions \ref{assumption 12} and \ref{assumption 4}, and assuming $r \leq d_* = \min\{d_0, d_L\}$, consider a depth-$L$ deep linear neural network trained with gradient descent with step size {\color{red} $\eta < \min\big\{\frac{1}{L}, \frac{1}{M^{2L-2}}\big\}$}. Let $\lambda_{i,t}$ for $i = 1, \ldots, r$ denote the effective eigenvalues of the weight matrix $\Sigma_t^{1/L}$ at time $t$. Suppose that $\lambda_{i,t} \in [m_t - \delta_t, m_t + \delta_t]$ and the condition $(m_t + \delta_t)/(m_t - \delta_t) < L^{\frac{1}{2(L-1)}}$ holds. Then the Hessian $H_{L,t}$ exhibits a three-part spectral structure:
    \begin{enumerate}[leftmargin=*]
        \item \textbf{(Dominant Space)} There exist $r^2$ eigenvalues of $H_{L,t}$ lying in:
        \begin{equation}
        \begin{split}
            \bigg[ & L(m_t - \delta_t)^{2(L-1)} - O(e^{-L\alpha\eta t}), \\
                   & L(m_t + \delta_t)^{2(L-1)} + O(e^{-L\alpha\eta t}) \bigg]
        \end{split}
        \end{equation}
        
        \item \textbf{(Bulk Space)} There exist $(d_0 + d_L - 2r)r$ eigenvalues of $H_{L,t}$ lying in:
        \begin{equation}
        \begin{split}
            \bigg[ & (m_t - \delta_t)^{2(L-1)} - O(e^{-L\alpha\eta t}), \\
                   & (m_t + \delta_t)^{2(L-1)} + O(e^{-L\alpha\eta t}) \bigg]
        \end{split}
        \end{equation}
        
        \item \textbf{(Zero Space)} The remaining $(d_L - r)(d_0 - r)$ eigenvalues of $H_{L,t}$ are $O(e^{-L\alpha\eta t})$, converging to zero as $t \to \infty$.
    \end{enumerate}
    
    Moreover, let $\lambda_{\text{dom}}$ denote an arbitrary eigenvalue belonging to the Dominant Space and $\lambda_{\text{bulk}}$ denote an arbitrary eigenvalue belonging to the Bulk Space. Their ratio satisfies:
    \begin{equation}
        \frac{\lambda_{\text{dom}}}{\lambda_{\text{bulk}}} = \Theta(L),
    \end{equation}
    where $\alpha$ is a positive constant independent of $t$, provided $t$ is sufficiently large.
\end{theorem-restated}

\begin{proof}
    By Lemma \ref{lemma generalize eigenvalues}, under the assumption $r \leq d_*$, the outer product Hessian $H_{o,t}$ at time $t$ has $(d_0 + d_L - r)r$ non-zero eigenvalues. These eigenvalues can be indexed by pairs $(i, j)$ and exhibit three different behaviors:
    
    \textbf{Case 1 (Dominant space):} If $i \leq r$ and $j \leq r$, the eigenvalue is
    \begin{align*}
        \nu_{i,j} = \sum_{l=1}^L \lambda_{i,t}^{2(L-l)} \lambda_{j,t}^{2(l-1)}
    \end{align*}
    When $i = j$, this equals $L \lambda_{i,t}^{2(L-1)}$. When $i \neq j$ and $\lambda_{i,t} \neq \lambda_{j,t}$, this equals $\frac{\lambda_{i,t}^{2L} - \lambda_{j,t}^{2L}}{\lambda_{i,t}^2 - \lambda_{j,t}^2}$. When $i \neq j$ but $\lambda_{i,t} = \lambda_{j,t}$, this also equals $L \lambda_{i,t}^{2(L-1)}$ by continuity.
    
    \textbf{Case 2 (Bulk space, part 1):} If $i \leq r$ and $j > r$, the eigenvalue is
    \begin{align*}
        \nu_{i,j} = \lambda_{i,t}^{2(L-1)}
    \end{align*}
    
    \textbf{Case 3 (Bulk space, part 2):} If $i > r$ and $j \leq r$, the eigenvalue is
    \begin{align*}
        \nu_{i,j} = \lambda_{j,t}^{2(L-1)}
    \end{align*}
    
    \textbf{Case 4 (Zero space):} If $i > r$ and $j > r$, the eigenvalue is zero.
    
    Since $\lambda_{i,t} \in [m_t - \delta_t, m_t + \delta_t]$ for all $i = 1, \ldots, r$, we bound the eigenvalues in each case.
    
    For Case 1, since each term $\lambda_{i,t}^{2(L-l)} \lambda_{j,t}^{2(l-1)}$ satisfies
    \begin{align*}
        (m_t - \delta_t)^{2(L-1)} \leq \lambda_{i,t}^{2(L-l)} \lambda_{j,t}^{2(l-1)} \leq (m_t + \delta_t)^{2(L-1)}
    \end{align*}
    summing over $l = 1, \ldots, L$ yields
    \begin{align*}
        L(m_t - \delta_t)^{2(L-1)} \leq \nu_{i,j} \leq L(m_t + \delta_t)^{2(L-1)}
    \end{align*}
    This gives $r \times r = r^2$ eigenvalues in the dominant space.
    
    For Cases 2 and 3:
    \begin{align*}
        (m_t - \delta_t)^{2(L-1)} \leq \nu_{i,j} \leq (m_t + \delta_t)^{2(L-1)}
    \end{align*}
    Case 2 gives $r \times (d_0 - r) = r(d_0 - r)$ eigenvalues, and Case 3 gives $(d_L - r) \times r = (d_L - r)r$ eigenvalues. Combining these, the bulk space contains
    \begin{align*}
        r(d_0 - r) + (d_L - r)r = (d_0 + d_L - 2r)r
    \end{align*}
    eigenvalues.
    
    For Case 4, there are $(d_L - r) \times (d_0 - r) = (d_L - r)(d_0 - r)$ zero eigenvalues.
    
    The condition $(m_t + \delta_t)/(m_t - \delta_t) < L^{\frac{1}{2(L-1)}}$ is equivalent to
    \begin{align*}
        (m_t + \delta_t)^{2(L-1)} < L (m_t - \delta_t)^{2(L-1)}
    \end{align*}
    which guarantees a spectral gap between the dominant space and the bulk space:
    \begin{align*}
        \min_{i,j \leq r} \nu_{i,j} \geq L(m_t - \delta_t)^{2(L-1)} > (m_t + \delta_t)^{2(L-1)} \geq \max_{\substack{i \leq r, j > r \\ \text{or } i > r, j \leq r}} \nu_{i,j}
    \end{align*}
    
    By the Gauss-Newton decomposition, $H_{L,t} = H_{o,t} + H_{f,t}$. According to Lemma \ref{lemma: 2norm-t}, the functional Hessian satisfies $\|H_{f,t}\|_2 = O(e^{-L\alpha\eta t})$. Applying Weyl's inequality, for the $k$th largest eigenvalue:
    \begin{align*}
        \lambda_k(H_{o,t}) - \|H_{f,t}\|_2 \leq \lambda_k(H_{L,t}) \leq \lambda_k(H_{o,t}) + \|H_{f,t}\|_2
    \end{align*}
    
    Therefore, the eigenvalues of the true Hessian $H_{L,t}$ lie within $O(e^{-L\alpha\eta t})$ of the eigenvalues of the outer product Hessian $H_{o,t}$. Combining with the bounds above:
    \begin{itemize}
        \item The $r^2$ dominant space eigenvalues of $H_{L,t}$ lie in
        \begin{align*}
            \left[ L(m_t - \delta_t)^{2(L-1)} - O(e^{-L\alpha\eta t}), \, L(m_t + \delta_t)^{2(L-1)} + O(e^{-L\alpha\eta t}) \right]
        \end{align*}
        
        \item The $(d_0 + d_L - 2r)r$ bulk space eigenvalues of $H_{L,t}$ lie in
        \begin{align*}
            \left[ (m_t - \delta_t)^{2(L-1)} - O(e^{-L\alpha\eta t}), \, (m_t + \delta_t)^{2(L-1)} + O(e^{-L\alpha\eta t}) \right]
        \end{align*}
        
        \item The remaining $(d_L - r)(d_0 - r)$ eigenvalues are $O(e^{-L\alpha\eta t})$, converging to zero as $t \to \infty$.
    \end{itemize}
    
    Finally, we establish the ratio between the dominant space eigenvalues and bulk space eigenvalues. For the lower bound:
    \begin{align*}
        \frac{\lambda_{\text{dom}}}{\lambda_{\text{bulk}}} \geq \frac{L(m_t - \delta_t)^{2(L-1)} - O(e^{-L\alpha\eta t})}{(m_t + \delta_t)^{2(L-1)} + O(e^{-L\alpha\eta t})}
    \end{align*}
    For sufficiently large $t$, the $O(e^{-L\alpha\eta t})$ terms become negligible, yielding
    \begin{align*}
        \frac{\lambda_{\text{dom}}}{\lambda_{\text{bulk}}} \geq \frac{L(m_t - \delta_t)^{2(L-1)}}{(m_t + \delta_t)^{2(L-1)}} = L \left( \frac{m_t - \delta_t}{m_t + \delta_t} \right)^{2(L-1)}
    \end{align*}
    
    For the upper bound:
    \begin{align*}
        \frac{\lambda_{\text{dom}}}{\lambda_{\text{bulk}}} \leq \frac{L(m_t + \delta_t)^{2(L-1)} + O(e^{-L\alpha\eta t})}{(m_t - \delta_t)^{2(L-1)} - O(e^{-L\alpha\eta t})}
    \end{align*}
    For sufficiently large $t$:
    \begin{align*}
        \frac{\lambda_{\text{dom}}}{\lambda_{\text{bulk}}} \leq \frac{L(m_t + \delta_t)^{2(L-1)}}{(m_t - \delta_t)^{2(L-1)}} = L \left( \frac{m_t + \delta_t}{m_t - \delta_t} \right)^{2(L-1)}
    \end{align*}
    
    Since the condition $(m_t + \delta_t)/(m_t - \delta_t) < L^{\frac{1}{2(L-1)}}$ implies that $\left( \frac{m_t + \delta_t}{m_t - \delta_t} \right)^{2(L-1)} < L$, both bounds are $O(L)$. Therefore:
    \begin{align*}
        \frac{\lambda_{\text{dom}}}{\lambda_{\text{bulk}}} = \Theta(L)
    \end{align*}
    which completes the proof.
\end{proof}

\begin{corollary-restated}\ref{corollary two eigenvalues}(Hessian Bifurcation with USI)
    Under Assumptions \ref{assumption 12}, \ref{assumption 4}, and \ref{assumption 5}, and assuming $r \leq d_* = \min\{d_0, d_L\}$, the Hessian $H_{L,t}$ has exactly two distinct nonzero eigenvalues (up to an exponentially small perturbation). Specifically:
    \begin{enumerate}[leftmargin=*]
        \item \textbf{(Dominant Space)} There exist $r^2$ eigenvalues equal to $L\mu_t^{2(L-1)} + O(e^{-L\alpha\eta t})$.
        
        \item \textbf{(Bulk Space)} There exist $(d_0 + d_L - 2r)r$ eigenvalues equal to $\mu_t^{2(L-1)} + O(e^{-L\alpha\eta t})$.
        
        \item \textbf{(Zero Space)} The remaining $(d_L - r)(d_0 - r)$ eigenvalues are $O(e^{-L\alpha\eta t})$.
    \end{enumerate}
    Moreover, the ratio between the dominant and bulk eigenvalues is exactly $L$, i.e.,
    \begin{align}
        \frac{\lambda_{\text{dom}}}{\lambda_{\text{bulk}}} = L
    \end{align}
    where $\mu_t$ denotes the common value of all effective eigenvalues $\lambda_{i,t} = \mu_t$ for $i = 1, \ldots, r$ at time $t$, and $\alpha$ is a positive constant.
\end{corollary-restated}

\begin{proof}
    Under Assumption \ref{assumption 5}, $\Sigma_0^{1/L} \simeq \mu I_r$, which implies $\lambda_{i,0} = \mu$ for all $i = 1, \ldots, r$. By Lemma \ref{lemma shared spectral structure}, the effective eigenvalues evolve according to:
    \begin{align*}
        \lambda_{i,t+1} = \lambda_{i,t} - \eta \lambda_{i,t}^{2L-1} + \eta \lambda_{i,t}^{L-1}
    \end{align*}
    for $i = 1, \ldots, r$.
    
    Since all initial eigenvalues are equal, and the update rule is identical for each $i = 1, \ldots, r$, we have $\lambda_{1,t} = \lambda_{2,t} = \cdots = \lambda_{r,t} =: \mu_t$ for all $t \geq 0$. By Lemma \ref{lemma eigenvalues convergence}, $\mu_t \to 1$ as $t \to \infty$.
    
    By Lemma \ref{lemma generalize eigenvalues}, the outer product Hessian $H_{o,t}$ has eigenvalues indexed by pairs $(i, j)$. For $i \leq r$ and $j \leq r$:
    \begin{align*}
        \nu_{i,j} = \sum_{l=1}^L \lambda_{i,t}^{2(L-l)} \lambda_{j,t}^{2(l-1)} = \sum_{l=1}^L \mu_t^{2(L-l)} \mu_t^{2(l-1)} = \sum_{l=1}^L \mu_t^{2(L-1)} = L \mu_t^{2(L-1)}
    \end{align*}
    
    For $i \leq r$ and $j > r$:
    \begin{align*}
        \nu_{i,j} = \lambda_{i,t}^{2(L-1)} = \mu_t^{2(L-1)}
    \end{align*}
    
    For $i > r$ and $j \leq r$:
    \begin{align*}
        \nu_{i,j} = \lambda_{j,t}^{2(L-1)} = \mu_t^{2(L-1)}
    \end{align*}
    
    For $i > r$ and $j > r$:
    \begin{align*}
        \nu_{i,j} = 0
    \end{align*}
    
    Therefore, the outer product Hessian $H_{o,t}$ has exactly two distinct nonzero eigenvalues:
    \begin{itemize}
        \item $L \mu_t^{2(L-1)}$ with multiplicity $r^2$ (dominant space)
        \item $\mu_t^{2(L-1)}$ with multiplicity $r(d_0 - r) + (d_L - r)r = (d_0 + d_L - 2r)r$ (bulk space)
    \end{itemize}
    and zero eigenvalues with multiplicity $(d_L - r)(d_0 - r)$ (zero space).
    
    By Lemma \ref{lemma: 2norm-t}, $\|H_{f,t}\|_2 = O(e^{-L\alpha\eta t})$. Applying Weyl's inequality:
    \begin{align*}
        \lambda_k(H_{o,t}) - O(e^{-L\alpha\eta t}) \leq \lambda_k(H_{L,t}) \leq \lambda_k(H_{o,t}) + O(e^{-L\alpha\eta t})
    \end{align*}
    
    Therefore, the Hessian $H_{L,t}$ has:
    \begin{itemize}
        \item $r^2$ eigenvalues equal to $L \mu_t^{2(L-1)} + O(e^{-L\alpha\eta t})$
        \item $(d_0 + d_L - 2r)r$ eigenvalues equal to $\mu_t^{2(L-1)} + O(e^{-L\alpha\eta t})$
        \item $(d_L - r)(d_0 - r)$ eigenvalues equal to $O(e^{-L\alpha\eta t})$
    \end{itemize}
    
    The ratio between the dominant and bulk eigenvalues is:
    \begin{align*}
        \frac{\lambda_{\text{dom}}}{\lambda_{\text{bulk}}} = \frac{L \mu_t^{2(L-1)} + O(e^{-L\alpha\eta t})}{\mu_t^{2(L-1)} + O(e^{-L\alpha\eta t})} = L + O(e^{-L\alpha\eta t})
    \end{align*}
    For sufficiently large $t$, this ratio converges to exactly $L$, which completes the proof.
\end{proof}

\section{Simulations}
\label{appendix_simulation}

\begin{figure}[ht]
    \centering
    \includegraphics[width=\linewidth]{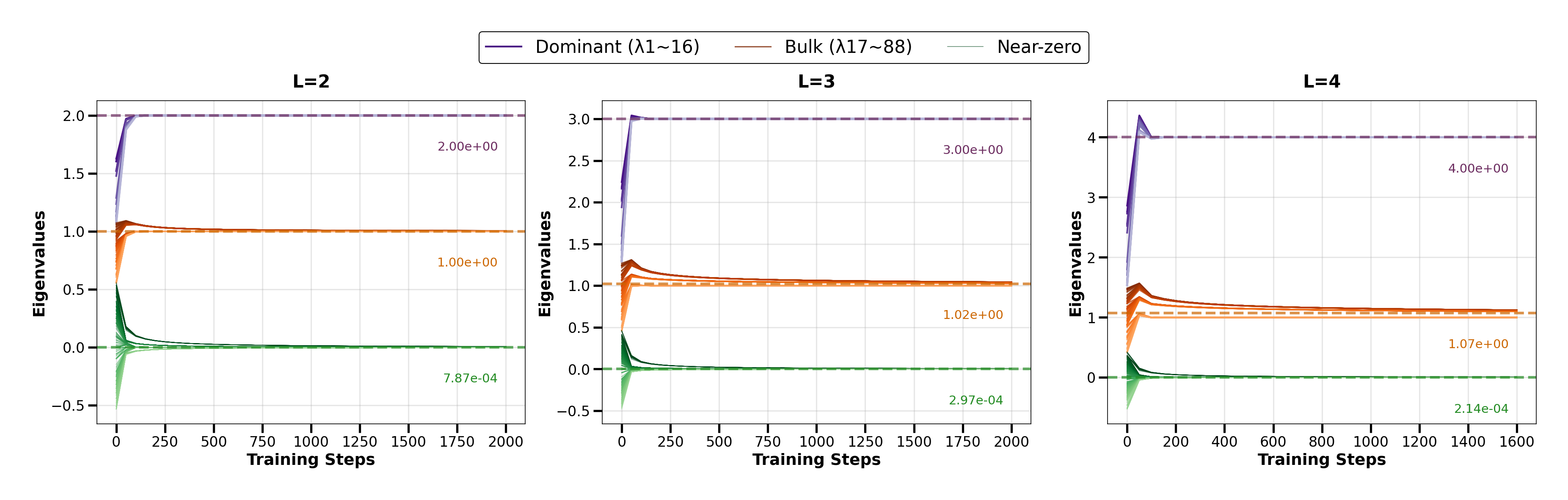}
    \includegraphics[width=\linewidth]{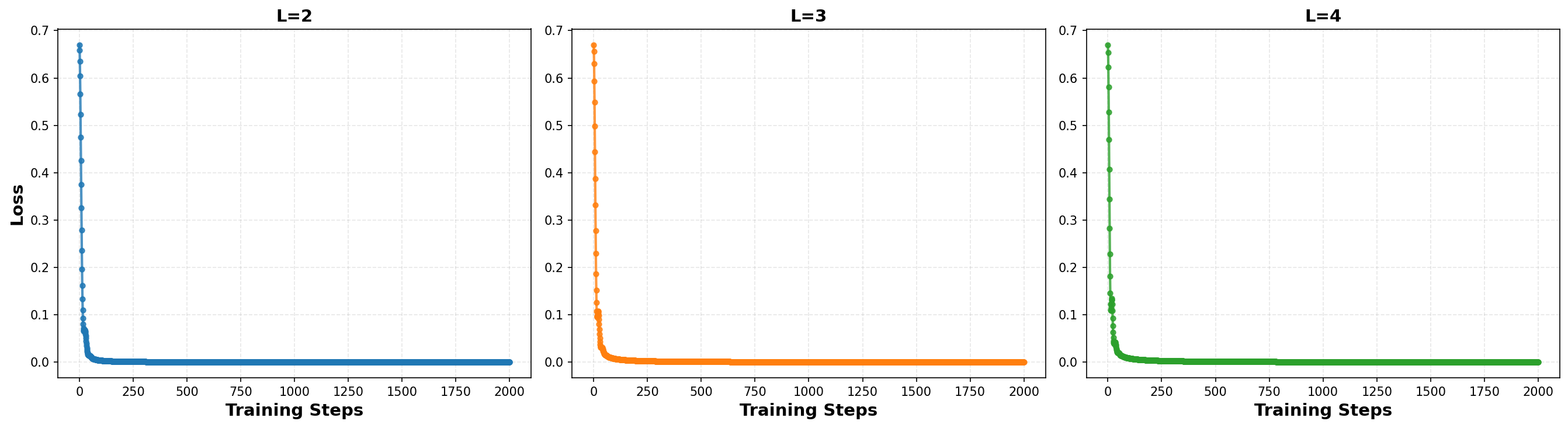}
    \caption{\textbf{Eigenvalue evolution.} The curves are color-coded by subspace: purple for the dominant space, orange for the bulk space, and green for the near-zero space. The dominant space has a dimension of $\text{rank}^2$, while the combined dimension of the dominant and bulk spaces equals the product of the input and output dimensions. The final eigenvalues of the dominant space converge to $L$ times those of the bulk space. The panels correspond to $L=3$ (left), $L=4$ (middle), and $L=5$ (right).}
    \label{fig:eigenvalue_evolution_loss_L234_r4}
\end{figure}

\begin{figure}[ht]
    \centering
    \includegraphics[width=\linewidth]{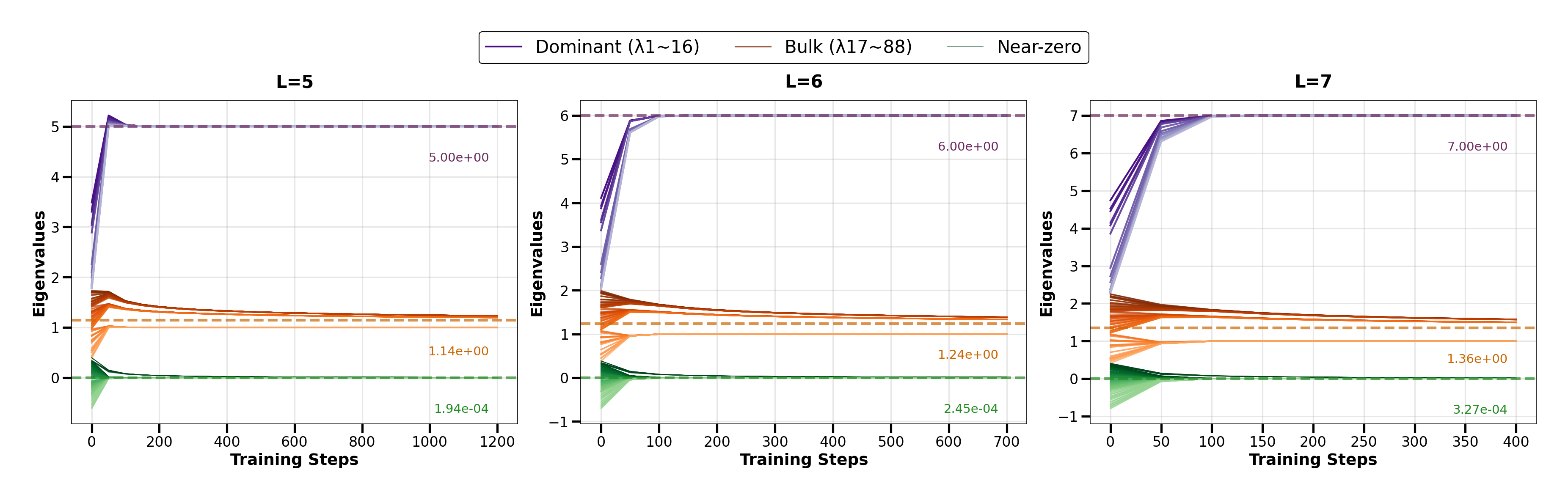}
    \includegraphics[width=\linewidth]{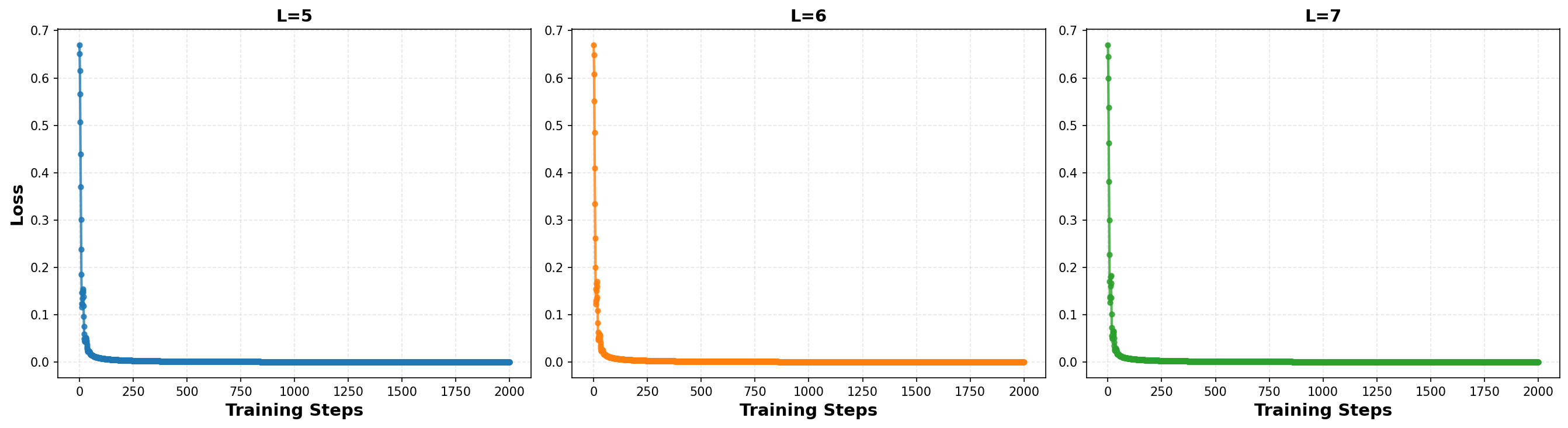}
    \caption{\textbf{Eigenvalue evolution.} The curves are color-coded by subspace: purple for the dominant space, orange for the bulk space, and green for the near-zero space. The dominant space has a dimension of $\text{rank}^2$, while the combined dimension of the dominant and bulk spaces equals the product of the input and output dimensions. The final eigenvalues of the dominant space converge to $L$ times those of the bulk space. The panels correspond to $L=3$ (left), $L=4$ (middle), and $L=5$ (right).}
    \label{fig:eigenvalue_evolution_loss_L567_r4}
\end{figure}

\begin{figure}[ht]
    \centering
    \includegraphics[width=\linewidth]{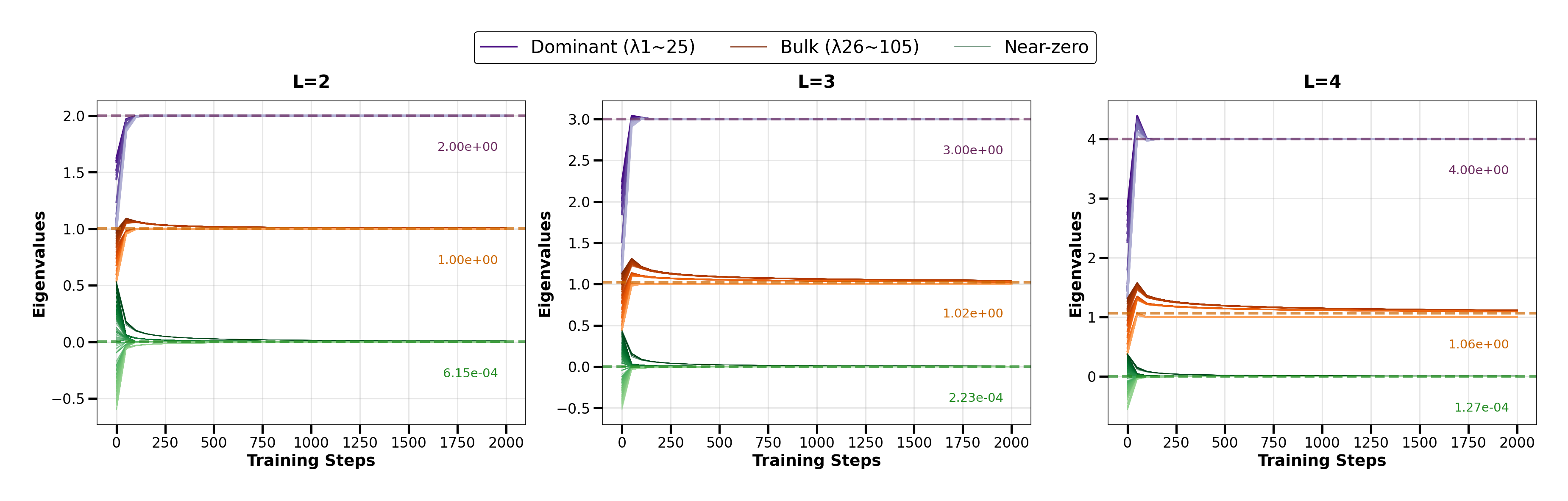}
    \includegraphics[width=\linewidth]{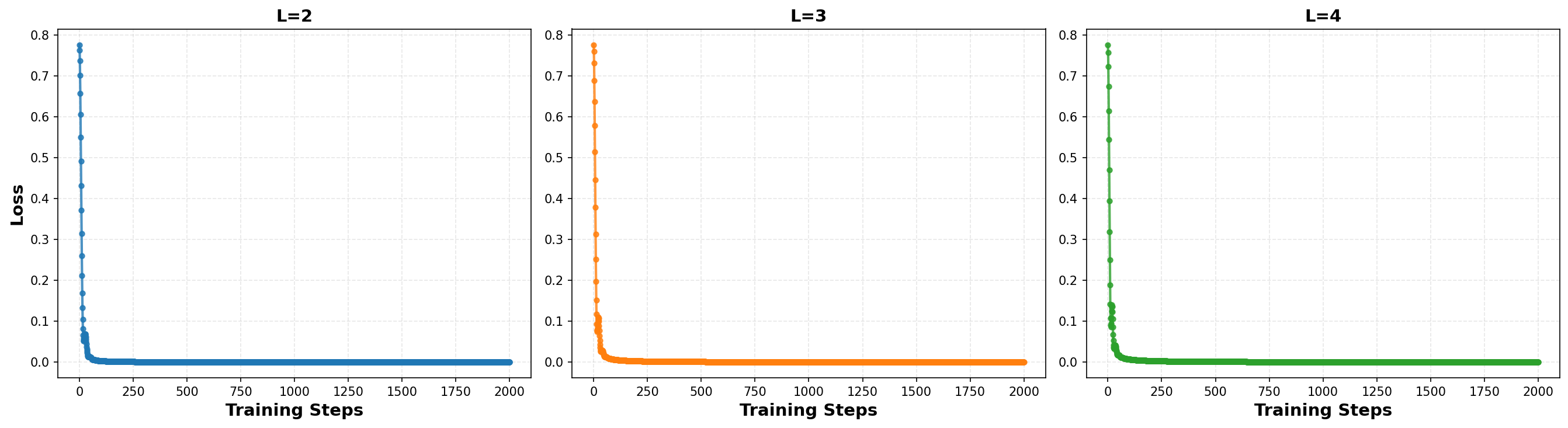}
    \caption{\textbf{Eigenvalue evolution.} The curves are color-coded by subspace: purple for the dominant space, orange for the bulk space, and green for the near-zero space. The dominant space has a dimension of $\text{rank}^2$, while the combined dimension of the dominant and bulk spaces equals the product of the input and output dimensions. The final eigenvalues of the dominant space converge to $L$ times those of the bulk space. The panels correspond to $L=3$ (left), $L=4$ (middle), and $L=5$ (right).}
    \label{fig:eigenvalue_evolution_loss_L234_r5}
\end{figure}

\begin{figure}[ht]
    \centering
    \includegraphics[width=\linewidth]{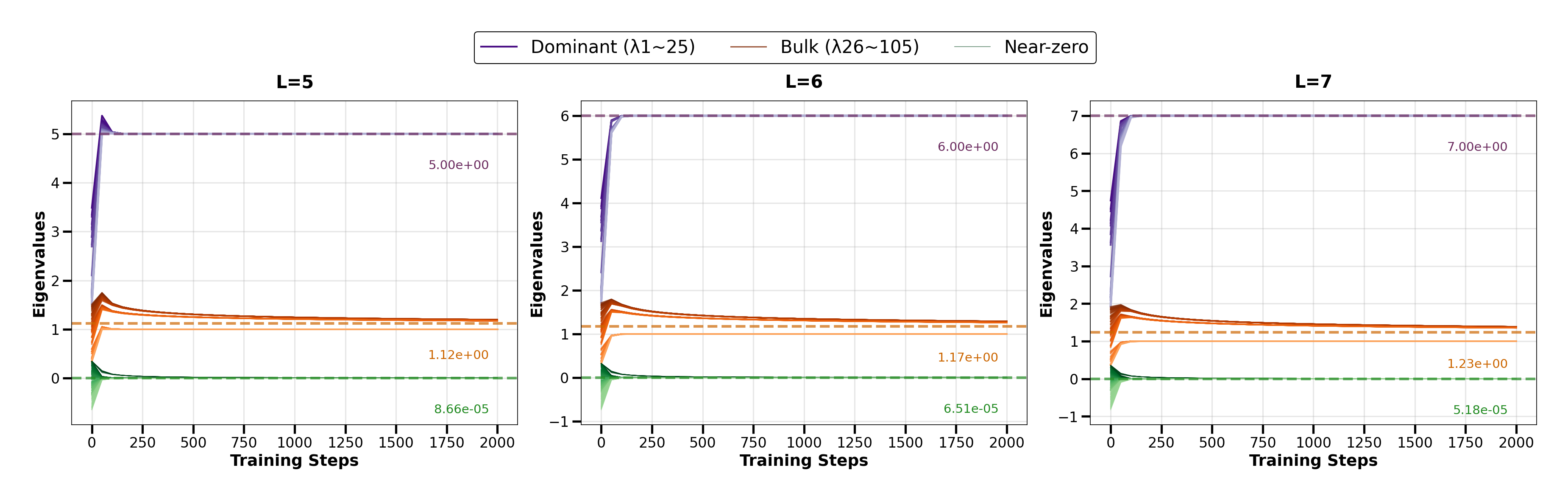}
    \includegraphics[width=\linewidth]{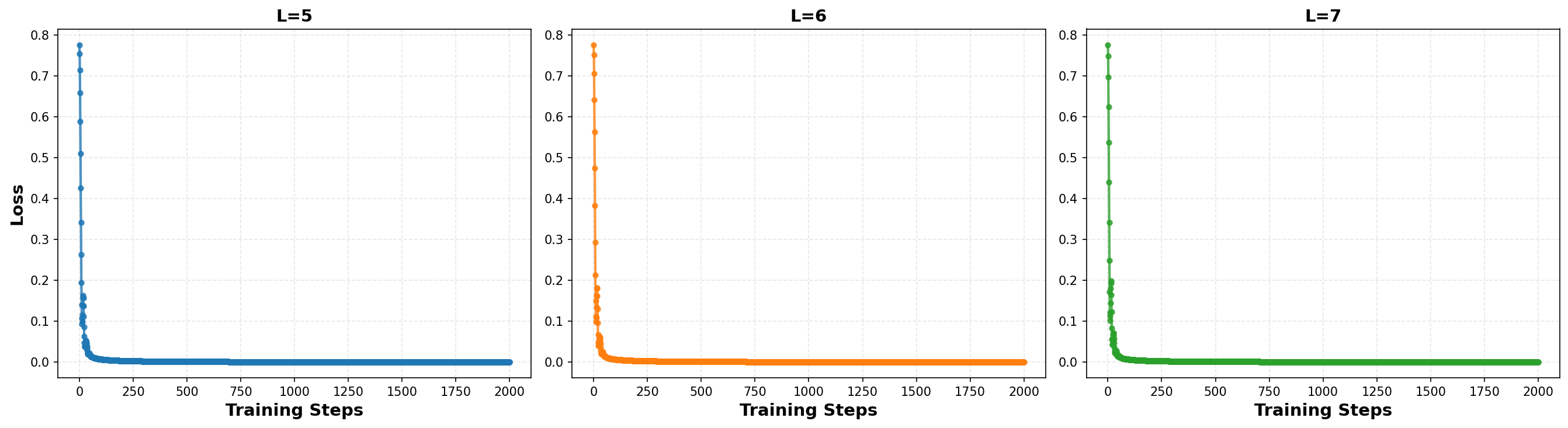}
    \caption{\textbf{Eigenvalue evolution.} The curves are color-coded by subspace: purple for the dominant space, orange for the bulk space, and green for the near-zero space. The dominant space has a dimension of $\text{rank}^2$, while the combined dimension of the dominant and bulk spaces equals the product of the input and output dimensions. The final eigenvalues of the dominant space converge to $L$ times those of the bulk space. The panels correspond to $L=3$ (left), $L=4$ (middle), and $L=5$ (right).}
    \label{fig:eigenvalue_evolution_loss_L567_r5}
\end{figure}

\begin{figure}[ht]
    \centering
    \includegraphics[width=\linewidth]{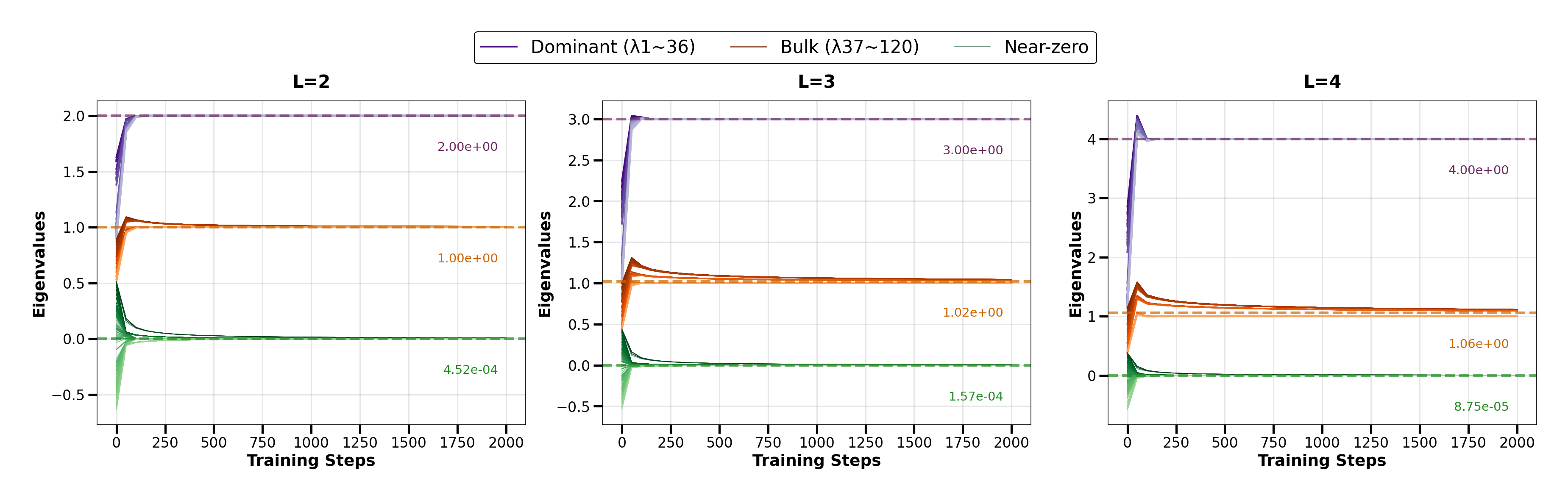}
    \includegraphics[width=\linewidth]{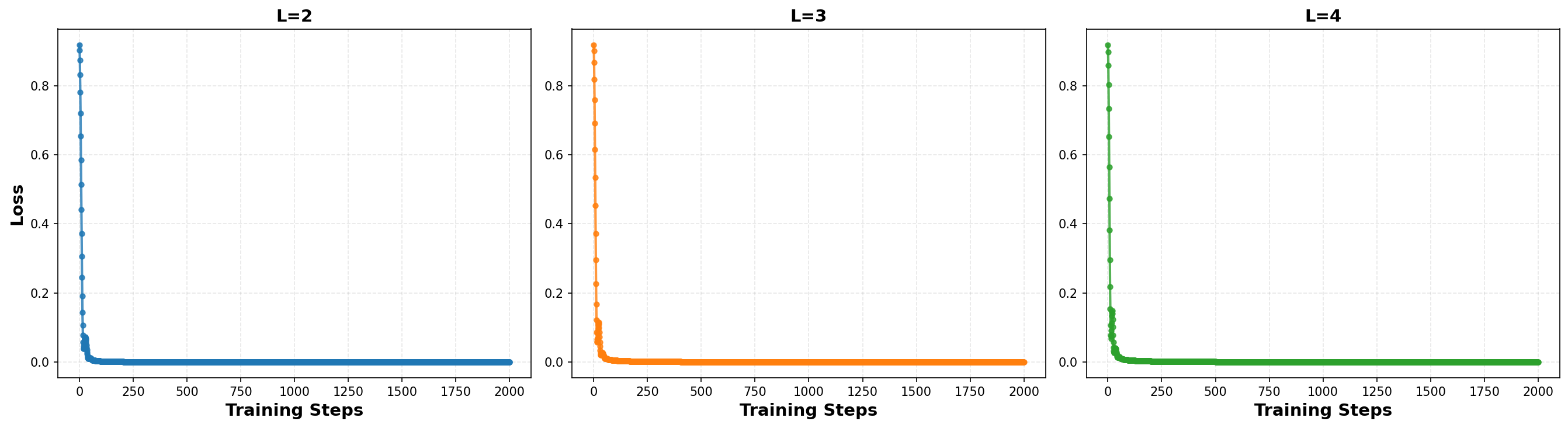}
    \caption{\textbf{Eigenvalue evolution.} The curves are color-coded by subspace: purple for the dominant space, orange for the bulk space, and green for the near-zero space. The dominant space has a dimension of $\text{rank}^2$, while the combined dimension of the dominant and bulk spaces equals the product of the input and output dimensions. The final eigenvalues of the dominant space converge to $L$ times those of the bulk space. The panels correspond to $L=3$ (left), $L=4$ (middle), and $L=5$ (right).}
    \label{fig:eigenvalue_evolution_loss_L234_r6}
\end{figure}

\begin{figure}[ht]
    \centering
    \includegraphics[width=\linewidth]{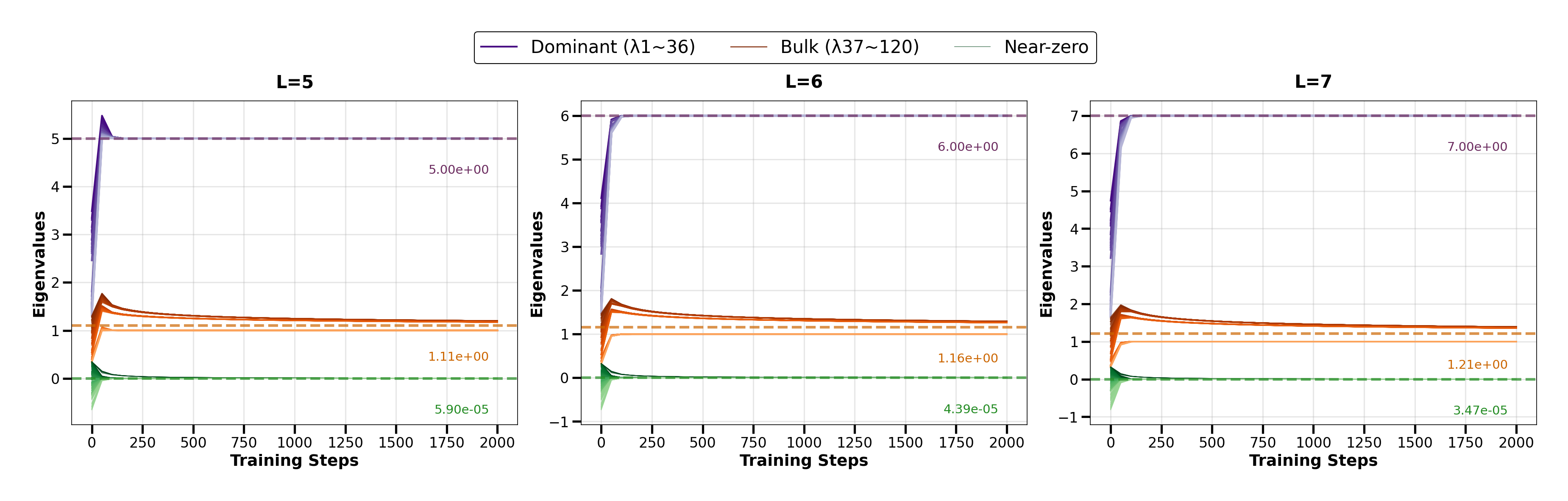}
    \includegraphics[width=\linewidth]{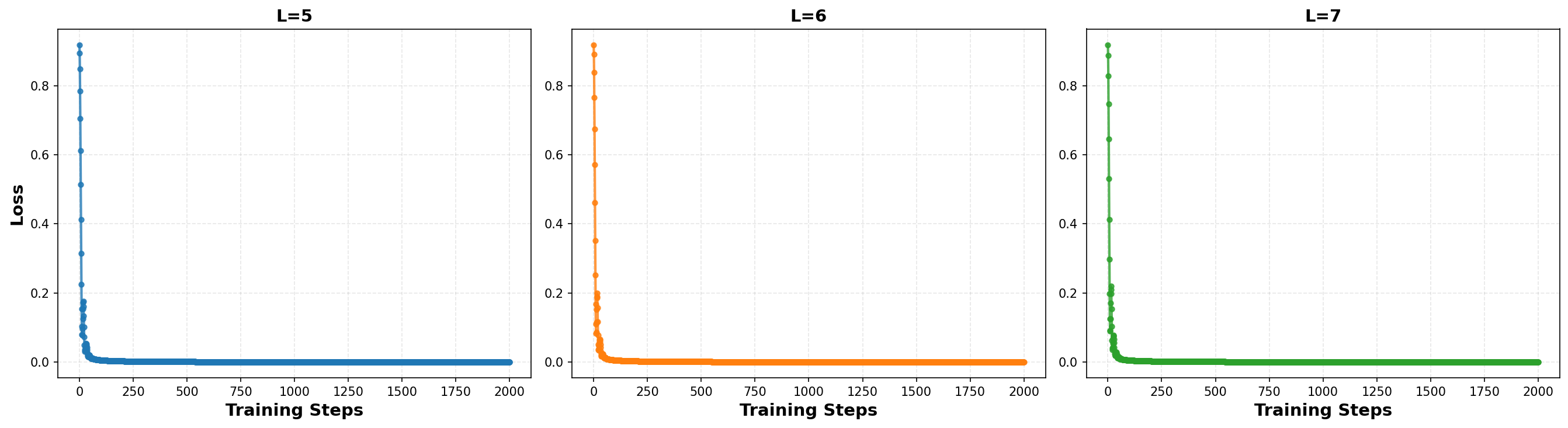}
    \caption{\textbf{Eigenvalue evolution.} The curves are color-coded by subspace: purple for the dominant space, orange for the bulk space, and green for the near-zero space. The dominant space has a dimension of $\text{rank}^2$, while the combined dimension of the dominant and bulk spaces equals the product of the input and output dimensions. The final eigenvalues of the dominant space converge to $L$ times those of the bulk space. The panels correspond to $L=3$ (left), $L=4$ (middle), and $L=5$ (right).}
    \label{fig:eigenvalue_evolution_loss_L567_r6}
\end{figure}

\begin{figure}[ht]
    \centering
    \includegraphics[width=\linewidth]{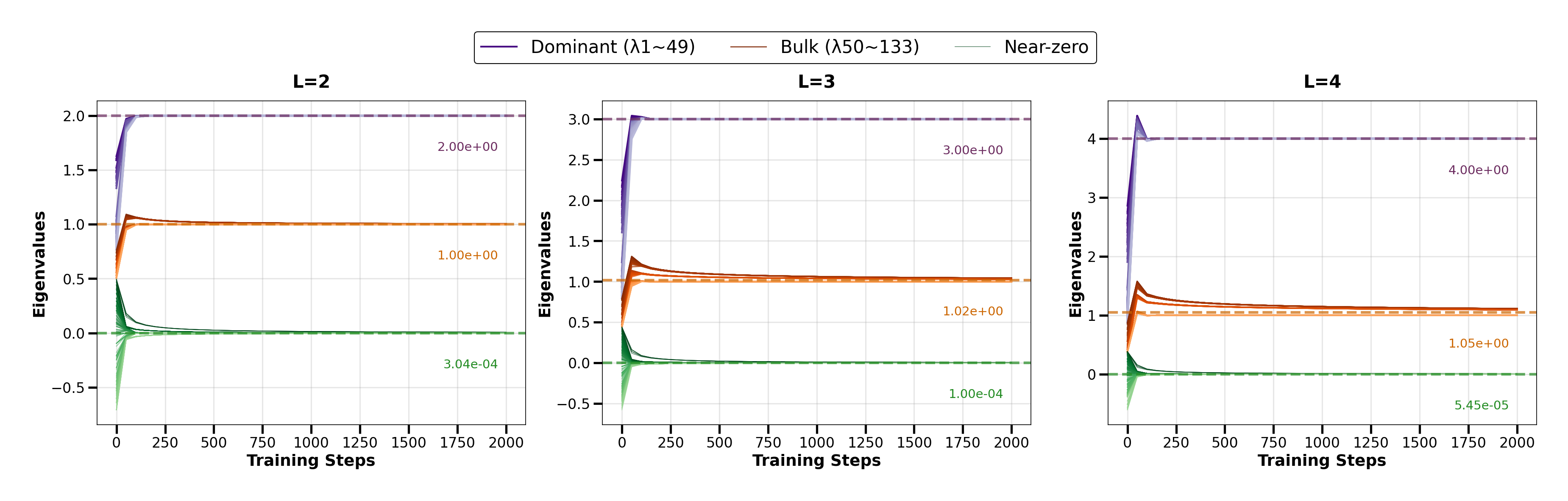}
    \includegraphics[width=\linewidth]{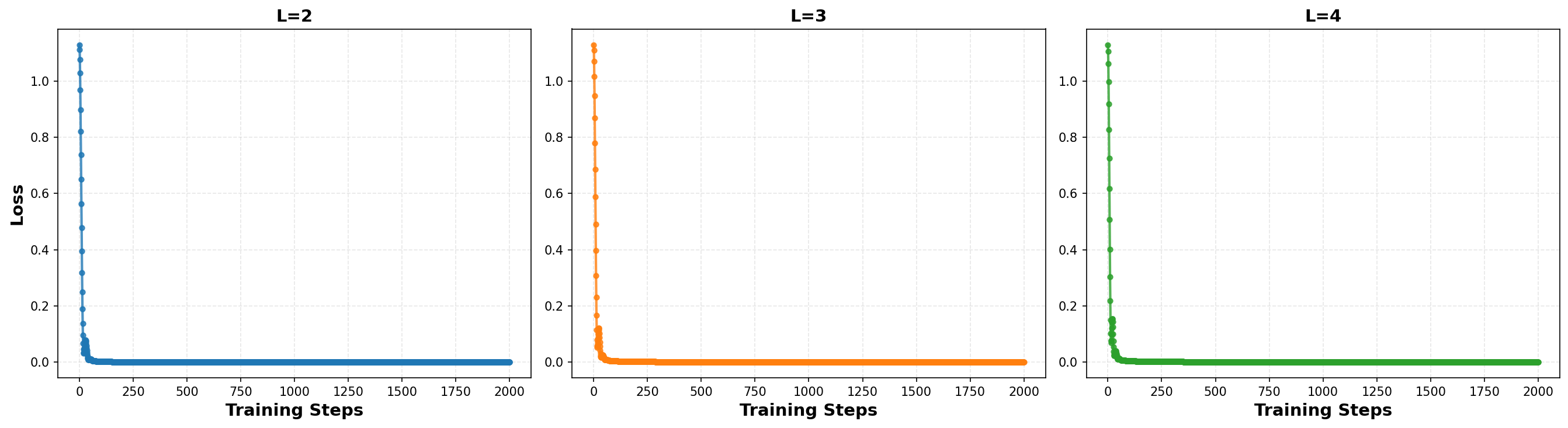}
    \caption{\textbf{Eigenvalue evolution.} The curves are color-coded by subspace: purple for the dominant space, orange for the bulk space, and green for the near-zero space. The dominant space has a dimension of $\text{rank}^2$, while the combined dimension of the dominant and bulk spaces equals the product of the input and output dimensions. The final eigenvalues of the dominant space converge to $L$ times those of the bulk space. The panels correspond to $L=3$ (left), $L=4$ (middle), and $L=5$ (right).}
    \label{fig:eigenvalue_evolution_loss_L234_r7}
\end{figure}

\begin{figure}[ht]
    \centering
    \includegraphics[width=\linewidth]{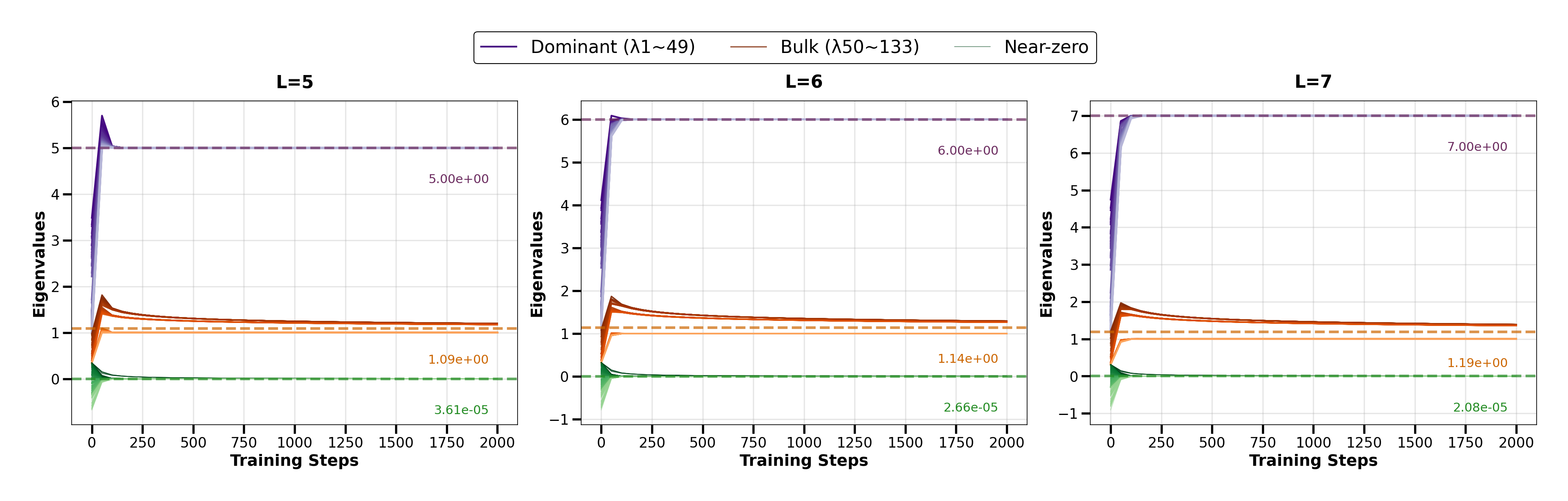}
    \includegraphics[width=\linewidth]{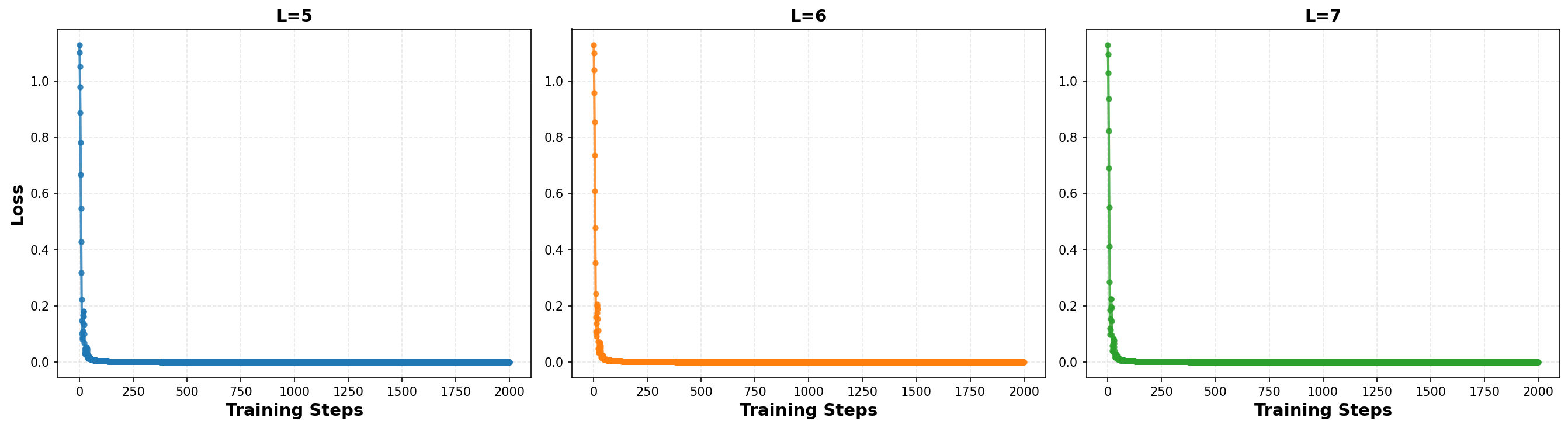}
    \caption{\textbf{Eigenvalue evolution.} The curves are color-coded by subspace: purple for the dominant space, orange for the bulk space, and green for the near-zero space. The dominant space has a dimension of $\text{rank}^2$, while the combined dimension of the dominant and bulk spaces equals the product of the input and output dimensions. The final eigenvalues of the dominant space converge to $L$ times those of the bulk space. The panels correspond to $L=3$ (left), $L=4$ (middle), and $L=5$ (right).}
    \label{fig:eigenvalue_evolution_loss_L567_r7}
\end{figure}

\begin{figure}[ht]
    \centering
    \includegraphics[width=\linewidth]{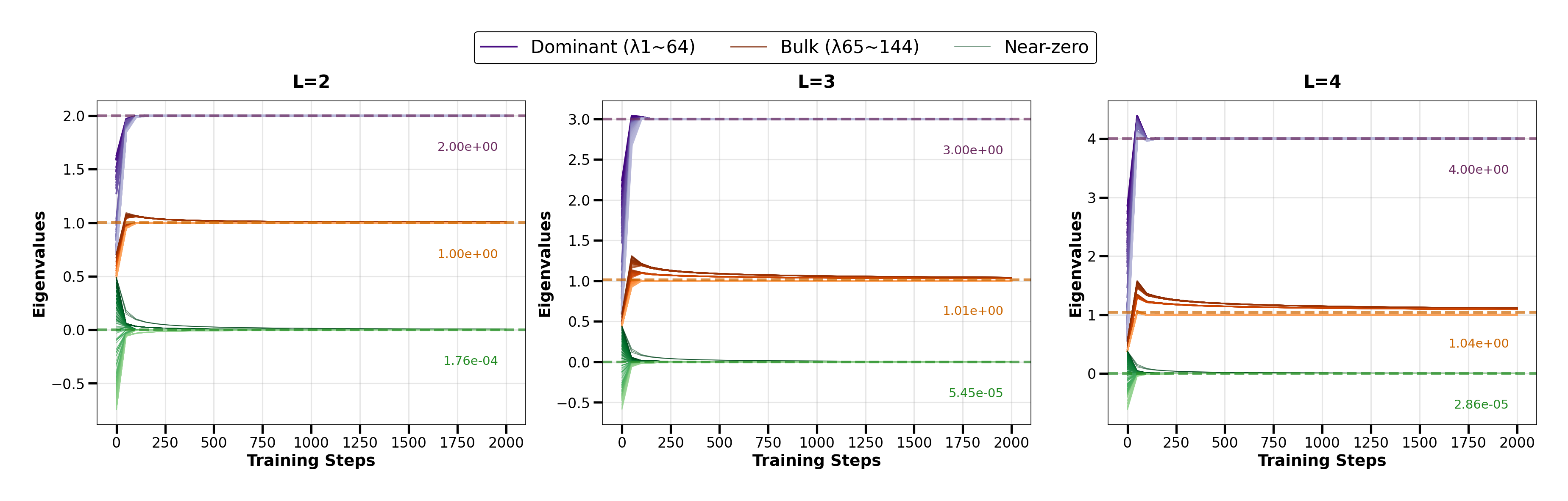}
    \includegraphics[width=\linewidth]{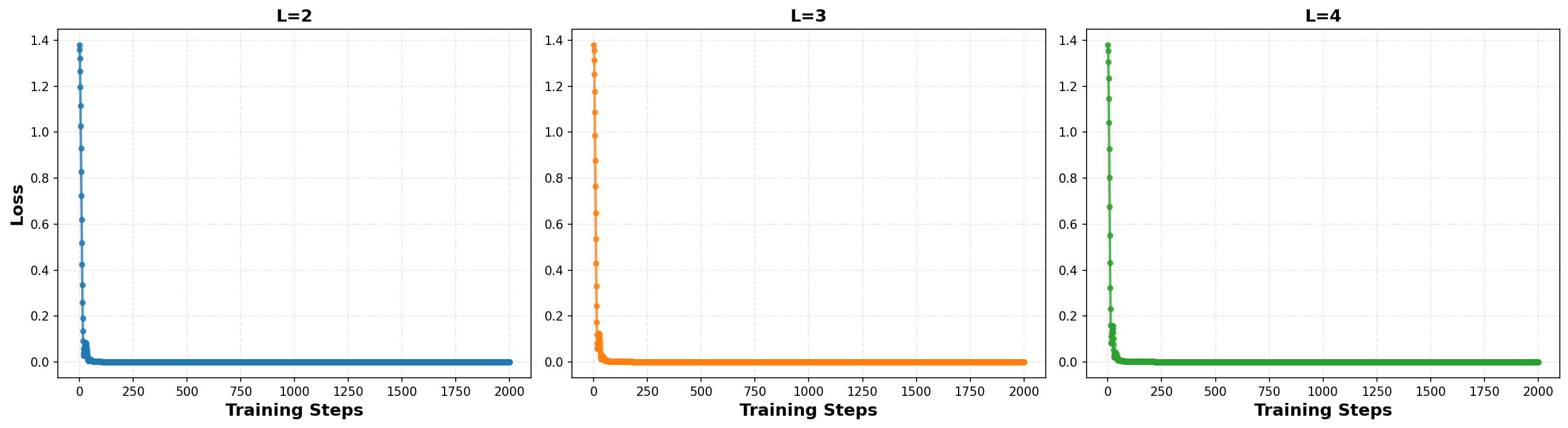}
    \caption{\textbf{Eigenvalue evolution.} The curves are color-coded by subspace: purple for the dominant space, orange for the bulk space, and green for the near-zero space. The dominant space has a dimension of $\text{rank}^2$, while the combined dimension of the dominant and bulk spaces equals the product of the input and output dimensions. The final eigenvalues of the dominant space converge to $L$ times those of the bulk space. The panels correspond to $L=3$ (left), $L=4$ (middle), and $L=5$ (right).}
    \label{fig:eigenvalue_evolution_loss_L234_r8}
\end{figure}

\begin{figure}[ht]
    \centering
    \includegraphics[width=\linewidth]{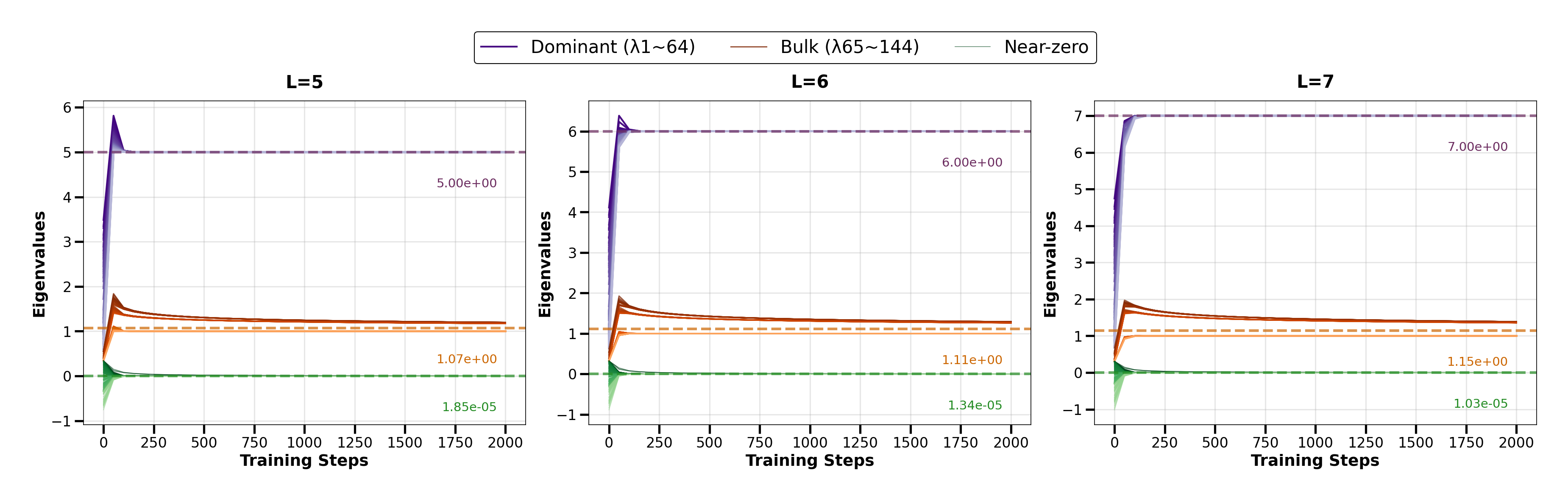}
    \includegraphics[width=\linewidth]{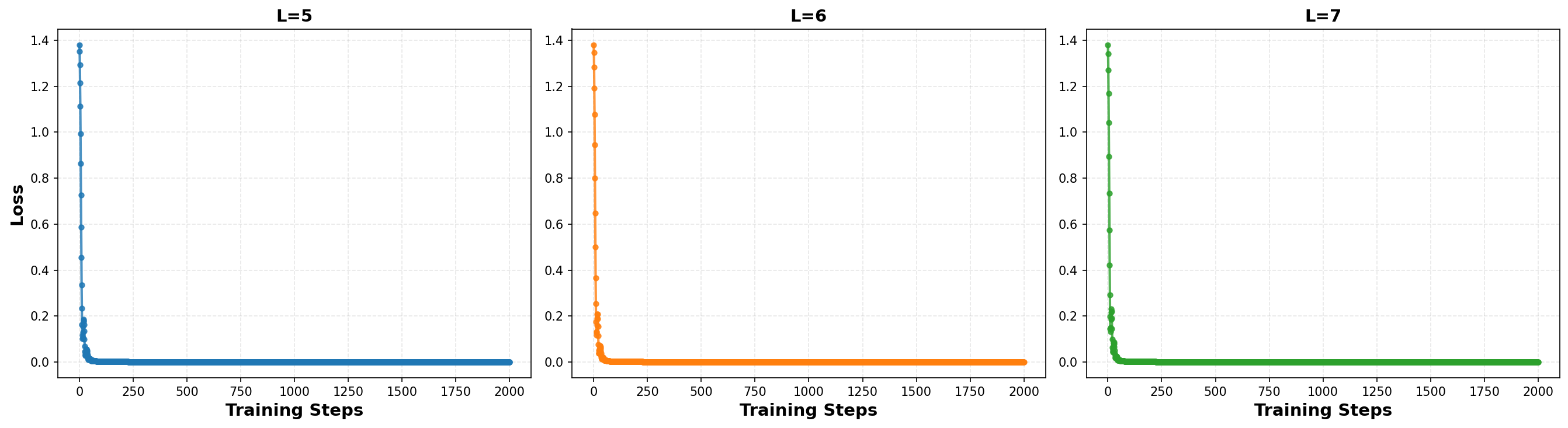}
    \caption{\textbf{Eigenvalue evolution.} The curves are color-coded by subspace: purple for the dominant space, orange for the bulk space, and green for the near-zero space. The dominant space has a dimension of $\text{rank}^2$, while the combined dimension of the dominant and bulk spaces equals the product of the input and output dimensions. The final eigenvalues of the dominant space converge to $L$ times those of the bulk space. The panels correspond to $L=3$ (left), $L=4$ (middle), and $L=5$ (right).}
    \label{fig:eigenvalue_evolution_loss_L567_r8}
\end{figure}

\begin{figure}[ht]
    \centering
    \includegraphics[width=\linewidth]{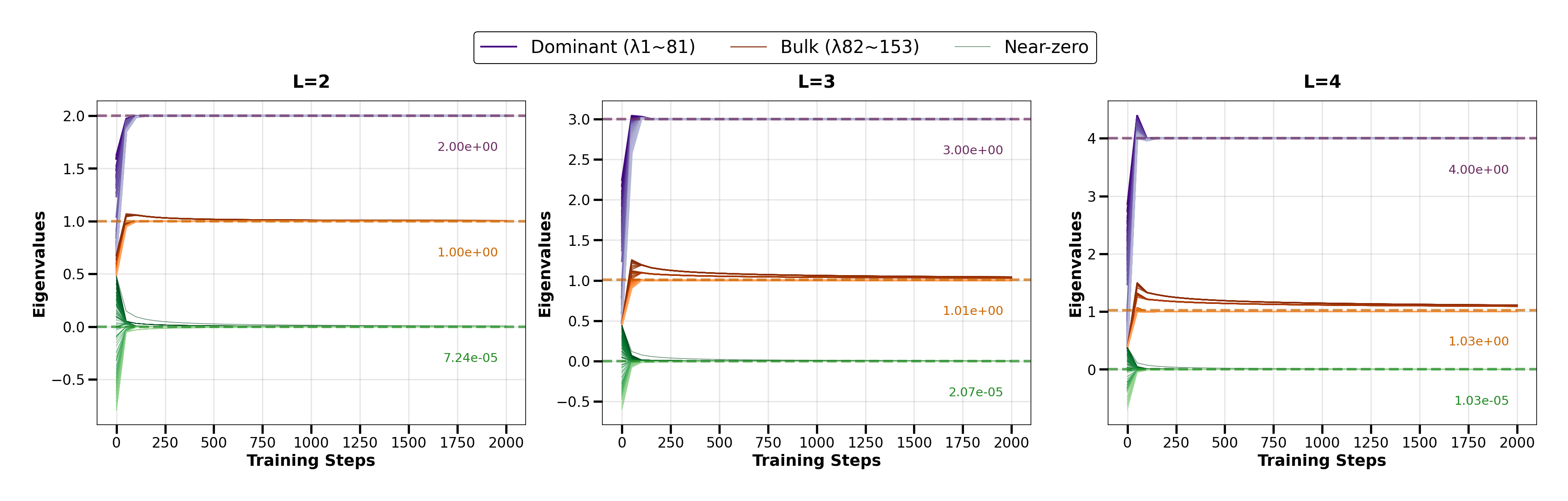}
    \includegraphics[width=\linewidth]{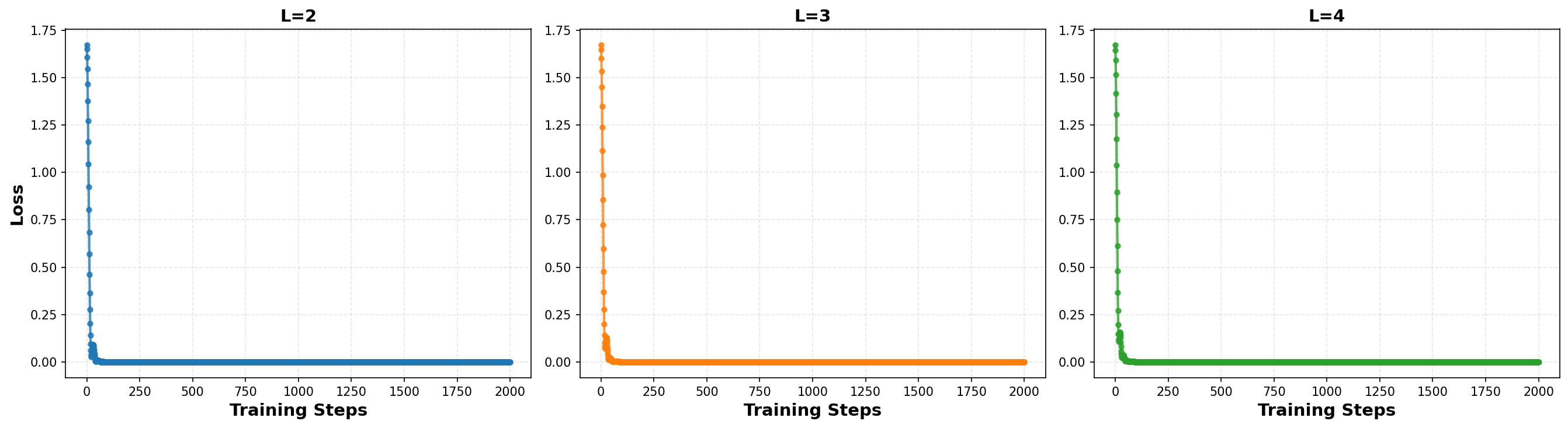}
    \caption{\textbf{Eigenvalue evolution.} The curves are color-coded by subspace: purple for the dominant space, orange for the bulk space, and green for the near-zero space. The dominant space has a dimension of $\text{rank}^2$, while the combined dimension of the dominant and bulk spaces equals the product of the input and output dimensions. The final eigenvalues of the dominant space converge to $L$ times those of the bulk space. The panels correspond to $L=3$ (left), $L=4$ (middle), and $L=5$ (right).}
    \label{fig:eigenvalue_evolution_loss_L234_r9}
\end{figure}

\begin{figure}[ht]
    \centering
    \includegraphics[width=\linewidth]{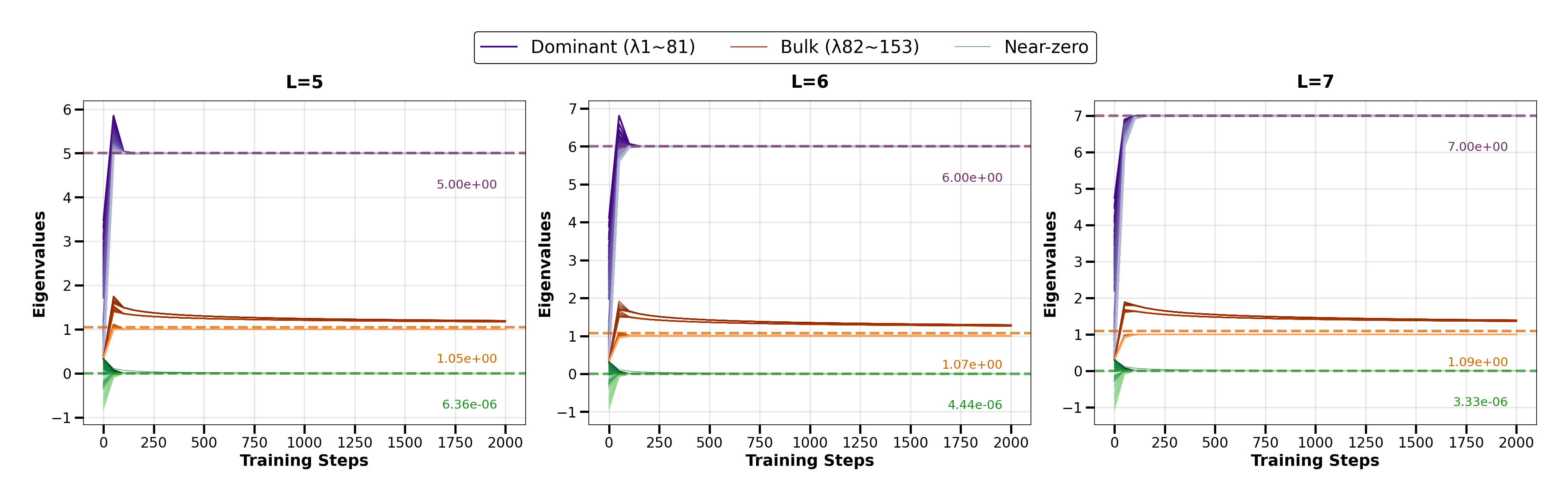}
    \includegraphics[width=\linewidth]{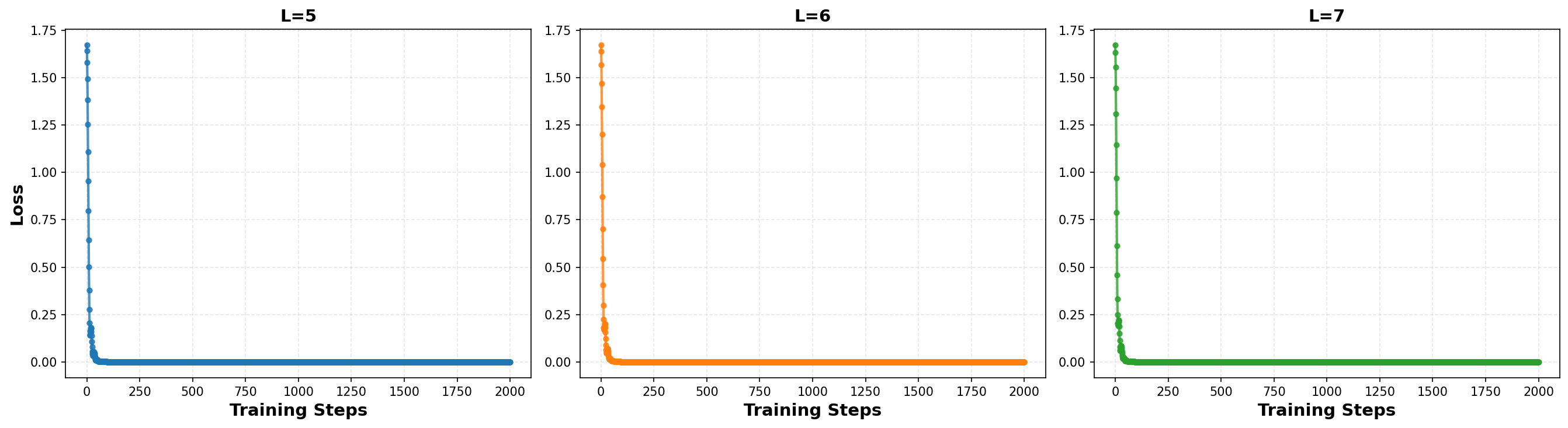}
    \caption{\textbf{Eigenvalue evolution.} The curves are color-coded by subspace: purple for the dominant space, orange for the bulk space, and green for the near-zero space. The dominant space has a dimension of $\text{rank}^2$, while the combined dimension of the dominant and bulk spaces equals the product of the input and output dimensions. The final eigenvalues of the dominant space converge to $L$ times those of the bulk space. The panels correspond to $L=3$ (left), $L=4$ (middle), and $L=5$ (right).}
    \label{fig:eigenvalue_evolution_loss_L567_r9}
\end{figure}

\begin{figure}[ht]
    \centering
    \includegraphics[width=\linewidth]{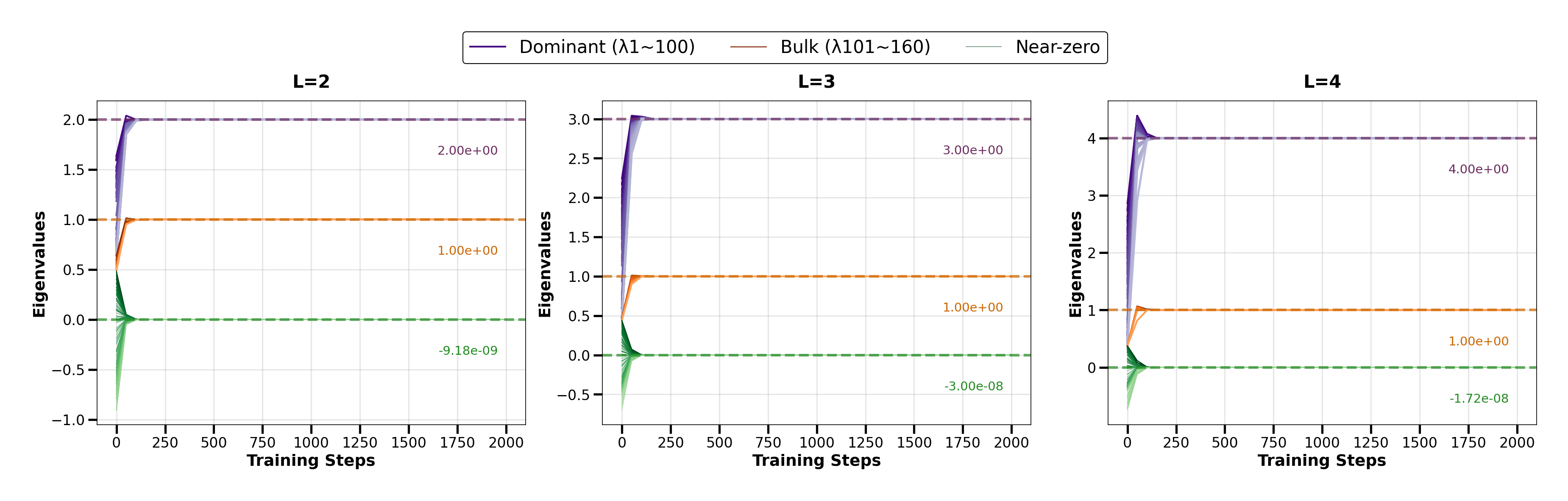}
    \includegraphics[width=\linewidth]{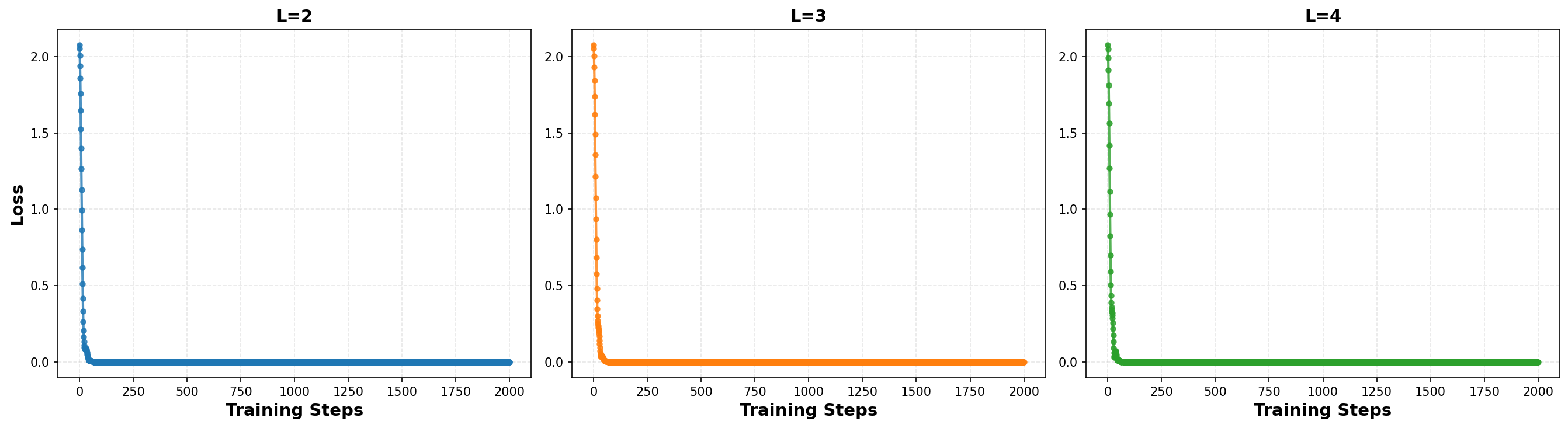}
    \caption{\textbf{Eigenvalue evolution.} The curves are color-coded by subspace: purple for the dominant space, orange for the bulk space, and green for the near-zero space. The dominant space has a dimension of $\text{rank}^2$, while the combined dimension of the dominant and bulk spaces equals the product of the input and output dimensions. The final eigenvalues of the dominant space converge to $L$ times those of the bulk space. The panels correspond to $L=3$ (left), $L=4$ (middle), and $L=5$ (right).}
    \label{fig:eigenvalue_evolution_loss_L234_r10}
\end{figure}

\begin{figure}[ht]
    \centering
    \includegraphics[width=\linewidth]{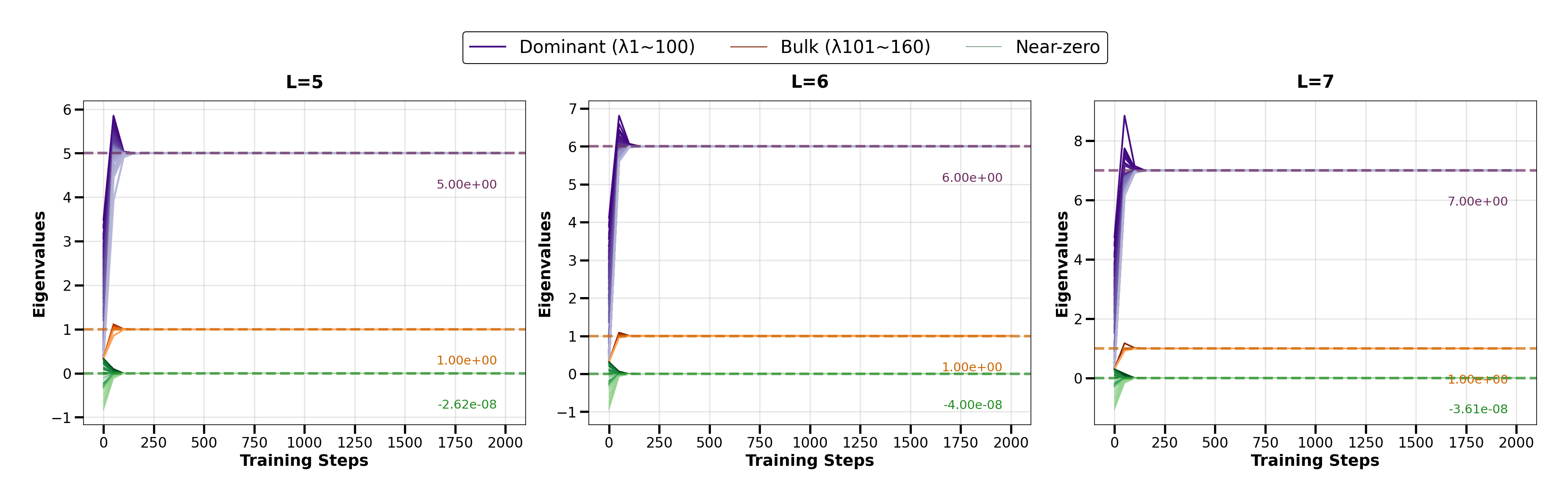}
    \includegraphics[width=\linewidth]{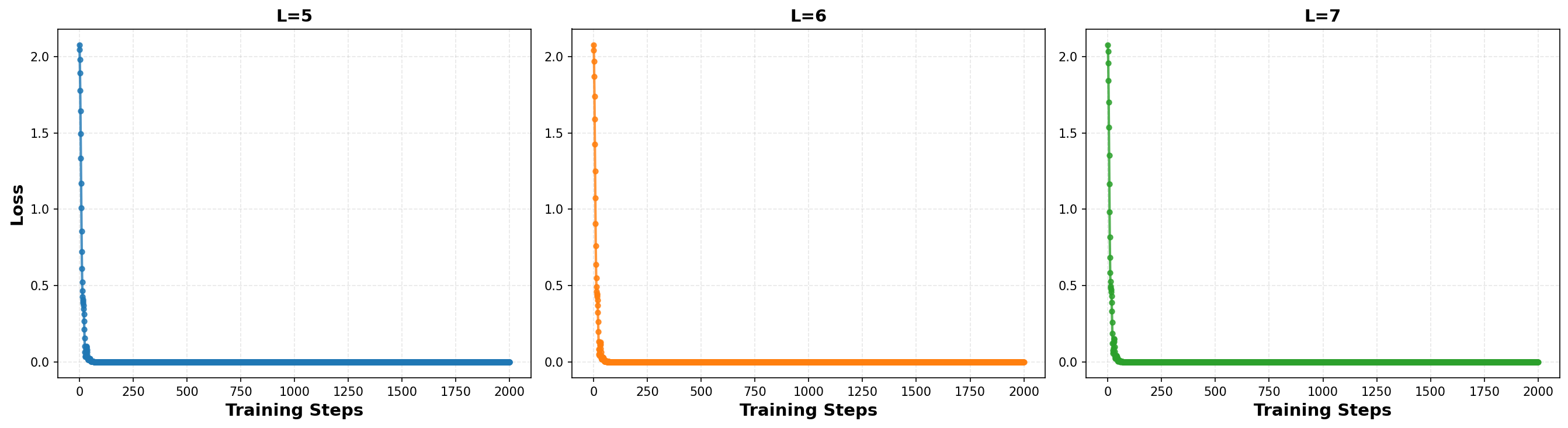}
    \caption{\textbf{Eigenvalue evolution.} The curves are color-coded by subspace: purple for the dominant space, orange for the bulk space, and green for the near-zero space. The dominant space has a dimension of $\text{rank}^2$, while the combined dimension of the dominant and bulk spaces equals the product of the input and output dimensions. The final eigenvalues of the dominant space converge to $L$ times those of the bulk space. The panels correspond to $L=3$ (left), $L=4$ (middle), and $L=5$ (right).}
    \label{fig:eigenvalue_evolution_loss_L567_r10}
\end{figure}

\end{document}